\documentclass[12pt]{article}

\usepackage[margin=1in]{geometry} 
\usepackage{amsmath,amsthm,amssymb}
\usepackage{enumerate}
\usepackage{dsfont}
\usepackage{algorithm}
\usepackage{algpseudocode}
\usepackage{diagbox}
\usepackage{mathtools}
\usepackage{mathdots}
\usepackage{diagbox} \usepackage{mathtools}
\usepackage{mathdots}
\usepackage{graphicx}
\usepackage{hyperref}
\usepackage{multirow}
\usepackage{abstract}
\usepackage{subfig}
\usepackage{tikz}
\usepackage[shortlabels]{enumitem}
\usetikzlibrary{arrows,snakes,backgrounds,decorations.pathreplacing}
\usetikzlibrary{matrix,chains,positioning,decorations.pathreplacing,arrows}

\allowdisplaybreaks

\usepackage{enumitem}

\setlistdepth{9}
\setlist[itemize,1]{label=$\bullet$}
\setlist[itemize,2]{label=$\bullet$}
\setlist[itemize,3]{label=$\bullet$}
\setlist[itemize,4]{label=$\bullet$}
\setlist[itemize,5]{label=$\bullet$}
\setlist[itemize,6]{label=$\bullet$}
\setlist[itemize,7]{label=$\bullet$}
\setlist[itemize,8]{label=$\bullet$}
\setlist[itemize,9]{label=$\bullet$}
\renewlist{itemize}{itemize}{9}

\theoremstyle{theorem}
\newtheorem{theorem}{Theorem}
\newtheorem{corollary}[theorem]{Corollary}
\newtheorem{lemma}[theorem]{Lemma}
\newtheorem{proposition}[theorem]{Proposition}

\theoremstyle{definition}
\newtheorem{definition}[theorem]{Definition}
\newtheorem{example}[theorem]{Example}

\begin{document}

\title{A Framework for the construction of upper bounds on the number of affine linear regions of ReLU feed-forward neural networks}
\author{Peter Hinz\\
  Sara van de Geer
}

\maketitle
\begin{abstract}
  We present a framework to derive upper bounds on the number of regions that feed-forward neural networks with ReLU activation functions are affine linear on. It is based on an inductive analysis that keeps track of the number of such regions per dimensionality of their images within the layers. More precisely, the information about the number regions per dimensionality is pushed through the layers starting with one region of the input dimension of the neural network and using a recursion based on an analysis of how many regions per output dimensionality a subsequent layer with a certain width can induce on an input region with a given dimensionality. The final bound on the number of regions depends on the number and widths of the layers of the neural network and on some additional parameters that were used for the recursion. It is stated in terms of the $L1$-norm of the last column of a product of matrices and provides a unifying treatment of several previously known bounds: Depending on the choice of the recursion parameters that determine these matrices, it is possible to obtain the bounds from Mont\'{u}far~\cite{Montufar:2014:NLR:2969033.2969153} (2014),~\cite{Montufar17} (2017) and Serra et. al.~\cite{DBLP:BoundingCounting} (2017) as special cases. For the latter, which is the strongest of these bounds, the formulation in terms of matrices provides new insight. In particular, by using explicit formulas for a Jordan-like decomposition of the involved matrices, we achieve new tighter results for the asymptotic setting, where the number of layers of the same fixed width tends to infinity. 
\end{abstract}
\section{Introduction}
In recent time, artificial neural networks get increasingly important in state-of-the-art machine-learning technology. Their success as a machine-learning algorithm is based partly on their flexibility that allows a myriad of possible architectures and on efficient training algorithms and specialized hardware. The theoretical properties of these functions that are currently a field of active study. 

In this work, we focus a special type of feed-forward neural networks. They are a functions that are a composition of \emph{layer functions}. These layer functions map a real vector to another real vector of possibly different length by first applying an affine linear map and then sending each coordinate through a function, called \emph{activation function}. If the \emph{Rectifier Linear Unit} (ReLU) that maps $x\in\mathbb{R}$ to $\max(0,x)$ is used as activation function, the overall composition of the layer functions is a \emph{ReLU feed-forward neural network} and these are the functions we will study in this work.
Such a function has the interesting property that it is piece-wise affine linear. More precisely, there exists a finite number of convex subsets of the input space such that it is represented as an affine linear function on each of these subsets. The goal of this work is to find upper bounds on the number of these subsets in terms of the number $L$ of layer functions used in the composition and the \emph{widths}, i.e. the dimensions that the layer functions map from and to. 

Such bounds could be of potential use in the context of approximation theory. For example, they could be used to prove that a target piece-wise affine linear function cannot be exactly represented by a ReLU feed-forward neural network if a resulting upper bound on the number of affine linear regions is smaller than necessary for the target function. More generally, these bounds could be combined with results from approximation theory~\cite{piecewiseQuadratic} about how well a piece-wise affine linear function with a certain number of affine linear regions can approximate other functions of specific properties at best. This way one could derive theoretical lower bounds on approximation errors of these neural networks, which in turn could be of interest for the analysis of generalization bounds for empirical risk minimization whose derivations usually involve such approximability properties \cite{Bartlett2006}. 

We present an abstract framework that allows the construction of such bounds. Its core idea is to push information about affine linear regions through the layers. This information consists of a sequence of numbers which represent a histogram of the dimensionalities of the images of the affine linear regions up to the current layer. For example, let the input width of the network be $n_0=2$ and let the first layer function $h_1$ have output dimension $n_1=3$. Then it could be that this layer function induces 7 regions in the input space $\mathbb{R}^2$ and on every of these regions, the first layer function $h_1$ is affine linear. Therefore, we can consider the dimensionality of the image or equivalently the rank of $h_1$ on every of these regions. This information can be represented by a sequence of natural numbers. Since we start with the input space $\mathbb{R}^2$, the initial histogram would be $(0,0,1,0,\dots)$ because the indexing starts with $0$ and there is only one region of dimension $2$ which is the whole input space $\mathbb{R}^2$. The histogram after $h_1$ was applied could then be $(1,3,3,0,\dots)$. 

Of course, this histogram depends on the weights for the affine linear map of the layer function $h_1$ but one can do a worst-case analysis. To do so, one needs to introduce an order relation on the set of these histograms. Given worst-case histogram bounds for ReLU layer functions of every input and output dimension, one can push these worst-case bounds on the histograms of the region image dimensionalities through the layers of the network. In a final step we sum up all the entries, i.e. take the $L1$ norm of the last histogram to obtain a bound on the number of affine linear regions. It turns out that the transition of a worst-case histogram from one layer to the next can be written as a linear map such that our main result can be written in terms of matrices: For a ReLU feed-forward neural network with input dimension $n_0$ and $L$ layers of output widths $n_1,\dots,n_L$, the number of affine linear regions is bounded by
\begin{equation}
  \label{eq:boundIntroduction}
 \|B^{(\gamma)}_{n_L}M_{n_{L-1},n_L}\dots B^{(\gamma)}_{n_1}M_{n_{0},n_{1}}e_{n_{0}+1}\|_1,
\end{equation}
where $B^{(\gamma)}_{n_1},\dots,B^{(\gamma)}_{n_L}$ are square upper triangular matrices of dimensions $n_1,\dots,n_L$ specific to a parameterization $\gamma$ related to the worst-case histogram bounds above and the $n_{i-1}\times n_{i}$ matrices $M_{n_{i-1},n_i}$, $i\in\left\{ 1,\dots,L \right\}$ serve the purpose of connecting inputs and outputs of different dimensionality. The vector $e_{n_0+1}$ is the unit vector in $\mathbb{R}^{n_0+1}$ with zeros at indices $1,\dots,n_0$ and value $1$ at index $n_0+1$. 

It turns out that all known concrete bounds from \cite{Montufar:2014:NLR:2969033.2969153}, \cite{Montufar17} and~\cite{DBLP:BoundingCounting} can be derived as special cases of our bound~\eqref{eq:boundIntroduction} for appropriate parametrization $\gamma$. The matrix representation that our framework yields is different from the existing representation and is very useful because the eigenvalues of the involved matrices can be read directly from the diagonal. For the strongest of the above bounds from~\cite{DBLP:BoundingCounting}, it is even possible to find a Jordan-like decomposition for the matrices $(B^{(\gamma)}_{n'})_{n'\in\mathbb{N}_+}$. This fact allows us to enter asymptotic settings. As an illustration, we consider the case where the input dimension $n_0\in\mathbb{N}_{+}$ is arbitrary and the dimensions of the other layers $n_1,\dots,n_L$ are fixed to be equal to $n\in\mathbb{N}_{+}$. The number of layers $L$ is variable. We provide a new explicit analytical formulae where previously only a weaker bound based on a Stirling approximation was known, see~\cite{DBLP:BoundingCounting}. In particular, when the input dimension is also equal to $n$, i.e. all $L$ layers have the same input and output width, for odd $n$ and $L\to\infty$ we achieve an asymptotic order $\mathcal{O}(2^{L (n-1)})$ compared to the order $ \mathcal{O}\left( 2^{L\left( n-1/2+\log_2\left( 1+1/\sqrt{\pi n} \right)/2 \right)} \right)$ from~\cite{DBLP:BoundingCounting}. This means that our new bound gains a half dimension in each layer in this setting. 

However, the use of our framework is not limited to the above results. We also explain how our theory can be exploited to derive further stronger bounds. For this, a combinatorial and geometrical problem needs to be solved to find a specific parametrization $\gamma$. 

This article is structured as follows. Section~\ref{sec:preliminaries} states basic definitions and results needed for the construction of our framework. Section~\ref{sec:existingBounds} gives an overview of existing bounds on the number of linear regions of ReLU feed-forward neural networks. In Section~\ref{sec:framework}, we derive and explain our main result and show how the bounds from Section~\ref{sec:existingBounds} can be derived as special cases. Furthermore, we compare them in an asymptotic setting and obtain a new tighter result. Finally, we note how our framework can be used to derive new stronger results. Section~\ref{sec:summary} summarizes our findings. The proofs are deferred to the Appendix~\ref{sec:appendix}.

\section{Preliminaries}
\label{sec:preliminaries}
In this section, we will provide definitions and explain their motivation. First, we will focus on a single layer and later on multiple composed layers. As a convention, we will write $\mathbb{N}$ for the nonnegative integers and $\mathbb{N}_+$ for $\mathbb{N}\setminus\left\{ 0 \right\}$.
Furthermore, let $\textnormal{diag}( n_1,\dots,n_k)$ be the $k\times k$ diagonal matrix with values $n_1,\dots,n_k$ on its diagonal and $I_{n}$ be the $n\times n$ identity matrix. The indicator function of a set $A$ will be denoted by $\mathds{1}_A$.
\subsection{One layer}
\label{sec:AnalysisOneLayer}
The \emph{ReLU activation function} is $\sigma: \mathbb{R}\to\mathbb{R},\; x\mapsto \max(0,x)$. For $n,n'\in\mathbb{N}_{+}$, we will call $h:\mathbb{R}^{n}\to \mathbb{R}^{n'}$ a \emph{ReLU Layer function} with weight matrix $W^{(h)}\in\mathbb{R}^{n'\times n}$ and bias vector $b^{(h)}\in\mathbb{R}^{n'}$ if it has the form 
\begin{equation}
  h:\mathbb{R}^{n}\to\mathbb{R}^{n'}, x\mapsto \left( \sigma(\langle x,w^{(h)}_{i}\rangle+b_i^{(h)}) \right)_{i\in\left\{ 1,\dots,n' \right\}},
\end{equation}
  where $w_i^{(h)}$ is the $i$-th row of $W^{(h)}$ for $i\in\left\{ 1,\dots,n' \right\}$. We define the set of such functions by
  \begin{figure}
  \centering
\begin{tikzpicture}[
plain/.style={
  draw=none,
  fill=none,
  },
net/.style={
  matrix of nodes,
  nodes={ draw,
    circle,
    inner sep=10pt
    },
  nodes in empty cells,
  column sep=-0.5cm,
  row sep=-9pt
  },
>=latex
]
\matrix[net] (mat)
{
  |[plain]|  $\mathbb{R}^{n}$ & |[plain]|$\overset{h}{\longrightarrow}$&|[plain]|$\mathbb{R}^{n'}$ \\
													& |[plain]|& 			\\
													|[plain]|$\vdots$& |[plain]|&|[plain]|$\vdots$\\
													& |[plain]|& 		\\    };
\foreach \ai in {2,4}
{\foreach \aii in {2,4}
{
  \draw[->] (mat-\ai-1) -- (mat-\aii-3);
}
}
\end{tikzpicture}
\caption{A function $h\in \textnormal{RL}(n,n')$ maps between the spaces $\mathbb{R}^{n}$ and $\mathbb{R}^{n'}$. This corresponds to a fully connected layer.}
  \label{fig:oneLayer}
\end{figure}
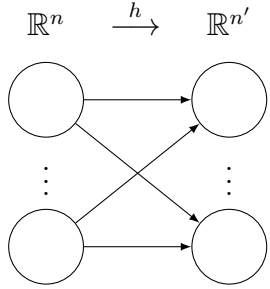
  \begin{equation}
	\label{eq:RL}
	\text{RL}(n,n'):=\left\{ h:\mathbb{R}^{n}\to\mathbb{R}^{n'}\mid h\text{ is a ReLU Layer function} \right\}.
  \end{equation}
In the sequel, we will assume $n,n'\in\mathbb{N}_{+}, h\in\textnormal{RL}(n,n')$.
For $x\in \mathbb{R}^{n}$ and $i\in\left\{ 1,\dots,n' \right\}$, we say that the $i$-th unit of $h\in\textnormal{RL}(n,n')$ is \emph{active} if $\langle x, w_i\rangle +b_i>0$. In the input domain $\mathbb{R}^{n}$, the subsets where a unit is active and inactive are separated by the $(n-1)$-dimensional hyperplanes
  \begin{equation}
	H_i^{(h)}:=\left\{ x\in\mathbb{R}^{n}\mid \langle x,w_i^{(h)}\rangle+b^{(h)}_i=0  \right\}\subset \mathbb{R}^{n}\quad \textnormal{ for } i\in\left\{ 1,\dots,n' \right\}.
  \end{equation}
Obviously, these $n'$ hyperplanes partition the space $\mathbb{R}^{n}$ into at most $2^{n'}$ regions $R_{h}(s), s\in\left\{ 0,1 \right\}^{n'}$ as defined below.  
\begin{definition}
  For $x\in\mathbb{R}^{n}$, we define the \emph{signature} $S_h(x)\in\left\{ 0,1 \right\}^{n'}$ of $x$ by
  \label{def:Signat}
  \begin{equation*}
	S_h(x)_i=
	\begin{cases}
	  1\text{ if } \langle x, w_i^{(h)}\rangle +b_i>0\\
	  0\text{ if } \langle x, w_i^{(h)}\rangle +b_i\le0
	\end{cases}\quad \text{ for }i\in\left\{ 1,\dots,n' \right\}
  \end{equation*}
\end{definition} The signature $S_h(x)$ tells us which units of $h$ are active for a specific input value $x\in\mathbb{R}^{n}$. Now we define the set of inputs that have a specific signature.
\begin{definition}
  \label{def:SigSReg}
For $s\in\left\{ 0,1 \right\}^{n'}$, the \emph{region $R_h(s)\subset\mathbb{R}^{n}$} corresponding to the signature $s$ is the set
\bibliographystyle{ieeetr}
  \begin{equation*}
	R_h(s):=\left\{ x\in\mathbb{R}^{n} \;\vert\;S_h(x)=s \right\}.
  \end{equation*}
\end{definition} Lemma~\ref{lem:Convex} from the appendix shows that these regions are always convex subsets of the input domain $\mathbb{R}^{n}$.
The Figures~\ref{fig:ex2a} and \ref{fig:ex2b} illustrate the definitions for two different $h\in\textnormal{RL}(2,3)$.
\begin{figure}[htbp]
  \centering
  \includegraphics[width=0.985\textwidth]{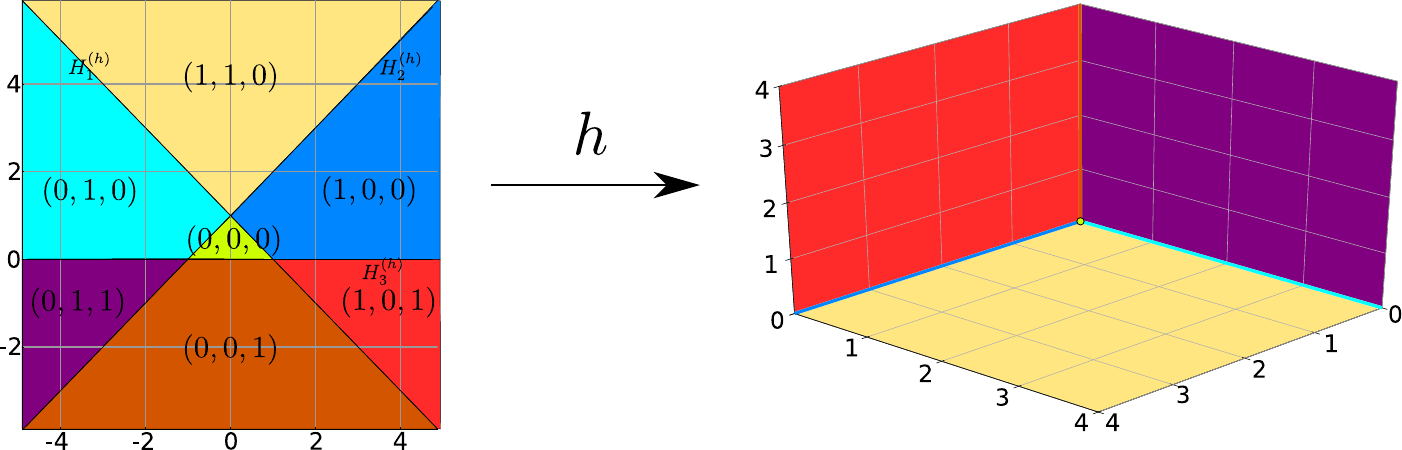}
  \caption{For $h\in\text{RL}(2,3)$ with
	$W^{(h)}=\tiny\protect\begin{pmatrix} 
	  \tfrac{1}{\sqrt{2}}&\tfrac{1}{\sqrt{2}}\\
	  -\tfrac{1}{\sqrt{2}}&\tfrac{1}{\sqrt{2}}\\
	   0&-1
\protect\end{pmatrix}$
and 
$b^{(h)}= 
\tiny
  \protect\begin{pmatrix}
	-\tfrac{1}{\sqrt{2}}\\
	-\tfrac{1}{\sqrt{2}}\\
	0
  \protect\end{pmatrix}
  $ we get seven regions in the input space on the left. Each such region is labeled with its signature $s\in\left\{ 0,1 \right\}^{3}$ and colored for better visibility. The same colors are used in the projection of the image of $h$ on the right side. For $s
  \in\left\{ 0,1 \right\}^3$, the region $R_{h}(s)$ on the input space is mapped to a point, a subset of a line or a subset of a plane in $\mathbb{R}^{3}$, depending on the number of active units $|s|$.
}
  \label{fig:ex2a}
\end{figure}
\begin{figure}[htbp]
  \centering
  \includegraphics[width=0.985\textwidth]{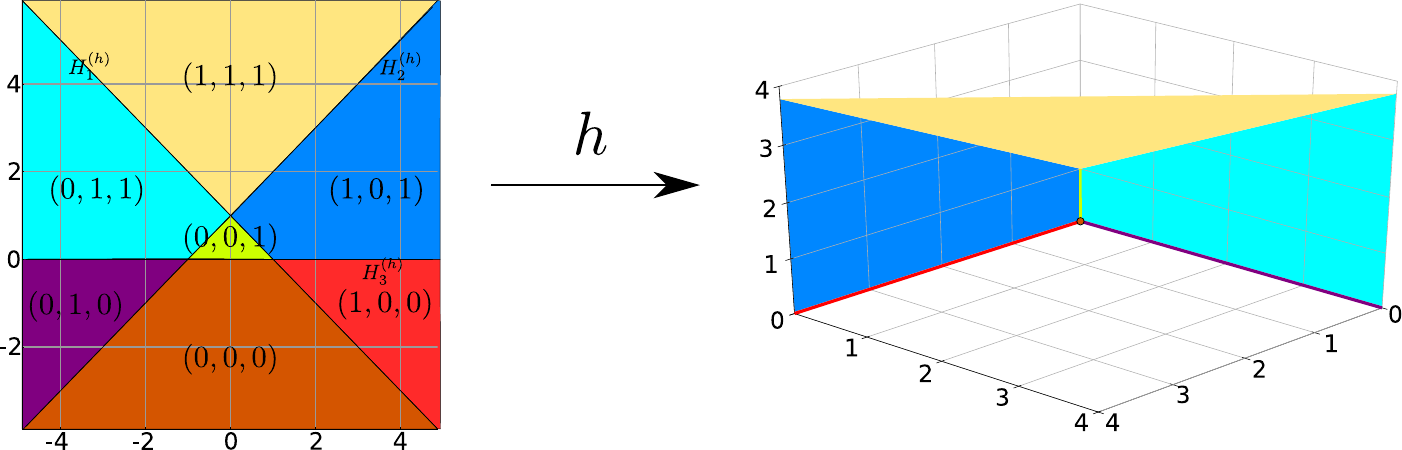}
  \caption{For $h\in\text{RL}(2,3)$ with 
	$W^{(h)}=\tiny\protect\begin{pmatrix} 
	   \tfrac{1}{\sqrt{2}}&\tfrac{1}{\sqrt{2}}\\
	   -\tfrac{1}{\sqrt{2}}&\tfrac{1}{\sqrt{2}}\\
	   0&1
\protect\end{pmatrix}$ and 
$b^{(h)}= 
\tiny
  \protect\begin{pmatrix}
	-\tfrac{1}{\sqrt{2}}\\
	-\tfrac{1}{\sqrt{2}}\\
	0
\protect\end{pmatrix}$, the image looks very different from Figure~\ref{fig:ex2a}
. Even though the yellow area $R_h(s)$ corresponds to $s=(1,1,1)$ with $|s|=3$ active neurons, the dimension of its image $h(R_h(s))$ is $2$ because it is bounded by the input dimension $2$.
}
  \label{fig:ex2b}
\end{figure}
\begin{definition}
  \label{def:AttainedSig}
We denote the \emph{attained signatures} of $h$ by $\mathcal{S}_h:=\left\{ S_h(x)\in\left\{ 0,1 \right\}^{n'}\;\vert\;x\in\mathbb{R}^{n} \right\}$.
\end{definition}This can also be written as the set of all signatures such that the corresponding region is non-empty, i.e. $ \mathcal S_h=\left\{ s\in\left\{ 0,1 \right\}^{n'}|R_h(s)\neq\left\{  \right\} \right\}$.
\subsection{Multiple Layers}
\label{sec:multipleLayers}

We will denote the number of layers excluding the input layer by $L\in\mathbb{N}_{+}$ and the dimension of the input layer by $n_0$ . The dimensions of the other layers are denoted by $n_1,\dots,n_L\in\mathbb{N}_{+}$. For convenience, we will write $\mathbf{n}=(n_1,\dots,n_L)\in\mathbb{N}_{+}^{L}$ and define
  \begin{equation}
	\text{RL}(n_0,\mathbf{n}):=\text{RL}(n_0,n_1)\times\text{RL}(n_1,n_2)\times\dots\times\text{RL}(n_{L-1},n_L).
  \end{equation} 
  Usually, the last layer of a feed-forward neural network is an affine linear map without activation function. However, since such a final additional affine linear map obviously does not increase the number of regions that a neural network with ReLU activation functions is affine linear on, they are not important for the construction of corresponding upper bounds and hence, we define our networks of interest as above without a final affine linear map.
Throughout this section, we will assume
\begin{equation}
  n_0\in\mathbb{N}_{+}, L\in\mathbb{N}_+, \mathbf{n}=(n_1,\dots,n_L)\in\mathbb{N}_+^{L}, \mathbf{h}\in\textnormal{RL}(n_0,\mathbf{n}).
  \label{eq:MultiLayerSetting}
\end{equation}
  
  Note that for $\mathbf{h}=(h_1,\dots,h_L)\in\text{RL}(n_0,\mathbf{n})$, the functions $h_1,\dots,h_L$ can be composed. We will denote this composition by $f_{\mathbf{h}}$, i.e.
  \begin{equation}
	\label{eq:compositionf}
	f_{\mathbf{h}}:
	\begin{cases}
	  \mathbb{R}^{n_0}&\to\mathbb{R}^{n_L}\\
	  x&\mapsto {h}_L\circ\dots\circ {h}_1 (x).
	\end{cases}
  \end{equation}
  Such functions are the multilayer feed-forward neural networks we are analyzing in this article.
  Figure~\ref{fig:multi} visualizes this setting. For $l\in\left\{ 1,\dots,L \right\}$, the matrix and bias vector corresponding to $h_l$ are denoted by $W^{(h_l)}$ with rows $w^{(h_l)}_1,\dots,w^{(h_l)}_{n_l}$ and $b^{(h_l)}$.
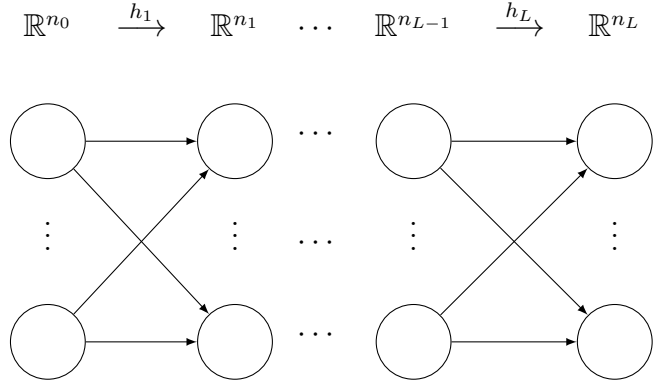
\begin{figure}[h]
  \centering
\begin{tikzpicture}[
plain/.style={
  draw=none,
  fill=none,
  },
net/.style={
  matrix of nodes,
  nodes={
    draw,
    circle,
    inner sep=10pt
    },
  nodes in empty cells,
  column sep=-0.5cm,
  row sep=-9pt
  },
>=latex
]
\matrix[net] (mat)
{
  |[plain]|  $\mathbb{R}^{n_0}$ & |[plain]|$\overset{h_1}{\longrightarrow}$&|[plain]| \centering $\mathbb{R}^{n_1}$ & |[plain]|$\cdots$&|[plain]| $\mathbb{R}^{n_{L-1}}$ & |[plain]|$\overset{h_{L}}{\longrightarrow}$&|[plain]| $\mathbb{R}^{n_L}$ 	\\
													& |[plain]|& 				&|[plain]|$\cdots$&	        &|[plain]|&								\\
|[plain]|$\vdots$ 									& |[plain]|&|[plain]|$\vdots$	&|[plain]|$\cdots$&|[plain]|$\vdots$&|[plain]|&|[plain]|$\vdots$ \\
													& |[plain]|& 				&|[plain]|$\cdots$&		    &|[plain]|&							\\    };

\foreach \ai in {2,4}
{\foreach \aii in {2,4}
{
  \draw[->] (mat-\ai-1) -- (mat-\aii-3);
  \draw[->] (mat-\ai-5) -- (mat-\aii-7);
}
}
\end{tikzpicture}
\caption{The functions $h_1,\dots,h_L$ that map between the spaces $\mathbb{R}^{n_0},\dots,\mathbb{R}^{n_L}$ are the fully-connected layer functions of their composition $f_{\mathbf{h}}=h_L\circ\dots\circ h_1$. }
  \label{fig:multi}
\end{figure}
For each $x$ in the input space $\mathbb{R}^{n_0}$, there are some units active in the neural network $f_{\mathbf{h}}$. We formalize this idea in the following definition.

\begin{definition}
  \label{def:MultiSignat}
  For $x\in\mathbb{R}^{n_0}$, we define the \emph{multi signature} $S_{\mathbf{h}}(x)\in \left\{ 0,1 \right\}^{n_1}\times\cdots\times\left\{ 0,1 \right\}^{n_L}$ of $x$ by
  \begin{equation*}
	S_{\mathbf{h}}(x)=\left( S_{{h}_1}(x),S_{{h}_2}\left( {h}_1(x) \right),\dots,S_{{h}_L}\left( {h}_{L-1}\circ\dots\circ {h}_1(x) \right) \right).
  \end{equation*}
\end{definition} Note that this multi signature is just a $L$-tuple of the signatures of the $L$ single layers from Definition~\ref{def:Signat} evaluated at their respective input. In analogy to Definition~\ref{def:SigSReg}, the input space $\mathbb{R}^{n_0}$ can be divided into regions indexed by multi signatures.
\begin{definition}
  \label{def:MultiSignatReg}
  For $s\in \left\{ 0,1 \right\}^{n_1}\times\cdots\times\left\{ 0,1 \right\}^{n_L}$, define the \emph{multi signature $s$ region $R_{\mathbf{h}}(s)\subset \mathbb{R}^{n_0}$} by
  \begin{equation*}
	\label{eq:FunF}
	R_{\mathbf{h}}(s):=\left\{ x\in\mathbb{R}^{n_0}\;\vert\; S_{\mathbf{h}}(x)=s\right\}.
  \end{equation*}
\end{definition}
\begin{definition}
  \label{def:AttainedMultiSig}
  We define the \emph{attained multi signatures of $\mathbf{h}$} as
  \begin{equation*}
	\mathcal{S}_{\mathbf{h}}:=\left\{ S_{\mathbf{h}}(x)\in \left\{ 0,1 \right\}^{n_1}\times\cdots\times\left\{ 0,1 \right\}^{n_L}\;\vert\;x\in \mathbb{R}^{n_0}\right\}.	
  \end{equation*}
\end{definition}As before, we can write this as
$\mathcal{S}_{\mathbf{h}}=\left\{ s\in \left\{ 0,1 \right\}^{n_1}\times\cdots\times\left\{ 0,1 \right\}^{n_L} \;\vert\; R_{\mathbf{h}}(s)\neq\left\{  \right\}\right\}$. We are interested in this quantity because bounds on the number of attained multi signatures $|\mathcal{S}_{\mathbf{h}}|$ will also be bounds on the number of affine linear regions as we will see in the next section.

\subsection{The number of affine linear regions}
Let $\mathcal{P}$ denote the set of all possible partitions of $\mathbb{R}^{n_0}$ into connected subsets. For $n_0,n\in\mathbb{N}_+$ and a general function $f:\mathbb{R}^{n_0}\to\mathbb{R}^{n}$ we now define the number of affine linear regions of $f$. Let $\mathcal{P}_{f}$ be the partitions $P\in\mathcal P$ such that $f$ is affine linear on each element of $P$, i.e.
  \begin{equation*}
	\mathcal{P}_{f}=\left\{ P\in\mathcal{P}\;\vert\; \forall R\in P\;f \text{ is affine linear on }R \right\}.
  \end{equation*}
  \begin{definition}
	\label{def:NumberAffineRegions}
The \emph{number of affine linear regions} $N_{f}$ of $f$ is
  \begin{equation*}
	N_{f}=\inf_{P\in\mathcal{P}_{f}}\;|P|.
  \end{equation*}This is the smallest number of elements a partition of $\mathbb{R}^{n_0}$ into connected subsets can have such that $f_{\mathbf{h}}$ is affine linear on every element of that partition. Note that by convention, $\inf\left\{  \right\}=\infty$ such that $N_f=\infty$ for a function that is not piece-wise affine linear.
\end{definition}
In the appendix, Corollary~\ref{cor:RepMultiLayers} and Lemma~\ref{lem:Convex}, we show that 
\begin{equation}
  f_{\mathbf{h}}=\sum_{s\in\mathcal{S}_\mathbf{h}}^{}\mathds{1}_{R_{\mathbf{h}}(s)}\tilde f_{\mathbf{h},s}
\end{equation}for affine linear functions $\tilde f_{\mathbf{h},s}$, $s\in\mathcal{S}_\mathbf{h}$ and convex, hence in particular connected, sets $R_{\mathbf{h}}(s)$, $s\in\mathcal{S}_{\mathbf{h}}$. But this shows that $\left\{ R_{\mathbf{h}}(s)|s\in\mathcal{S}_{\mathbf{h}} \right\}$ is such a partition in $\mathcal{P}_{f_{\mathbf{h}}}$ and therefore
\begin{equation}
  N_{f_\mathbf{h}}\le |\mathcal{S}_\mathbf{h}|.
\end{equation} In other words, the number of affine linear connected regions of $f_h$ is bounded by the number of attained multi signatures of $\mathbf{h}$ from Definition~\ref{def:AttainedMultiSig}.  We will use this fact and find bounds on $|\mathcal S_{\mathbf{h}}|$ which will also be bounds on $N_{f_\mathbf{h}}$.

\section{Existing bounds on the number of regions}
\label{sec:existingBounds}
Let $n_0, L, \mathbf{n}=(n_1,\dots,n_L), \mathbf{h}=(h_1,\dots,h_L)$ be as in equation~\eqref{eq:MultiLayerSetting}.  The most basic upper bound on the number of regions $N_{f_{\mathbf{h}}}$ of the ReLU feed forward neural network $f_{\mathbf{h}}=h_L\circ\dots\circ h_1$, is based on the sum of the layer widths $n_1,\dots,n_L$, c.f.~\cite{Montufar:2014:NLR:2969033.2969153}, Proposition~3:
\begin{equation}
  N_{f_h}\le2^{\sum_{l=1}^{L}n_l}
  \label{eq:MostBasicBound}
\end{equation} It is based on the idea that each of the $\sum_{l=1}^{L}n_l$ ReLU units can at most double the number of affine linear regions in the input space $\mathbb{R}^{n_0}$.
A result from 2017, \cite{pmlr-v70-raghu17a}, cf. Theorem~1 states that for fixed equal layer widths $n_1=\dots=n_L=:n\in\mathbb{N}_+$ and variable number of layers $L$ and input width $n_0\in\mathbb{N}_+$:
\begin{equation}
  \label{eq:RaguBound}
  N_{f_{\mathbf{h}}}=\mathcal{O}({n}^{Ln_0})
\end{equation}
Later in 2017, \cite{Montufar17} Proposition 3, showed the following upper bound for $n_0, \mathbf{n}, \mathbf{h}$, as in equation~\eqref{eq:MultiLayerSetting}:
\begin{equation}
  \label{eq:Montufarbound}
  N_{f_{\mathbf{h}}}\le\prod_{l=1}^L\sum_{j=0}^{\min(n_0,\dots,n_{l-1})}{n_{l}\choose j}
\end{equation}
It was noted there that this result can be improved by using a more detailed dimensionality analysis. This was done in the preprint~\cite{DBLP:BoundingCounting} from 2017 resulting in the bound
\begin{equation}
  \label{eq:BoundingCounting}
  N_{f_{\mathbf{h}}}\le\sum_{j_1,\dots,j_L\in J}^{}\prod_{l=1}^L{n_{l}\choose j},
\end{equation}where $J=\left\{ (j_1,\dots,j_L)\in\mathbb{N}^L\vert\;\forall l\in\left\{ 1,\dots,L \right\}:\;j_l\le\min(n_0,n_1-j_1,\dots,n_{L-1}-j_{L-1},n_L)\right\}$.
The above bounds are in the following hierarchy: The bound~\eqref{eq:BoundingCounting} implies~\eqref{eq:Montufarbound}, which implies both,~\eqref{eq:RaguBound} and~\eqref{eq:MostBasicBound}.
Using our framework, we can derive the bounds from equations~\eqref{eq:MostBasicBound}, \eqref{eq:Montufarbound} and~\eqref{eq:BoundingCounting} as special cases. 
The strongest of these bounds~\eqref{eq:BoundingCounting} as stated above is in a form that is not well-suited for explicit evaluations and asymptotic considerations. For example, the authors consider the special case where $L$ is variable but $n_1=\dots=n_L=:n\in\mathbb{N_+}$ and $n_0\in\mathbb{N}_+$ are fixed. In their Proposition 15, they lose precision by first setting $n_0=n$ and then using a Stirling approximation, arriving at
\begin{equation}
  \label{eq:BoundingCoundingWeakened}
  N_{f_{\mathbf{h}}}\le 2^{Ln}\left( \frac{1}{2}+ \frac{1}{2\sqrt{\pi n}}\right)^{L/2}\sqrt{2}.
\end{equation}
In contrast, using our framework makes it possible to find an analytic expression in this setting without the need to weaken the bound. This is possible because of the bound formulation in terms of matrices which can be decomposed for explicit evaluation of arbitrary powers that arise for equal fixed layer widths $n_1=\dots=n_L=n$, see Section~\ref{sec:Asymptotic}.

There also exist lower bounds on the number of affine linear regions that can be achieved. Usually, they are derived by explicit constructions of ReLU neural networks. The work~\cite{Montufar:2014:NLR:2969033.2969153} presents the bound
\begin{equation}
  \label{eq:lowerBound}
  \max_{\mathbf{h}\in\textnormal{RL}(n_0,\mathbf{n})}N_{f_\mathbf{h}}\ge \left( \prod_{l=1}^{L-1}\Bigl\lfloor \frac{n_l}{n_0}\Bigr\rfloor^{n_0} \right) \sum_{j=0}^{n_0}{n_L\choose j}
\end{equation}and it was noted in~\cite{Montufar17} that for fixed $n_0,L\in\mathbb{N}_+$ the quotient of the upper and the lower bounds from equations~\eqref{eq:Montufarbound} and~\eqref{eq:lowerBound} satisfies
\begin{equation}
  \label{eq:quotientOfBounds}
  \limsup_{n_1,\dots,n_L\to\infty}
  \frac{
  \prod_{l=1}^L\sum_{j=0}^{\min(n_0,\dots,n_{l-1})}{n_{l}\choose j}
}{
\left( \prod_{l=1}^{L-1}\Bigl\lfloor \frac{n_l}{n_0}\Bigr\rfloor^{n_0} \right) \sum_{j=0}^{n_0}{n_L\choose j}
}\le
  \limsup_{n_1,\dots,n_L\to\infty}
\prod_{l=1}^{L-1} 
  \frac{
	\sum_{j=0}^{n_0}{n_{l}\choose j}
}{
  \left( \frac{n_l}{n_0}\right)^{n_0}
}\le e^{n_0 (L-1)}.
\end{equation}
A similar lower bound can be found in~\cite{DBLP:BoundingCounting}. One of the main goals of the analysis of the number of affine linear regions is to improve the results such that upper and lower bounds match or at least are of the same order for some asymptotic scenarios because otherwise, it is not clear if the bounds are sharp. Our framework is not concerning such lower bounds but by improving the understanding of such upper bounds makes a further step towards this goal.

	\section{Framework for the construction of upper bounds}
	\label{sec:framework}
\subsection{Intuitive Motivation}
\label{sec:Motivation}

To illustrate the ideas and problems involved in the construction of the bounds in Section~\ref{sec:existingBounds}, we reprove the bound~\eqref{eq:Montufarbound} of Mont\'ufar,
\begin{equation}
	\label{eq:firstUpperBound}
	|\mathcal{S}_{\mathbf{h}}| \le\prod_{l=1}^L\sum_{j=0}^{\min(n_0,\dots,n_{l-1})}{n_{l}\choose j}.
\end{equation}
It is important to understand the ideas behind this bound because our results are derived from a more fine-grained analysis starting from the same basis. This way, we can point out clearly, where our analysis differs. Here, we only state the ideas, the proofs are deferred to Appendices~\ref{app:firstBound} and~\ref{app:secondBound}.

We introduce definitions for the first $l$ entries of elements of the set of attained multi signatures $\mathcal{S}_{\mathbf{h}}$. Recall that these entries are tuples of lengths $n_1,\dots,n_l$.
\begin{definition}
  \label{def:Signatures}Let
  \begin{eqnarray*}
	\mathcal{S}^{(1)}_{\mathbf{h}}&:=&\left\{ (s_1)\in\left\{ 0,1 \right\}^{n_1}\;\vert\;(s_1,\dots,s_L):=s\in\mathcal{S}_{\mathbf{h}}\right\}\\
	\mathcal{S}^{(2)}_{\mathbf{h}}&:=&\left\{ (s_1,s_2)\in\left\{ 0,1 \right\}^{n_1}\times\left\{ 0,1 \right\}^{n_2}\;\vert\;(s_1,\dots,s_L):=s\in\mathcal{S}_{\mathbf{h}}\right\}\\
	&\vdots&\\
	\mathcal{S}^{(L)}_{\mathbf{h}}&:=&\left\{ (s_1,\dots,s_{L})\in\left\{ 0,1 \right\}^{n_1}\times\dots\times\left\{ 0,1 \right\}^{n_{L}}\;\vert\;(s_1,\dots,s_L):=s\in\mathcal{S}_{\mathbf{h}}\right\}=\mathcal{S}_{\mathbf{h}}.
  \end{eqnarray*}
\end{definition}

The Mont\'ufar bound is based on an anchor inequality and a recursion inequality which is weakened and then unpacked to obtain~\eqref{eq:firstUpperBound}. 
\begin{itemize}
  \item The anchor inequality is based on the well-known fact by Zaslavsky from 1975, that the $n_1$ hyperplanes which are induced by the first layer function $h_1$ can partition the input space $\mathbb{R}^{n_0}$ into at most $\sum_{j=0}^{n_0}{n_1\choose j}$ regions, see \cite{1975Zaslavsky}. We state this result in the appendix in Lemma~\ref{lem:main}:
  \begin{equation}
	\label{eq:MainAnchor}
	\vert\mathcal{S}_{\mathbf{h}}^{(1)}|\le\sum_{j=0}^{n_0}{n_1\choose j}.\\
  \end{equation}
  This result gives a bound on the number of attained signatures of the function $h_1$ which maps the input layer to the first layer. 

  \item 
   The recursion inequality is stated in Theorem~\ref{thm:main}: 
  \begin{equation}
	\label{eq:MainRecursive}
	|\mathcal{S}^{(l+1)}_{\mathbf{h}}|\le\sum_{\left( s_1,\dots,s_l \right)\in\mathcal{S}^{(l)}_{\mathbf{h}}}^{}\sum_{j=0}^{\min\left( n_0,|s_1|,\dots,|s_l| \right)} {n_{l+1}\choose j}\quad \textnormal{ for } l\in\left\{ 1,\dots,L-1 \right\},
\end{equation}using the notation $|s_l|=\sum_{i=0}^{n_l} (s_l)_i$, $l \in\left\{1,\dots,L  \right\}$. This kind of recursive relation is not well-suited to unpack because it only relates the \emph{number} of attained multi signatures in $\mathcal{S}_{\mathbf{h}}^{(l+1)}$ with the \emph{type} of multi signatures in $\mathcal{S}^{(l)}_\mathbf{h}$ for $l\in\left\{ 1,\dots,L-1 \right\}$. More precisely, the summation range of the inner sum depends on the activations of the neurons in the layers $1$ to $l$. 
  If we weaken the above recursive relation~\eqref{eq:MainRecursive} by replacing $\min(n_0,|s_1|,\dots,|s_l|)$ by $\min(n_0,\dots,n_l)$, the inner sum does not depend on the outer sum anymore such that we get a relation $|\mathcal{S}^{(l+1)}_{\mathbf{h}}|\le|\mathcal{S}^{(l)}_{\mathbf{h}}|\sum_{j=0}^{n_l^*}{n_{l+1}\choose j}$ with $n_l^*=\min(n_0,\dots,n_l)$ for $l\in\left\{ 1,\dots,L-1 \right\}$, see Corollary~\ref{cor:firstUpperBound}. Now equation~\eqref{eq:firstUpperBound} follows by induction.
\end{itemize}

In our framework we want to avoid the above described required weakening of the recursive relation. We find a recursive relation similar to those in equations~\eqref{eq:MainAnchor} and~\eqref{eq:MainRecursive} with the important difference that we do not bound the number of elements in $\mathcal{S}^{(l)}_{\mathbf{h}}$ but a histogram of the set $\left\{ \min\left( n_0,|s_1|,\dots,|s_l| \right)\mid\vert s\in\mathcal{S}^{(l)}_{\mathbf{h}}\right\}$ denoted by $\tilde{\mathcal{H}}^{(l)}( \mathcal{S}_{\mathbf{h}}^{(l)})$, $l\in\left\{ 1,\dots,L \right\}$ which we will refer later as the \emph{dimension histogram}. We call these quantities dimension histograms because for every region $R_{\mathbf{h}}(s)$ with multi signature $s\in\mathcal{S}_{h}$ the dimension or rank of $h_l\circ\dots\circ h_1$ is bounded by $\min\left( n_0,|s_1|,\dots,|s_l| \right)$. With a particular order relation ``$\preceq$'' on the set of histograms, we then prove the following anchor and recursion relations:
  \begin{align}
	\tilde{\mathcal{H}}^{(1)}\left( \mathcal{S}_{\mathbf{h}}^{(1)} \right)&\preceq\varphi^{(\gamma)}_{n_1}({\rm e}_{n_0})\\
	\tilde{\mathcal{H}}^{(l+1)}\left(\mathcal{S}_{\mathbf{h}}^{(l+1)}\right)&\preceq \varphi^{(\gamma)}_{n_{l+1}}(\tilde{\mathcal{H}}^{(l)}(\mathcal{S}_{\mathbf{h}}^{(l)}))\quad\textnormal{ for } l\in\left\{ 1,\dots,L-1 \right\}.
  \end{align}
  In the above inequalities, ${\rm e}_{n_0}$ is the histogram that contains a $1$ at index $n_0$ and $0$ else and represents the input space $\mathbb{R}^{n_0}$. The functions $\varphi^{(\gamma)}_{n_{l}}$, $l\in\left\{ 1,\dots,L \right\}$ are transitions from an input histogram to an output histogram. These transition functions have a monotonicity property such that smaller inputs yield smaller outputs. Therefore, we can directly unpack the recursive relation without the need to weaken it as above and conclude $\tilde{\mathcal{H}}^{(L)}(\mathcal{S}_{\mathbf{h}}^{(L)}  )\preceq \varphi^{(\gamma)}_{n_L}\circ\dots\circ\varphi^{(\gamma)}_{n_1}(e_{n_0})$ from which we derive a bound on $|\mathcal{S}_{\mathbf{h}}|$. The above described ideas are made more formal in the next sections.
  
	\subsection{Formal description}
	\label{sec:Definitions}
	We first define the notion of a ``histogram'' formally. They are elements of the set defined below.
\begin{definition}
  \label{def:V}
  Let $V=\left\{ x\in\mathbb{N}^{\mathbb{N}}\;\middle|\;\|x\|_1=\sum_{j=0}^{\infty}x_j<\infty \right\}$.
\end{definition}
A histogram is therefore a sequence of natural numbers with indices starting with $0$ with only finitely many values different from zero. We give names ${\rm e}_i\in V$, $i\in\mathbb{N}$ to the histograms defined by the property $({\rm e}_i)_j=\delta_{ij}$ for $i,j\in\mathbb{N}$. They shall not be confused with the unit vectors $e_i\in\mathbb{R}^j$ for $i\le j\in \mathbb{N}$.
 Obviously, the element-wise sum $v_1+v_2$ of two histograms $v_1, v_2\in V$ and element-wise scalar multiplication $a v_1$ for $a\in \mathbb{N}$, $v_1\in V$ are again histograms in $V$.

 As explained above, to state our anchor inequality and recursive relation we need to be able to compare two histograms. We therefore introduce the following order relation. 
\begin{definition}
  \label{def:preceq}
  For $v,w\in V$, let
  \begin{equation*}
	v \preceq w\quad :\iff \forall J\in\mathbb{N}: \sum_{j=J}^{\infty}v_j\le\sum_{j=J}^{\infty}w_j
  \end{equation*}
\end{definition} Figure~\ref{fig:preceq} gives an intuition of this order relation. 
\begin{figure}[htbp]
  \centering
  \includegraphics[width=0.5\textwidth]{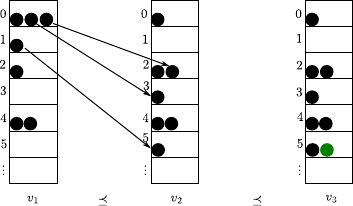}
  \caption{Intuitively, if we imagine a histogram $v\in V$ as balls in boxes indexed by $\mathbb{N}$ as above, for $v,w\in V$, $v\preceq w$ if and only if $w$ can be obtained from $v$ by moving balls into boxes with higher indices and adding new balls. Above, $v_1=3{\rm e}_0+{\rm e}_1+{\rm e}_2+2{\rm e}_4$, $v_2={\rm e}_0+2{\rm e}_2+{\rm e}_3+2{\rm e}_4+{\rm e}_5$ and $v_3=v_2+{\rm e}_5$. }
  \label{fig:preceq}
\end{figure}
Lemma \ref{lem:partialOrder} states that $\preceq$ is a partial order on $V$, which follows immediately from the definitions. Two histograms $u,v\in V$ are not always comparable, i.e. it does necessarily hold that either $u\preceq v$ or $v\preceq u$. However, we can always find a third histogram $w\in V$ that dominates both, i.e. that satisfies $u\preceq w$ and $v\preceq w$. The maximum of $u$ and $v$ as defined below is such an element.
\begin{definition}
  \label{def:maxpreceq}
  For a collection $(v^{(i)})_{i\in I}$ of histograms in $V$ indexed by a finite set $I$, define $\max_{i\in I}(v^{(i)})\in V$ on its entries by
  \begin{equation}
	\label{eq:maxPreceqFinite}
	\max_{i\in I}(v^{(i)})_J=\max_{i\in I}(\sum_{j=J}^{\infty}v^{(i)}_j)-\max_{i\in I}(\sum_{j=J+1}^{\infty}v^{(i)}_j)\quad \textnormal{ for }J\in \mathbb{N}
  \end{equation}
\end{definition} This is well-defined and in $V$ because only finitely many summands in these sums are non-zero by Definition~\ref{def:V}. Obviously, 
equation \eqref{eq:maxPreceqFinite} implies 
  \begin{equation*}
	\forall J\in\mathbb{N}: \max_{i\in I}(\sum_{j=J}^{\infty}v^{(i)}_j)= \sum_{j=J}^{\infty}\max_{i\in I}(v^{(i)})_j.
  \end{equation*}
  Hence, $v^{(i')}\preceq\max_{i\in I}(v^{(i)})$ for every $i'\in I$ and for any other $w\in V$ that also satisfies $v^{(i')}\preceq w$ for all $i'\in I$, it holds that $\max_{i\in I}(v^{(i)})\preceq w$. Therefore, the maximum of some histograms in $V$ is their smallest dominating histogram in $V$.  Figure~\ref{fig:preceqmax} relates Definition~\ref{def:maxpreceq} to the intuition from Figure~\ref{fig:preceq}.
\begin{figure}[htbp]
  \centering
  \includegraphics[width=0.5\textwidth]{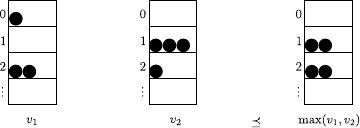}
  \caption{The visualization corresponding to the intuition from Figure~\ref{fig:preceq} for the maximum $\max(v_1,v_2)=2{\rm e}_1+2{\rm e}_2$ of $v_1={\rm e}_0+2{\rm e}_2$ and $v_2=3{\rm e}_1+{\rm e}_2$.}
  \label{fig:preceqmax}
\end{figure}

In the sequel, we will use histograms to represent signatures and multi signatures in a way that is useful for the construction of our framework.

\begin{definition}
  \label{def:Hist}
  For $n'\in \mathbb{N}_+$, define the \emph{activation histogram}
  \begin{equation*}
	\mathcal{H}_{n'}:
	\begin{cases}
	  \mathcal{P}(\left\{ 0,1 \right\}^{n'})&\to V\\
	\mathcal{S}&\mapsto\left( \sum_{s\in \mathcal{S}}^{}\mathds{1}_{ \left\{ j \right\}}(|s|) \right)_{j\in\mathbb{N}}
	\end{cases}
  \end{equation*}
\end{definition}
The above function $\mathcal{H}_{n'}$ maps a set of signatures $\mathcal{S}\subset \left\{ 0,1 \right\}^{n'}$ to a histogram of $|s|=\sum_{i=1}^{n'}s_i$ for $s\in\mathcal{S}$. For $n,n'\in\mathbb{N}_+$, and a single layer $h\in\textnormal{RL}(n,n')$ the activation histogram $\mathcal{H}_{n'}(\mathcal{S}_{h})$ tells us how often a certain number of neurons is active. The reason why we are interested in the number of active neurons is the fact that the rank of $h$ on a region with $k$ active neurons cannot exceed $k$.

For the sets of multi signatures $\mathcal{S}^{(1)}_{\mathbf{h}},\dots,\mathcal{S}^{(L)}_{\mathbf{h}}$ from Definition~\ref{def:Signatures} we also need similar functions that map to a histogram.
In the sequel, let $\mathcal{P}(A)$ denote the power set of a set $A$, i.e. the set of all subsets of $A$.
\begin{definition}
  For $l\in\left\{ 1,\dots,L \right\}$, define the \emph{dimension histogram} by
  \begin{equation*}
	\tilde{\mathcal{H}}^{(l)}:
	\begin{cases}
	\mathcal{P}\left( \left\{ 0,1 \right\}^{n_1}\times \dots\times \left\{ 0,1 \right\}^{n_l} \right)
	  &\to V\\
	  U&\mapsto\left( \sum_{(s_1,\dots,s_l)\in U}^{}\mathds{1}_{ \left\{ j \right\}}\left( \min(n_0,|s_1|,\dots,|s_l|) \right) \right)_{j\in \mathbb{N}_{0}}
	\end{cases}.
  \end{equation*}
\end{definition}The above functions depend implicitly on $n_0,\mathbf{n}$ and map a set of multi signatures to a histogram in $V$. Note that $\tilde{\mathcal{H}}^{(l)}( \mathcal{S}_{\mathbf{h}}^{(l)} )$ is a histogram of $\left\{ \min(n_0,|s_1|,\dots,|s_l|)\vert (s_1,\dots,s_l)\in \mathcal{S}^{(l)}_{\mathbf{h}} \right\}$ for $l\in\left\{ 1,\dots,L \right\}$. For every $s\in\mathcal{S}^{(l)}_{\mathbf{h}}$ the minimum $\min(n_0,|s_1|,\dots,|s_l|)$ is an upper bound of the rank of the affine linear map that is computed by $h_l\circ\dots\circ h_1$ on input region with multi signature $s$. Therefore $\tilde{\mathcal{H}}^{(l)}( \mathcal{S}_{\mathbf{h}}^{(l)} )$ is an upper bound histogram for the histogram of the image dimensions of the induced regions for $h_l\circ\dots\circ h_1$, hence the name ``dimension histogram''.
These dimension histograms should not be confused with the activation histograms $\mathcal{H}_{n'}, n'\in\mathbb{N}_+$ from Definition~\ref{def:Hist} which act on single signatures and do not involve $n_0$.

The idea of our framework is to start with the histogram ${\rm e}_{n_0}$ representing the input space $\mathbb{R}^{n_0}$, which is one region with dimension $n_0$ and push this histogram through transition functions $\varphi^{(\gamma)}_{n_1},\dots,\varphi^{(\gamma)}_{n_L}$ corresponding to the layers of the network. The main ingredient for the recursion in Mont\'ufar's bound was Zaslavsky's result from Lemma~\ref{lem:OneLayer1975}, which bounds for every input and output dimension $n,n'\in\mathbb{N}$ the number of the attained signatures $\mathcal{S}_{h}$ of every $h\in\textnormal{RL}(n,n')$. In contrast, for our transition function we need to bound for every $n,n'\in\mathbb{N}$ and every $h\in\textnormal{RL}(n,n')$ the activation histogram of the attained signatures. This is formalized in the following definition.

\begin{definition}
  \label{def:boundCondition}
  We say that a collection $(\gamma_{n,n'})_{n'\in\mathbb{N}_+,n\in\left\{ 0,\dots,n' \right\}}$ of elements in $V$ satisfies the \emph{bound condition} if the following statements are true:
  \begin{enumerate}
	\item $\forall n'\in\mathbb{N}_+,n\in\left\{ 0,\dots,n' \right\}\quad \max\left\{ \mathcal{H}_{n'}(\mathcal{S}_h)\mid h\in \textnormal{RL}(n,n') \right\}\preceq \gamma_{n,n'}$
	\item $\forall n'\in \mathbb{N}_+,n,\tilde n\in\left\{ 0,\dots,n' \right\}\quad n\le \tilde n\implies \gamma_{n,n'}\preceq \gamma_{\tilde n,n'}$
  \end{enumerate}
  Here, we use the convention that for $n'\in\mathbb{N}_+$, $\textnormal{RL}(0,n')=\left\{ h:\left\{ 0 \right\}\to\mathbb{R}^{n'}, x\mapsto c |c\in\mathbb{R}\right\}$ and for $h\in\textnormal{RL}\left( 0,n' \right)$, $\mathcal{H}_{n'}(\mathcal{S}_h)={\rm e}_0$. 
  The set of all such $\gamma$ is denoted by 
  \begin{equation}
	\label{eq:Gamma}
	\Gamma:=\left\{ (\gamma_{n,n'})_{n'\in\mathbb{N}_+,n\in\left\{ 0,\dots,n' \right\}}\mid \gamma \text{ satisfies the bound condition}  \right\}
  \end{equation}
\end{definition}
Such collections exist, i.e. $\Gamma\neq \left\{  \right\}$ because the maximum in the first criterion of the bound condition is a maximum over finitely many activation histograms since $\mathcal{S}_h\subset\left\{ 0,1 \right\}^{n'}$. Hence, it is a well-defined element in $V$ by Definition~\ref{def:maxpreceq}. If the second criterion does not hold for a given collection $\gamma$, we can easily construct another collection $\gamma'$ via $\gamma'_{n,n'}=\max_{\tilde n\le n}\left(\gamma_{\tilde n,n'}  \right)$ for $n,n'$ as in Definition~\ref{def:boundCondition}. For this new collection $\gamma'$, the second criterion of the bound condition holds. We will see below why we need this second criterion. A collection $\gamma\in\Gamma$ gives a worst-case bound for all possible activation histograms of ReLU layer functions $h\in\textnormal{RL}(n,n')$ for input dimensions $n$ output dimensions $n'$ with $n\in\left\{ 0,\dots,n' \right\}$. For $n>n'$ we can reuse the histogram bound of $n=n'$ because for $n\ge n'$ the maximum $\max\left\{ \mathcal{H}_{n'}(\mathcal{S}_h)\mid h\in \textnormal{RL}(n,n') \right\}$ in $V$ is the same since all possible signatures can occur, i.e. there exists $h\in\textnormal{RL}(n,n')$ with $\mathcal{S}_h=\left\{ 0,1 \right\}^{n'}$, see Lemma~\ref{lem:gammaMin}. 

Note that for a function $h\in\textnormal{RL}(n,n')$ and an attained signature $s\in\mathcal{S}_{h}$, the dimension of $h$ on the corresponding region $R_h(s)$ is not only bounded by the number of active neurons but also by the input dimension $n$. To account for this fact we need a clipping function that acts on histograms by putting all weight from entries with index greater than a threshold $i^*$ to the index $i^*$.
\begin{definition}
  \label{def:alphafunction}
  For $i^*\in\mathbb{N}$ define the \emph{clipping function}
	$\textnormal{cl}_{i^*}:
	V\to V$,
	  $\left( v_i \right)_{i\in\mathbb{N}_n}\mapsto\left( \textnormal{cl}_{i^*}(v)_i \right)_{i\in\mathbb{N}_{0}}$ by 
\begin{equation*}
  \quad \textnormal{cl}_{i^*}(v)_{i}=
	\begin{cases}
	  v_i&\textnormal{ for }i<i^*\\
	  \sum_{j=i^*}^{\infty}v_{j}&\textnormal{ for }i=i^*\\
	  0&\textnormal{ for }i>i^*\\
	\end{cases}
\end{equation*}
\end{definition}Figure~\ref{fig:clipping} provides intuition on the clipping function. 
\begin{figure}[htbp]
  \centering
  \includegraphics[width=0.3\textwidth]{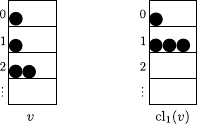}
  \caption{The histogram $v={\rm e}_0+{\rm e}_1+2{\rm e}_2$ is clipped at index $1$ by $\textnormal{cl}_1$ causing the entries with index greater than $1$ to be added to the entry with index $1$. The result is $\textnormal{cl}_1(v)={\rm e}_0+3{\rm e}_1$.}
  \label{fig:clipping}
\end{figure}
As described above, for an output dimension $n'\in\mathbb{N}$, an input dimension $n\in\left\{ 0,\dots,n' \right\}$ and a collection $\gamma\in\Gamma$, if we clip the bound $\gamma_{n,n'}$ on the activation histograms $\mathcal{H}_{n'}(\mathcal{S}_{h})$, $h\in \textnormal{RL}(n,n')$ at the input dimension $n$, the clipped version $\textnormal{cl}_{n}(\gamma_{n,n'})$ will bound the histogram corresponding to the dimensions of the output regions of $h$, i.e.
\begin{equation*}
  \left(\sum_{s\in\mathcal{S}_{h}}^{}\mathds{1}_{\left\{ i \right\}}(\textnormal{rank of $h$ on $R_h(s)$}) \right)_{i\in \mathbb{N}_{+}}\preceq \textnormal{cl}_{n}(\gamma_{n,n'})
\end{equation*}In other words $\textnormal{cl}_n(\gamma_{n,n'})$ is an upper bound on the histograms of output region dimensionalities that can arise from $h\in\textnormal{RL}(n,n')$. It makes therefore sense to define the transition function for output dimension $n'\in\mathbb{N}_+$ as follows.
 \begin{definition}
   \label{def:PhiFunc}
For $n'\in\mathbb{N}_+$, $\gamma\in\Gamma$ let
   \begin{equation*}
	 \varphi^{(\gamma)}_{n'}:
	 \begin{cases}
	   V&\to V\\
	   v&\mapsto \sum_{n=0}^{\infty}v_n \textnormal{cl}_{\min(n,n')}(\gamma_{\min(n,n'),n'})
	 \end{cases}.
   \end{equation*}
 \end{definition}
 The sum in the above definition is actually a finite sum because by Definition~\ref{def:V}, only finitely many entries of $v\in V$ are different from zero.
 Furthermore note that we reused $\gamma_{n',n'}$ for the case $n>n'$ as described above. 

 The intuition of such a transition function is as follows. For every input dimensionality we have a worst-case bound of what histograms can arise for the output region dimensions from an input region with that dimensionality when a ReLU layer function with $n'$ neurons is applied. If we sum up these bounds for several input regions of possibly different dimensions, what we get is exactly the map $\varphi^{(\gamma)}_{n'}$. Therefore, these maps allow to transform a worst-case bound on the dimension histogram $\tilde{\mathcal{H}}^{(l)}(\mathcal{S}^{(l)}_{\mathbf{h}})$ corresponding to region image dimensionalities of $h_l\circ\dots\circ h_1$ to a dimension histogram $\tilde{\mathcal{H}}^{(l+1)}(\mathcal{S}^{(l+1)}_{\mathbf{h}})$ corresponding to $h_{l+1}\circ\dots\circ h_1$ for $l\in\left\{ 1,\dots,L-1 \right\}$. Proposition~\ref{prop:main2} in the appendix states that for $\gamma\in\Gamma$
  \begin{align}
	\tilde{\mathcal{H}}^{(l+1)}\left(\mathcal{S}_{\mathbf{h}}^{(l+1)}\right)&\preceq \varphi^{(\gamma)}_{n_{l+1}}(\tilde{\mathcal{H}}^{(l)}(\mathcal{S}_{\mathbf{h}}^{(l)}))\quad\textnormal{ for } l\in\left\{ 1,\dots,L-1 \right\}.
  \end{align}
  Now the second criterion of the bound condition allows to conclude that $\varphi^{(\gamma)}_{n'}(v_1)\preceq \varphi^{(\gamma)}_{n'}(v_2)$ whenever $v_1\preceq v_2$ for $n'\in \mathbb{N}_+$ and $v_1,v_2\in V$, see Lemma~\ref{lem:monotonicityPhi}. We can thus derive 
  \begin{equation*}
  \tilde{\mathcal{H}}^{(L)}\left(\mathcal{S}_{\mathbf{h}}^{(L)}\right)\preceq \varphi^{(\gamma)}_{n_{L}}\circ\dots\circ\varphi^{(\gamma)}_{n_{2}}\circ\tilde{\mathcal{H}}^{(1)}(\mathcal{S}_{\mathbf{h}}^{(1)})).
\end{equation*} Together with Lemma~\ref{lem:main2}, which states $\tilde{\mathcal{H}}^{(1)}\left( \mathcal{S}_{\mathbf{h}}^{(1)} \right)\preceq\varphi^{(\gamma)}_{n_1}({\rm e}_{n_0})$ it follows that
\begin{equation}
  \label{eq:mainBound}
  |\mathcal{S}_{\mathbf{h}}|=\|\tilde{\mathcal{H}}^{(L)}\left(\mathcal{S}_{\mathbf{h}}^{(L)}\right)\|_1\le\|\varphi^{(\gamma)}_{n_{L}}\circ\dots\circ\varphi^{(\gamma)}_{n_{1}}\left( {\rm e}_{n_0} \right)\|_1\quad \textnormal{ for }\gamma\in\Gamma
\end{equation} because $v_1\preceq v_2$ implies $\|v_1\|_1\le \|v_2\|_1$ by Lemma~\ref{lem:preceqNorm}. This is our main result and it can also be stated in a matrix formulation. Despite the fact that histograms are sequences of infinite length, one can use vectors of finite length instead for the bound. For $n'\in\mathbb{N}_{+}$ the relevant information of the mapping $\varphi^{(\gamma)}_{n'}$ is in the clipped histograms $\textnormal{cl}_{0}(\gamma_{0,n'}),\dots,\textnormal{cl}_{n'}(\gamma_{n',n'})$ which have non-zero entries only in the first indices $0$ to $n'$. Therefore, they fit perfectly as columns in a $n'+1$ times $n'+1$ square matrix as defined below

 \begin{definition}
	 \label{def:boundMatrix}
    For $\gamma\in\Gamma$ and $n'\in\mathbb{N}_{+}$ define the matrix $B^{(\gamma)}_{n'}\in\mathbb{N}^{(n'+1)\times (n'+1)}$ as
   \begin{equation*}
	 (B^{(\gamma)}_{n'})_{i,j}=\left(\varphi_{n'}^{(\gamma)}({\rm e}_{j-1})\right)_{i-1}=\left( \textnormal{cl}_{j-1}(\gamma_{j-1,n'}) \right)_{i-1}\quad \textnormal{ for }i,j\in\left\{ 1,\dots,n'+1 \right\}.
   \end{equation*}
 \end{definition} Note that the matrix indexing starts with $1$ whereas the indexing of $V$ begins with $0$. This is the reason for the shift of $i$ and $j$ by one in the above definition.  Because of the involved clipping functions the matrices $(B^{(\gamma)}_{n'})_{n'\in\mathbb{N}_+}$ are upper triangular. This implies two things:
  \begin{enumerate}
	\item The eigenvalues $\lambda^{(\gamma)}_{n',1},\dots\lambda^{(\gamma)}_{n',n'+1}$ of $B^{(\gamma)}_{n'}$, $n'\in\mathbb{N}_+$ are its diagonal entries.
	\item For fixed $n'\in\mathbb{N}$, $i\in\left\{ 1,\dots,n'+1 \right\}$ and increasing $k\in\mathbb{N}_{+}$, the norm $\|(B^{^{(\gamma})}_{n'})^{k}v\|_1$ is of order $\mathcal{O}(\max(\vert\lambda^{(\gamma)}_{n',1}\vert,\dots,\vert\lambda^{(\gamma)}_{n',i}\vert)^k)$ when $v\in\mathbb{N}^{n'+1}$ contains $0$ in its entries with index greater than $i$, given that this maximum is greater than 1.
  \end{enumerate}
  To state the bound from equation~\eqref{eq:mainBound} in a matrix version we furthermore have to connect matrices $B^{(\gamma)}_{n_{l+1}}$ and $B^{(\gamma)}_{n_l}$ for $l\in\left\{ 1,\dots,L-1 \right\}$ of possibly different sizes. To make matrix multiplication possible, an additional matrix defined below is inserted in between.
	\begin{definition}
	  \label{def:connMatrix}
	  For $n,n'\in\mathbb{N}$, define the \emph{connector matrix} $M_{n,n'}$ by
  \begin{equation*}
	M_{n,n'}\in\mathbb{R}^{n'+1\times n+1},\quad M_{i,j}=\delta_{i,\min(j,n'+1)} \text{ for }i\in\left\{ 1,\dots,n'+1 \right\},j\in\left\{ 1,\dots,n+1 \right\}.
  \end{equation*}
	\end{definition}
  \begin{example} The matrices $M_{4,2}$ and $M_{2,4}$ are 
	\begin{equation*}
	M_{4,2}= 
	\begin{pmatrix}
	  1&0&0&0&0\\
	  0&1&0&0&0\\
	  0&0&1&1&1
	\end{pmatrix},\quad
	M_{2,4}=
	\begin{pmatrix}
	  1&0&0\\
	  0&1&0\\
	  0&0&1\\
	  0&0&0\\
	  0&0&0\\
	\end{pmatrix}.
	\end{equation*}
  \end{example}
  The Definition~\ref{def:PhiFunc} of $\varphi^{(\gamma)}_{n'}$, $n'\in\mathbb{N}_+$ $\gamma\in\Gamma$ involves ``$\min(n,n')$'' to reuse $\textnormal{cl}_{n'}(\gamma_{n',n'})$ for entries in the input histogram with index $n$ greater than $n'$. The same is achieved by the above matrices. For $n>n'$ the matrix $M_{n,n'}$ has more columns than rows and the additional columns have a $1$ in the last row. The matrix version of equation~\eqref{eq:mainBound} is stated below in equation~\eqref{eq:MainRecursiveMatrix}. 

  The Appendix~\ref{app:secondBound} contains formal proofs of the results of this section. In Table~\ref{tab:ComparisonMontufarVsFramework} we give an overview over similar concepts of our framework and their counterpart in the bound of Mont\'ufar. 
\begin{table}[htpb]
  \centering
  \begin{tabular}{|p{3cm}|p{5cm}|p{6.4cm}|}
	\hline
	&Mont\'ufar's bound& Our framework\\
	\hline
	ordered set&$\mathbb{N}$& V \\
	\hline
	underlying\newline inequality& $\forall n,n'\in\mathbb{N}_+, h\in\textnormal{RL}(n,n')$\newline
	$|\mathcal{S}_{h}|\le\sum_{j=0}^{n}{n'\choose j}$
	&$\forall n,n'\in\mathbb{N}_+, h\in\textnormal{RL}(n,n')$\newline
	$\mathcal{H}_{n'}(\mathcal{S}_{h})\preceq \gamma_{n,n'}$
	\\
	\hline
	derived \newline recursion&$|\mathcal{S}^{(l+1)}_{\mathbf{h}}|\le|\mathcal{S}^{(l)}_{\mathbf{h}}|\sum_{j=0}^{n_l^*}{n_{l+1}\choose j}$,\newline 
	$n_l^*=\min(n_0,\dots,n_{l})$
	&
	$\tilde{\mathcal{H}}^{(l+1)}(\mathcal{S}_{\mathbf{h}}^{(l+1)})\preceq \varphi^{(\gamma)}_{n_{l+1}}(\tilde{\mathcal{H}}^{(l)}(\mathcal{S}_{\mathbf{h}}^{(l)}))$\\
	\hline
  \end{tabular}
  \caption{Comparison of the central ideas and quantities involved in Mont\'ufar's bound and our framework.}
  \label{tab:ComparisonMontufarVsFramework}
\end{table}
 
\subsection{Main Result}
\label{sec:MainResult}
In the Appendix \ref{app:secondBound} we show that the following statements are true.
\begin{theorem}Assume $\gamma\in\Gamma$, i.e. $\gamma$ is a collection of elements in $V$ that satisfies the bound condition from Definition~\ref{def:boundCondition}.  For the number of elements $|\mathcal{S}_{\mathbf{h}}|$ we have the following upper bound:
   \label{thm:MainResultPhi}
 \begin{equation}
   |\mathcal{S}_{\mathbf{h}}|\le\|\varphi^{(\gamma)}_{n_L}\circ\dots\circ\varphi^{(\gamma)}_{n_1}({\rm e}_{n_0})\|_1.
   \label{eq:MainRecursivePhi}
 \end{equation}
 \end{theorem}
 As explained in the previous section, we can derive a matrix formulation of the above result.
 \begin{corollary}
   \label{cor:MainResultMatrix}
   For a collection $\gamma\in\Gamma$,  $|\mathcal{S}_{\mathbf{h}}|$ is bounded by
   \begin{equation}
	 \label{eq:MainRecursiveMatrix}
	 |\mathcal{S}_{\mathbf{h}}|\le \|B^{(\gamma)}_{n_L}M_{n_{L-1},n_L}\dots B^{(\gamma)}_{n_1}M_{n_{0},n_{1}}e_{n_0+1}\|_1,
   \end{equation}where
   $e_{n_0+1}$ is the unit vector in $\mathbb{R}^{n_0+1}$ that has value $0$ at the indices $1$ to $n_0$ and value $1$ at index $n_0+1$.
\end{corollary}
   
 \subsection{Applications}
 In the next section we use different collections $\gamma\in \Gamma$  to show how concrete bounds can be obtained from our main result~\eqref{eq:MainRecursiveMatrix}. To this end, we first give a particular collection $\gamma$, then show that it satisfies the bound conditions from Definition~\ref{def:boundCondition} and finally state the corresponding matrices $B^{(\gamma)}_{n'}$, $n'\in\mathbb{N}$.
 
 First, we derive the three bounds from equations~\eqref{eq:MostBasicBound}, \eqref{eq:Montufarbound} and~\eqref{eq:BoundingCounting} in this way. Then we consider an asymptotic setting and discuss the results. Finally, we explain how even stronger results may be obtained. 
 \subsubsection{Naive bound}
 Define the collection $(\gamma_{n,n'})_{n'\in\mathbb{N}_+,n\in\left\{ 0,\dots,n' \right\} }$ of elements in $V$ by
 \begin{equation}
   \label{eq:naiveGamma}
   \gamma_{n,n'}=2^{n'}{\rm e}_{n'}\quad \textnormal{ for }n'\in\mathbb{N}, n\in\left\{ 0,\dots,n' \right\}.
 \end{equation}
 \begin{lemma}
   \label{lem:naiveProof}
   With $\gamma_{n,n'}=2^{n'}{\rm e}_{n'}$, the collection $(\gamma_{n,n'})_{n'\in\mathbb{N}_+,n\in\left\{ 0,\dots,n' \right\}}$ satisfies the bound condition from Definition~\ref{def:boundCondition}.
   \begin{proof}
	The first property of the bound condition requires that for all $n'\in\mathbb{N}_+$ and $n\in\left\{ 0,\dots,n' \right\}$ 
	 \begin{equation*}
	   \max\left\{ \mathcal{H}_{n'}(\mathcal{S}_h)|h\in\textnormal{RL}(n,n') \right\}\preceq \gamma_{n,n'}.
	 \end{equation*} To prove this let $n,n'$ as above and $h\in\textnormal{RL}(n,n')$. Then for all $J\in\mathbb{N}$,
	 $\sum_{j=J}^{\infty}\left( \mathcal{H}_{n'}(\mathcal{S}_{h}) \right)_j\le|\mathcal{S}_h|\le 2^{n'}=\sum_{j=J}^{\infty}(\gamma_{n,n'})_j$. The second property of the bound condition requires for all $n'\in\mathbb{N}_+, n,\tilde n\in\left\{ 0,\dots,n' \right\}$ that $\gamma_{n,n'}\preceq \gamma_{\tilde n,n'}$ whenever $n\le \tilde n$. But this monotonicity in the first index is clearly fulfilled.
   \end{proof} 
 \end{lemma}
 According to Definition~\ref{def:boundMatrix}, the corresponding bound matrices are $B^{(\gamma)}_{n'}=2^{n'} I_{n'+1}$, $n'\in\mathbb{N}_+$. Hence, our main result becomes $|\mathcal{S}_{\mathbf{h}}|\le 2^{n_1+\dots+n_L}$, which is the bound from equation~\eqref{eq:MostBasicBound}.
 \subsubsection{Using Zaslavsky's result yields Mont\'ufar's bound}
 \label{sec:Zaslavsky}
 Note that by Zaslavsky's result (Lemma~\ref{lem:OneLayer1975}), $|\mathcal{S}_{h}|\le\sum_{i=0}^{n}{n' \choose n}$ for $n,n'\in\mathbb{N}_+$, $h\in\textnormal{RL}(n,n')$. This motivates the definition
 \begin{equation}
   \label{eq:zaslavskyGamma}
   \gamma_{n,n'}=\sum_{j=0}^{n}{n'\choose j}{\rm e}_{n'}\quad \textnormal{ for }n'\in\mathbb{N}_{+}, n\in\left\{ 0,\dots,n' \right\}.
 \end{equation}
 \begin{lemma}
   \label{lem:zaslavskyProof}
   With $\gamma_{n,n'}=\sum_{j=0}^{n}{n'\choose j}{\rm e}_{n'}$, the collection $(\gamma_{n,n'})_{n'\in\mathbb{N}_{+},n\in\left\{ 0,\dots,n' \right\}}$ satisfies the bound condition from Definition~\ref{def:boundCondition}.
   \begin{proof}
	 Let $n'\in\mathbb{N}_{+},n\in\left\{ 0,\dots,n' \right\}$ and $h\in\textnormal{RL}(n,n')$. By Lemma~\ref{lem:OneLayer1975}, $|\mathcal{S}_h|\le\sum_{j=0}^{n}{n'\choose j}$. For all $J\in\mathbb{N}$ note that
	 $\sum_{j=J}^{\infty}\left( \mathcal{H}_{n'}(\mathcal{S}_{h}) \right)_j=\sum_{j=J}^{\infty}\mid\left\{ s\in\mathcal{S}_h| |s|=j \right\}|$. If $J>n'$, this is zero. If $J\le n'$, then 
	   \begin{equation*}
		 \sum_{j=J}^{\infty}(\mathcal{H}_{n'}(\mathcal{S}_h))_j\le\vert\mathcal{S}_{h}\vert\le \sum_{j=0}^{n}{n'\choose j}=(\gamma_{n,n'})_{n'}=\sum_{j=J}^{\infty}(\gamma_{n,n'})_j.
	   \end{equation*} This means that $\mathcal{H}_{n'}(\mathcal{S}_{h})\preceq \gamma_{n,n'}$ according to Definition~\ref{def:preceq}. The second bound property of the bound condition is fulfilled because $\sum_{j=0}^{n}{n'\choose j}{\rm e}_{n'}\preceq \sum_{j=0}^{\tilde n}{n'\choose j}{\rm e}_{n'}$ for $n\le \tilde n$.
   \end{proof} 
 \end{lemma} With this definition for $\gamma$, the bound matrices $(B^{(\gamma)}_{n'})_{n'\in\mathbb{N}_{+}}$ from Definition~\ref{def:boundMatrix} that appear in equation~\eqref{eq:MainRecursiveMatrix} are diagonal. We will call these matrices \emph{Zaslavsky bound matrices} and denote them by $(D_n)_{n\in\mathbb{N}_{+}}$. Following the construction from Definition~\ref{def:boundMatrix}, we obtain
	\begin{equation}
	  D_n=\textnormal{diag}\left( \sum_{j=0}^{0}{ n\choose j },\sum_{j=0}^{1}{ n\choose j },\dots,\sum_{j=0}^{n}{ n\choose j } \right) \quad\textnormal{ for }n\in\mathbb{N}_{+}
	\end{equation}
  \begin{example}
	The Zaslavsky bound matrices $D_1$ to $D_4$ are
	\begin{eqnarray*}
	D_1&=& 
	{\tiny \begin{pmatrix}
	  1&0\\
	  0&2\\
  \end{pmatrix}}\\
	D_2&=& 
	\tiny
	\begin{pmatrix}
	  1&0&0\\
	  0&3&0\\
	  0&0&4
	\end{pmatrix}\\
	D_3&=& 
	\tiny
	\begin{pmatrix}
	  1&0&0&0\\
	  0&4&0&0\\
	  0&0&7&0\\
	  0&0&0&8
	\end{pmatrix}\\
	D_4&=& 
	\tiny
	\begin{pmatrix}
	  1&0&0&0&0\\
	  0&5&0&0&0\\
	  0&0&11&0&0\\
	  0&0&0&15&0\\
	  0&0&0&0&16
	\end{pmatrix}
	\end{eqnarray*}
  \end{example}
	Now, equation~\eqref{eq:MainRecursiveMatrix} becomes
	$ |\mathcal{S}_{\mathbf{h}}|\le \|D_{n_L}M_{n_{L-1},n_L}\dots D_{n_1}M_{n_{0},n_{1}}e_{n_0+1}\|_1$,
	which may also be written in the form~\eqref{eq:Montufarbound}:
	\begin{equation}
	  \label{eq:MainRecursiveMatrixDAlternativeForm}
	  |\mathcal{S}_{\mathbf{h}}\vert\le\prod_{l=1}^{L}\sum_{j=0}^{\min(n_0,\dots,n_{l-1})}{n_{l}\choose j}.
	\end{equation}
  We will call this bound the \emph{Mont\'ufar bound}.

	\subsubsection{Using binomial coefficients in combination with Zaslavsky's result}
	\label{sec:BinomialCoefficients}
 Now, we define the collection $(\gamma_{n,n'})_{n'\in\mathbb{N}_{+},n\in\left\{ 0,\dots,n' \right\} }$ of elements in $V$ by
 \begin{equation}
   \label{eq:binomialGamma}
 \gamma_{n,n'}=\sum_{j=0}^{n}{n'\choose j}{\rm e}_{n'-j}\quad \textnormal{ for }n'\in\mathbb{N}_{+}, n\in\left\{ 0,\dots,n' \right\}
 \end{equation}
 \begin{lemma}
   With $\gamma_{n,n'}=\sum_{j=0}^{n}{n'\choose j}{\rm e}_{n'-j}$, the collection $(\gamma_{n,n'})_{n'\in\mathbb{N}_{+},n\in\left\{ 0,\dots,n' \right\}}$ satisfies the bound condition from Definition~\ref{def:boundCondition}.
\label{lem:binomialProof}
\begin{proof}
  We proceed as in the proof of Lemma~\ref{lem:zaslavskyProof}. Let $n'\in\mathbb{N}_{+},n\in\left\{ 0,\dots,n' \right\}$ and $h\in\textnormal{RL}(n,n')$. For all $J\in\mathbb{N}$,  $\sum_{j=J}^{\infty}\left( \mathcal{H}_{n'}(\mathcal{S}_{h}) \right)_j=\sum_{j=J}^{\infty}\mid\left\{ s\in\mathcal{S}_h| |s|=j \right\}|$. If $J>n'$, then this is zero. If $J\in\left\{ n'-n,\dots,n' \right\}$, 
  \begin{equation*}
	\sum_{j=J}^{\infty}(\mathcal{H}_{n'}(\mathcal{S}_h))_j=
  \sum_{j=J}^{\infty}\mid\left\{ s\in\mathcal{S}_h\mid |s|=j \right\}\mid
  \le\sum_{j=J}^{n'}{n'\choose j}=
  \sum_{j=J}^{\infty}(\gamma_{n,n'})_j.	
\end{equation*} For $J\in\left\{ 0,\dots,n'-n-1 \right\}$, $\sum_{j=J}^{\infty}(\mathcal{H}_{n'}(\mathcal{S}_h))_j\le |\mathcal{S}_h|\le\sum_{j=0}^{n}{n'\choose j}=\sum_{j=J}^{\infty}(\gamma_{n,n'})_j$. Hence, the first property of the bound condition is satisfied, but also the second condition is clearly met. 
\end{proof}
 \end{lemma}
 The previous lemma shows that we can use $\gamma$ as defined in equation~\eqref{eq:binomialGamma} in the main results stated above in equations~\eqref{eq:MainRecursivePhi} and~\eqref{eq:MainRecursiveMatrix}. We will denote the corresponding matrices $(B^{(\gamma)}_{n'})_{n'\in\mathbb{N}_{+}}$ by $(B_{n'})_{n'\in\mathbb{N}_{+}}$ and call them \emph{binomial bound matrices}. With this definition, equation~\eqref{eq:MainRecursiveMatrix} becomes
	\begin{equation}
	  \label{eq:MainRecursiveMatrixB}
	  |\mathcal{S}_{\mathbf{h}}|\le \|B_{n_L}M_{n_{L-1},n_L}\dots B_{n_1}M_{n_{0},n_{1}}e_{n_0+1}\|_1,
	\end{equation}which we will call the \emph{Binomial bound}.
	The Tables~\ref{tab:gammas} and~\ref{tab:phigammas} illustrate how the matrix $B_5$ can be obtained. 
	\begin{table}
  \centering
  {\def\arraystretch{1.3}\tabcolsep=3pt
  \begin{tabular}[htb]{|l||c|c|c|c|c|c|c|c}
	\hline
	\backslashbox{i}{n}&0&1&2&3&4&5 \\
	\hline\hline
	0&0&0&0&0&0&${5\choose 5}$\\
	\hline
	1&0&0&0&0&$5\choose 4$&${5\choose 4}$\\
	\hline
	2&0&0&0&$5\choose 3$&$5\choose 3$&${5\choose 3}$\\
	\hline
	3&0&0&$5\choose 2$&$5\choose 2$&$5\choose 2$&${5\choose 2}$\\
	\hline
	4&0&$5\choose 1$&$5\choose 1$&$5\choose 1$&$5\choose 1$&${5\choose 1}$\\
	\hline
	5&$5\choose 0$&$5\choose 0$&$5\choose 0$&$5\choose 0$&$5\choose 0$&${5\choose 0}$\\
	\hline
	6&$0$&$0$&$0$&$0$&$0$&${0}$\\
	\hline
	$\vdots$&$0$&$0$&$0$&$0$&$0$&${0}$\\
  \end{tabular}
}
\caption{The values of $(\gamma_{n,5})_i$ and different values for $i$ and $n$.}
  \label{tab:gammas}
	\end{table}

	\begin{table}
	  \centering
	  {\def\arraystretch{1.3}\tabcolsep=3pt
	  \begin{tabular}[htb]{|l||c|c|c|c|c|c|c|c}
		\hline
		\backslashbox{i}{n}&0&1&2&3&4&5&6&$\dots$ \\
		\hline\hline
		0&$5\choose 0$&0&0&0&0&${5\choose 5}$&$5\choose 5$&$5\choose 5$\\
		\hline
		1&0&${5\choose 0}+{5\choose 1}$&0&0&$5\choose 4$&${5\choose 4}$&$5\choose 4$&$5\choose 4$\\
		\hline
		2&0&0&${5\choose0}+{5\choose 1}+{5\choose 2}$&$5\choose 3$&$5\choose 3$&${5\choose 3}$&$5\choose 3$&$5\choose 3$\\
		\hline
		3&0&0&0&${5\choose0}+{5\choose 1}+{5\choose 2}$&$5\choose 2$&${5\choose 2}$&$5\choose 2$&$5\choose 2$\\
		\hline
		4&0&0&0&0&${5\choose 0}+{5\choose 1}$&${5\choose 1}$&$5\choose 1$&$5\choose 1$\\
		\hline
		5&0&0&0&0&0&${5\choose 0}$&$5\choose 0$&$5\choose 0$\\
		\hline
		6&$0$&$0$&$0$&$0$&$0$&${0}$&$0$&$0$\\
		\hline
		$\vdots$&$0$&$0$&$0$&$0$&$0$&${0}$&$0$&$\ddots$\\
	  \end{tabular}
	}
	\caption{The values of $(\varphi^{(\gamma)}_5\left( {\rm e}_n \right))_i=\textnormal{cl}_n(\gamma_{n,5})_i$ for different values of $i$ and $n$. The entries for $i,n\in\left\{ 0,\dots,5 \right\}$ represent the matrix $B^{(\gamma)}_{5}\in\mathbb{R}^{6\times 6}$. It is upper triangular because of the $\textnormal{cl}$ function involved in Definition~\ref{def:PhiFunc}. To illustrate the intuition behind this, take for example an input region with dimension $n=1$. The output of the five layers on this region cannot have dimension greater than $1$.}
\label{tab:phigammas}
  \end{table}
  \begin{example}
  The binomial bound matrices $B_1$ to $B_4$ are
  \begin{eqnarray*}
	B_1&=& 
	{\tiny
	\begin{pmatrix}
	  1\choose 0&1\choose 1\\
	  0&1\choose 0\\
	\end{pmatrix}=
	\begin{pmatrix}
	  1 &1\\
	  0&1
  \end{pmatrix}}\\
	B_2&=& 
	{\tiny
	\begin{pmatrix}
	  2\choose 0&0& 2\choose 2\\
	  0&{2\choose 0}+{2\choose 1}&2\choose 1\\
	  0&0&2\choose 0\\
	\end{pmatrix}=
	\begin{pmatrix}
	  1 &0&1\\
	  0&3&2\\
	  0&0&1
  \end{pmatrix}}\\
	B_3&=& 
	{\tiny
	\begin{pmatrix}
	  3\choose 0&0&0& 3\choose 3\\
	  0&{3\choose 0}+{3\choose 1}&3\choose 2& 3\choose 2\\
	  0&0&{3\choose 0}+{3\choose 1}& 3\choose 2\\
	  0&0&0&3\choose 0\\
	\end{pmatrix}=
	\begin{pmatrix}
	  1 &0&0&1\\
	  0&4&3&3\\
	  0&0&4&3\\
	  0&0&0&1
  \end{pmatrix}}\\
	B_4&=& 
	{\tiny
	\begin{pmatrix}
	  4\choose 0&0&0&0& 4\choose 4\\
	  0&{4\choose 0}+{4\choose 1}&0&4\choose 3& 4\choose 3\\
	  0&0&{4\choose 0}+{4\choose 1}+{4\choose 2}&4\choose 2& 4\choose 2\\
	  0&0&0&{4\choose 0}+{4\choose 1}& 4\choose 1\\
	  0&0&0&0& 4\choose 0\\
	\end{pmatrix}=
	\begin{pmatrix}
	  1 &0&0&0&1\\
	  0&5&0&4&4\\
	  0&0&11&6&6\\
	  0&0&0&5&4\\
	  0&0&0&0&1
  \end{pmatrix}}\\
  \end{eqnarray*}
  \end{example}
  In the Appendix~\ref{app:decompBound}, we give an explicit formula for these bound matrices and analyze them. It turns out that an explicit closed-form formula for a Jordan-like decomposition exists such that arbitrary powers can easily be computed. We have preferred this representation because in contrast to an ordinary Jordan decomposition, we get corresponding basis transformation matrices that are upper triangular.
  \begin{example}
	The binomial bound matrices $B_1$ to $B_4$ have the following decomposition;
	\begin{eqnarray*}
	B_1&=& 
	{\tiny
	\begin{pmatrix}
	  1&0\\
	  0&1
	\end{pmatrix}
	\begin{pmatrix}
	1&1\\
	0&1
	\end{pmatrix}
	\begin{pmatrix}
	  1&0\\
	0&1
	\end{pmatrix}^{-1}
  }\\
	B_2&=& 
	{\tiny
	\begin{pmatrix}
	  1&0&0\\
	  0&1&-1\\
	  0&0&1
	\end{pmatrix}
	\begin{pmatrix}
	  1&0&1\\
	  0&3&0\\
	  0&0&1
	\end{pmatrix}
	\begin{pmatrix}
	  1&0&0\\
	  0&1&-1\\
	  0&0&1
  \end{pmatrix}^{-1}}\\
	B_3&=& 
	{\tiny
	\begin{pmatrix}
	  1&0&0&0\\
	  0&3&0&0\\
	  0&0&1&-1\\
	  0&0&0&1
	\end{pmatrix}
	\begin{pmatrix}
	  1&0&0&1\\
	  0&4&1&0\\
	  0&0&4&0\\
	  0&0&0&1
	\end{pmatrix}
	\begin{pmatrix}
	  1&0&0&0\\
	  0&3&0&0\\
	  0&0&1&-1\\
	  0&0&0&1
  \end{pmatrix}^{-1}}\\
	B_4&=& 
	{\tiny
	\begin{pmatrix}
	  1&0&0&0&0\\
	  0&4&0&0&0\\
	  0&0&1&-1&0\\
	  0&0&0&1&-1\\
	  0&0&0&0&1
	\end{pmatrix}
	\begin{pmatrix}
	  1&0&0&0&1\\
	  0&5&0&1&0\\
	  0&0&11&0&0\\
	  0&0&0&5&0\\
	  0&0&0&0&1
	\end{pmatrix}
  \begin{pmatrix}
	  1&0&0&0&0\\
	  0&4&0&0&0\\
	  0&0&1&-1&0\\
	  0&0&0&1&-1\\
	  0&0&0&0&1
  \end{pmatrix}^{-1}}
	\end{eqnarray*}
	Such a decomposition can be useful for the analysis of the bound ~\eqref{eq:MainRecursiveMatrixB} in the case when several stacked layers are of the same dimension because $M_{n,n'}$ is the identity matrix for $n=n'\in\mathbb{N}_+$.
  \end{example}
  We want to note that our bound~\eqref{eq:MainRecursiveMatrixB} coincides with the bound from equation~\eqref{eq:BoundingCounting}. This can be seen from the recursion used in the proof of Theorem~1, \cite{DBLP:BoundingCounting} which is represented by the matrix multiplication $M_{n,n'}B_{n}$, $n,n'\in\mathbb{N}_+$ in~\eqref{eq:MainRecursiveMatrixB}.
 \subsubsection{Asymptotic setting}
  \label{sec:Asymptotic}
  We will now analyze the asymptotic behaviour of these bounds in the setting where the number of layers $L$ is varying but each of them has the same width $n\in\mathbb{N}_{+}$ except for the input, which is of arbitrary fixed dimension $n_0\in\mathbb{N}_+$.

  \begin{example}
	Assume that the input layer has dimension $n_{0}\le 4$, $L\in\mathbb{N}_{+}$ and all other layers are 4-dimensional, i.e. $n_1=n_2=\dots=n_L=4$. For $\mathbf{n}=(n_1,\dots,n_L)$ and $\mathbf{h}\in\text{RL}(n_0,\mathbf{n})$, the Mont\'{u}far bound from equation~\eqref{eq:MainRecursiveMatrixDAlternativeForm} states that
  \begin{equation}
	\label{eq:example4Naive}
	|\mathcal{S}_{\mathbf{h}}|\le\left( \sum_{j=0}^{n_0}{4\choose j}\right)^{L},
  \end{equation}whereas the better upper bound using binomial coefficients from equation~\eqref{eq:MainRecursiveMatrixB} implies
  \begin{equation}
	\label{eq:example4Binomial}
	|\mathcal{S}_{\mathbf{h}}|\le\|B_4^Le_{n_0+1}\|_1.
  \end{equation}
  With the theory from Appendix~\ref{app:decompBound}, we can compute this matrix power for arbitrary $L$:
  \begin{align*}
	B_4^L=
	{\tiny
	\begin{pmatrix}
	  1&0&0&0&1\\
	  0&5&0&4&4\\
	  0&0&11&6&6\\
	  0&0&0&5&4\\
	  0&0&0&0&1
  \end{pmatrix}}^L
	={
	\tiny
	\begin{pmatrix}
	  1&0&0&0&0\\
	  0&4&0&0&0\\
	  0&0&1&-1&0\\
	  0&0&0&1&-1\\
	  0&0&0&0&1
	\end{pmatrix}
	\begin{pmatrix}
	  1&0&0&0&L\\
	  0&5^L&0&L5^{L-1}&0\\
	  0&0&11^L&0&0\\
	  0&0&0&5^L&0\\
	  0&0&0&0&1
	\end{pmatrix}
  \begin{pmatrix}
	  1&0&0&0&0\\
	  0&\tfrac{1}{4}&0&0&0\\
	  0&0&1&1&1\\
	  0&0&0&1&1\\
	  0&0&0&0&1
	\end{pmatrix}
  }.
  \end{align*}
  This allows us to explicitly evaluate the two bounds, see Table~\ref{tab:exampleTwoBounds}. For $n_0\le2$, the bounds coincide because the upper left quarter of the matrices $B_n$ and $D_n$ are the same. However, for $n_0\ge 3$, for the asymptotic order of the bound~\eqref{eq:example4Binomial} is always $\mathcal{O}(11^L)$, whereas~\eqref{eq:example4Naive} is of order $\mathcal{O}( (\sum_{j=0}^{n_0}{4\choose j} )^L)$.
  \begin{table}
	\centering
	\begin{tabular}{|r|l|l|l|l|}
	  \hline
	  $n_0$&1&2&3&4\\ 
	  \hline
	 Mont\'{u}far bound&$5^L$&$11^L$&$15^L$&$16^L$\\ 
	 \hline
	 binomial bound&$5^L$&$11^L$&$11^L+4L5^{L-1}$&$11^L+4L5^{L-1}+L$\\ 
	 \hline
	\end{tabular}
	\caption{Explicit computation of the bounds in equations~\eqref{eq:example4Naive} and~\eqref{eq:example4Binomial}}
	\label{tab:exampleTwoBounds}
  \end{table}
  \end{example}
  We can exploit our theory to generalize the above example. Assume that $n_0, n\in\mathbb{N}_{+}$, $n_1=\dots=n_L=n$, $\mathbf{n}=(n_1,\dots,n_L)$ and $\mathbf{h}\in\textnormal{RL}(n_0,\mathbf{n})$. In this case, the Mont\'{u}far bound from equation~\eqref{eq:MainRecursiveMatrixDAlternativeForm} implies
  \begin{equation}
	|\mathcal{S}_{\mathbf{h}}|\le\left( \sum_{j=0}^{\min(n_0,n)}{n\choose j} \right)^L.
  \end{equation}
  In contrast, the binomial bound yields $|\mathcal{S}_{\mathbf{h}}|\le \|B^{L}_{n}M_{n_0,n}e_{n_0+1}\|_{1}=\|B^{L}_{n}e_{\min(n_0,n)+1}\|_1$, which is evaluated explicitly in Corollary~\ref{cor:powerBNorm}. This result is part of our contribution. Table~\ref{tab:asymptoticbounds} compares the results.  Note that the results in Table~\ref{tab:exampleTwoBounds} are a special case. In particular, we can conclude the asymptotic orders of the two bounds for the specified setting as the number of layers $L$ tends to infinity, see Table~\ref{tab:boundOrders}. 
  \begin{table}
	\centering
	\begin{tabular}{|r|c|l|}
	  \hline
	  &$n_0\le \lfloor n/2\rfloor$& $n_0>\lfloor n/2\rfloor$\\
	  \hline
  Mont\'{u}far bound&
  \multicolumn{2}{c|}{
$ (\sum_{j=0}^{\min(n_0,n)}{n\choose j})^L $} \\
\hline
equation \eqref{eq:BoundingCoundingWeakened}&
  \multicolumn{2}{c|}{
$
   2^{Ln}\left( \frac{1}{2}+ \frac{1}{2\sqrt{\pi n}}\right)^{L/2}\sqrt{2} $} \\

	  \hline
	  binomial bound&
$ (\sum_{j=0}^{n_0}{n\choose j})^L $
	  &
$\left( \sum_{j=0}^{\lfloor n/2\rfloor}{n\choose j} \right)^{L}+L\sum_{s=\lfloor n/2\rfloor +1}^{n_0}\left( \sum_{j=0}^{n-s}{n\choose j} \right)^{L-1}{n\choose n-s}$\\
	  \hline
	\end{tabular}
	\caption{Comparison of explicit bounds for the considered asymptotic setting. The tightest bound in the last row is our contribution. In contrast to the bound from the second row, it is a direct evaluation of~\eqref{eq:MainRecursiveMatrixB} without using approximations.}
	\label{tab:asymptoticbounds}
  \end{table}

  \begin{table}
	\centering
	\begin{tabular}{|r|c|l|}
	  \hline
	  &$n_0\le \lfloor n/2\rfloor$& $n_0>\lfloor n/2\rfloor$\\
	  \hline
  Mont\'{u}far bound&
  \multicolumn{2}{c|}{
$\mathcal{O}\left( (\sum_{j=0}^{\min(n_0,n)}{n\choose j})^L \right)$} \\
\hline
equation \eqref{eq:BoundingCoundingWeakened}&
  \multicolumn{2}{c|}{
  $\mathcal{O}\left( 2^{L\left( n-\frac{1}{2}+\frac{1}{2}\log_2\left( 1+\tfrac{1}{\sqrt{\pi n}} \right) \right)} \right)$}
\\

	  \hline
	  binomial bound, $n>1$&
	  $ \mathcal{O}\left( (\sum_{j=0}^{n_0}{n\choose j})^L  \right)$
	  &
	  $\mathcal{O}\left( (\sum_{j=0}^{\lfloor n/2\rfloor}{n\choose j} )^{L}\right)$\\
	  \hline
	\end{tabular}
	\caption{For $L\to\infty$, the asymptotic orders of the two bounds are the same for $n_0\in\left\{ 0,\dots,\lfloor n/2\rfloor\right\}$. For $n_0>\lfloor n/2\rfloor$, the order of the binomial bound is strictly better and does not depend on $n_0$. }
	\label{tab:boundOrders}
  \end{table}
  For example, in the case where $n\ge 3$ is an odd number and $n_0\ge n_1=\dots=n_L=n$, we obtain the asymptotic orders
  \begin{equation}
	\label{eq:equalWidthsRate}
	|\mathcal{S}_{\mathbf{h}}|=
	\begin{cases}
	  \mathcal{O}(2^{Ln}) &\quad\textnormal{ for the Mont\'{u}far bound}\\
	  \mathcal{O}\left( 2^{L\left( n-\frac{1}{2}+\frac{1}{2}\log_2\left( 1+\tfrac{1}{\sqrt{\pi n}} \right) \right)} \right) &\quad\textnormal{ for the bound~\eqref{eq:BoundingCoundingWeakened} from~\cite{DBLP:BoundingCounting}}\\
	  \mathcal{O}(2^{L(n-1)}) &\quad\textnormal{ for the binomial bound}
	\end{cases}
  \end{equation}
  This means that in a neural network where all layers have the same width $n$ and $L\to\infty$, the binomial bound on the number of convex regions in the input space $\mathbb{R}^{n_0}$ is of the same order as the Mont\'{u}far bound would be when all layers had width $n-1$, i.e. our new result gains one dimension in each layer compared to the Mont\'{u}far bound and more than a half dimension in each layer compared to the bound~\eqref{eq:BoundingCoundingWeakened} from \cite{DBLP:BoundingCounting}.

\subsubsection{Discussion}
\label{sec:discussion}
In the previous sections, we have derived three bounds as special cases of our framework: The naive bound from equation~\eqref{eq:MostBasicBound}, the Mont\'ufar bound from equation~\eqref{eq:MainRecursiveMatrixDAlternativeForm} and the Binomial bound from equation~\eqref{eq:MainRecursiveMatrixB}. Below we give two results on how they are related. They state that the above enumeration order is increasing in strictness and give precise necessary and sufficient conditions when one bound is strictly better than another. The proofs are given in the Appendix~\ref{sec:discussionProofs}.

\begin{lemma}
  \label{lem:MontufarNaive}
  The Mont\'{u}far bound from equation~\eqref{eq:MainRecursiveMatrixDAlternativeForm} is always at least as good as the naive bound from equation~\eqref{eq:MostBasicBound} and is strictly better if and only if the width is increasing at some layer, i.e. if there exists $l\in\left\{ 1,\dots,L \right\}$ such that $n_{l-1}<n_l$.
\end{lemma}
While the above result is directly obvious from the two involved bounds, a similar result for the Binomial bound and the Mont\'ufar bound is not obvious. With our framework, we are able derive the following lemma.
\begin{lemma}
  \label{lem:BinomialMontufar}
  The Binomial bound from equation~\eqref{eq:MainRecursiveMatrixB} is always at least as good as the Mont\'{u}far bound from equation~\eqref{eq:MainRecursiveMatrixDAlternativeForm} and is strictly better if and only 
  \begin{equation*}
	\exists l\in\left\{ 1,\dots,L-1 \right\}: n_l<\min\left( n_0,\dots,n_{l} \right)+\min\left( n_0,\dots,n_{l+1} \right).
  \end{equation*}
\end{lemma}
In particular, when $n_l/n_0$ is large for all $l\in\left\{ 1,\dots,L \right\}$ then both bounds are equal. Therefore, the quotient limit~\eqref{eq:quotientOfBounds} with the lower bound~\eqref{eq:lowerBound} is not improved (smaller) if the Binomial bound is used instead of the Mont\'ufar bound.  

We want to use the above results to give an overview of how the three bounds behave for different network architectures. Figure~\ref{fig:architectures} gives a qualitative overview of some such architectures.
\begin{figure}[htbp]
  \centering
  \includegraphics[width=0.8\textwidth]{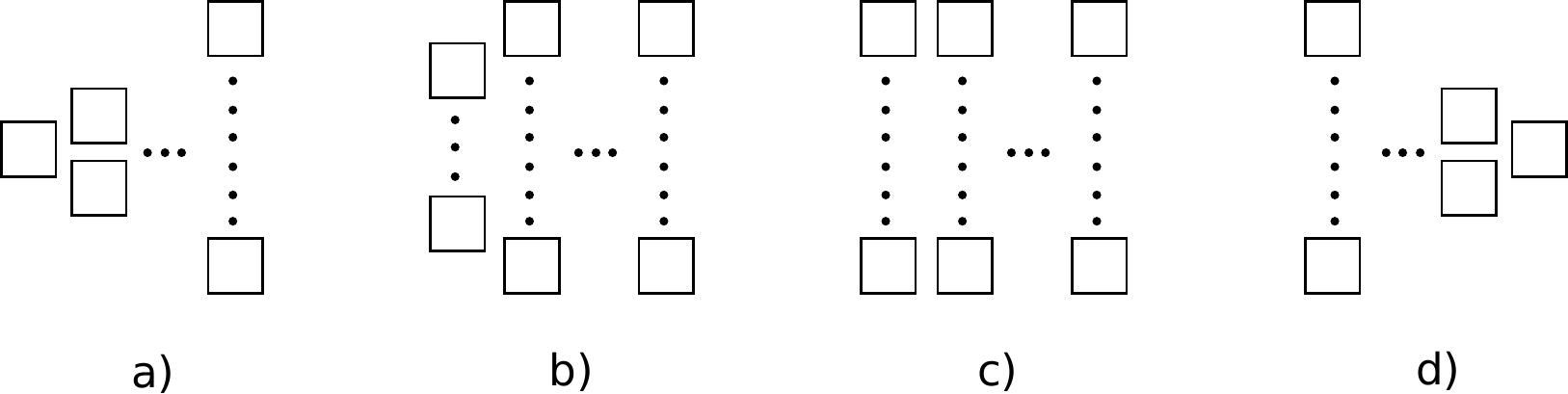}
\caption{Qualitative overview of several architecture types: a) increasing width architecture, b) same number of neurons $n$ in all layers, input width $n_0\le n$, c) same input width and number of neurons in all layers, d) decreasing width architecture.}
  \label{fig:architectures}
\end{figure}

The following four settings are depicted in Figure~\ref{fig:architectures}.
  \begin{enumerate}[a)]
  \item In the increasing width scenario we assume that $n_{l+1}\ge n_l$ for $l\in\left\{ 1,\dots,L-1 \right\}$. The Mont\'ufar bound and the naive bound evaluate to $\prod_{l=1}^L\sum_{j=0}^{n_{0}} {n_l\choose j}$ and $\prod_{l=1}^L\sum_{j=0}^{n_{l}} {n_l\choose j}$ respectively. By Lemma~\ref{lem:BinomialMontufar}, the binomial bound is able to improve on this if and only if $n_1<2n_0$.
  \item The case where $n_0\le n$ and $n_1=\dots=n_L=:n$ was analyzed in Section~\ref{sec:Asymptotic}. The results from Tables~\ref{tab:asymptoticbounds} and \ref{tab:boundOrders} give explicit formulas and asymptotic orders for this setting. Note that the case distinction concerning $n_0$ and $n$ in these tables directly reflects the criterion $n_1<2n_0$ from Lemma~\ref{lem:BinomialMontufar} about when the Binomial bound is strictly better than the Mont\'ufar bound. In practical networks such as wide residual networks~\cite{BMVC2016_87}, regardless of the fact that they are usually not entirely ReLU feed-forward neural networks and do not exactly have the considered architecture, often $n_0$ is much smaller than $n$ such that this criterion is not fulfilled and the improvements of the explicit formulas we contributed for the Binomial bound over the Mont\'ufar bound are not applicable.
  \item The special case of the previous setting where all widths are equal $n_0=\dots=n_L=:n$ causes the Mont\'ufar to coincide with naive bound which simply states that every ReLU activation unit appearing in the neural network doubles the possible number of regions, see Lemma~\ref{lem:MontufarNaive}. In contrast, Lemma~\ref{lem:BinomialMontufar} or the Tables~\ref{tab:asymptoticbounds} and \ref{tab:boundOrders} show that the Binomial bound is strictly better. Despite the rate improvement $n-1$ vs $n$ in equation~\eqref{eq:equalWidthsRate} is in the exponent, it is relatively small for large $n$ which is often the case in practice.
  \item In the setting of decreasing width we assume $n_{l+1}\le n_{l}$ for $l\in\left\{1,\dots,L-1  \right\}$.  An example for this architecture are convolutional neural networks  for image classification where a large number $n_0$ of pixels is used as input and the width decreases in later layers. Unfortunately, again the Mont\'ufar bound does not provide advantage over the naive bound because the condition of Lemma~\ref{lem:MontufarNaive} is not fulfilled. However, the Binomial bound matrix can still be applied in this setting, in fact we show in Lemma~\ref{lem:BinomialNeverBreaks} below that it is always stricter than the naive bound.
  \end{enumerate}
\begin{lemma}
  \label{lem:BinomialNeverBreaks}
  The Binomial bound is always sharper than the naive bound regardless of the number of layers and their widths.
  \begin{proof}
	If the condition of Lemma~\ref{lem:MontufarNaive} is not satisfied then $n_{l-1}\ge n_{l}$ for $l\in\left\{ 1,\dots,L \right\}$. But in this case for all $l\in\left\{ 1,\dots,L-1 \right\}$ it is true that $n_l<\min\left( n_1,\dots,n_l \right)+\min\left( n_1,\dots,n_{l+1} \right)$ such that the condition of Lemma~\ref{lem:BinomialMontufar} is satisfied.
  \end{proof}
\end{lemma}
 \subsubsection{Further improvements}
 We can set
 \begin{equation}
   \label{eq:optimalbound}
   \gamma_{n,n'}:= \max\left\{ \mathcal{H}_{n'}(\mathcal{S}_h)|h\in\textnormal{RL}(n,n') \right\}\quad \textnormal{ for } n'\in\mathbb{N}_{+}, n\in\left\{ 1,\dots,n' \right\}.
 \end{equation}This would yield the best possible bound obtainable with equation~\eqref{eq:MainRecursiveMatrix} because it satisfies the bound condition from Definition~\ref{def:boundCondition} with equality where the relation $\preceq$ is required. However, to the best of our knowledge, these maxima are not explicitly known. Their computation requires to solve a combinatorial and geometrical problem. This is left for future work.  
 For the corresponding matrices $B^{(\gamma)}_{n'}$, $n'\in\mathbb{N}_+$ there might not exist an easy decomposition that allows to compute arbitrary powers explicitly. However, they will be upper triangular, which makes it easy to compute asymptotic orders similar to the results from Section~\ref{sec:Asymptotic}, see the note below Definition~\ref{def:boundMatrix}.

  \section{Summary}
  \label{sec:summary}
  In this work we presented a formal framework for the construction of upper bounds on the number of connected affine linear regions of feed-forward neural networks with ReLU activation functions. We presented two formal criteria summarized as the bound condition. For a collection $\gamma$ of elements in $V$ that meets these criteria, a corresponding bound can be derived. In this sense, we have presented a whole class of upper bounds. In their matrix form, they can be stated as
  \begin{equation}
	\label{eq:mainSummary}
	 |\mathcal{S}_{\mathbf{h}}|\le \|B^{(\gamma)}_{n_L}M_{n_{L-1},n_L}\dots B^{(\gamma)}_{n_1}M_{n_{0},n_{1}}e_{n_0+1}\|_1,
   \end{equation}where the square matrices $B^{(\gamma)}_{m}, m\in\mathbb{N}_+$ can be easily constructed by using appropriately clipped finite-length versions of the infinite-length vectors $\gamma_{1,m},\dots,\gamma_{m,m}$ as their columns. They are always square and upper triangular matrices such that the eigenvalues, which might be interesting for asymptotic considerations, can be read directly from the diagonal.

  We then have derived three existing bounds from this result by plugging in concrete collections for $\gamma$ that satisfy the bound condition.
  \begin{enumerate}
	\item The first collection $\gamma$ is constructed based on the naive result that $n'$ hyperplanes can partition $\mathbb{R}^{n}$ in at most $2^{n'}$ regions. In this case, our framework yields the very basic result from~\cite{pmlr-v70-raghu17a}: For $L,n_0,\dots,n_L\in\mathbb{N}_{+}$ and ReLU layer functions $h_1,\dots,h_L$ mapping between the spaces $\mathbb{R}^{n_0},\dots,\mathbb{R}^{n_L}$, the neural network represented by their composition $f=h_L\circ\dots\circ h_1$ allows a partition of the input space $\mathbb{R}^{n_0}$ in at most $N_f\le 2^{\sum_{l=1}^{L}n_l}$ connected sets on which it is affine linear.
	\item In the second collection $\gamma$ we use the well known result that a ReLU layer function $h$ mapping from $\mathbb{R}^{n}$ to $\mathbb{R}^{n'}$ with $n'$ neurons partitions the input space $\mathbb{R}^{n}$ in at most $\sum_{j=0}^{n}{n'\choose j}$ convex regions on which $h$ is affine linear. The corresponding bound obtained reestablishes a result from~\cite{Montufar17}: 
	  \begin{equation}
		\label{eq:firstBoundSummary}
		N_f\le\prod_{l=0}^{L}\sum_{j=0}^{\min(n_0,\dots,n_{l-1})}{n_{l}\choose j}.
	  \end{equation}
	\item The third collection $\gamma$ is additionally based on the idea, that for a ReLU layer function $h$ mapping between spaces $\mathbb{R}^{n}$ and $\mathbb{R}^{n'}$, there are at most ${n'\choose i}$ regions in the input space on which $i\in \left\{ 1,\dots,n' \right\}$ neurons are active. In this case, the matrix formulation~\eqref{eq:mainSummary} obtained from our framework turns out to be more useful than the representation~\eqref{eq:BoundingCounting} from~\cite{DBLP:BoundingCounting}. We have given an explicit formula for the corresponding matrices $(B^{(\gamma)}_{n'})_{n'\in\mathbb{N}_{+}}$. In addition, we have found an explicit Jordan-like decomposition. This can be used to derive explicit formulas for arbitrary powers of these matrices, which is useful when we want to analyse equation~\eqref{eq:mainSummary} for varying $L$ because the matrices $(M_{nn'})_{n,n'\in\mathbb{N}_+}$ are the identity matrix for $n=n'$. 
  \end{enumerate}
  If the collection $\gamma$ used in our framework is tighter, so will be the resulting bound. Therefore the above bounds are ordered from weak to strong. In addition, we have given precise necessary and sufficient conditions on the network architecture when one of these bounds is strictly better than another.

  We then have considered an asymptotic scenario where the number of layers $L$ is variable, the dimension of the input space $n_0\in\mathbb{N}_+$ and the widths of the individual layers $n_1=\dots=n_L=:n\in\mathbb{N}_+$ are fixed. The results show that for $n_0\le\lfloor n/2\rfloor$, the second and the strongest third bound are the same, but for $n_0>\lfloor n/2\rfloor$, the third is much better, see Table~\ref{tab:boundOrders}. This new detailed analysis was only possible due to our matrix representation with Jordan-like decomposition and is also part of our contribution. In particular, when also the input dimensionality $n_0=n$ and $n$ is odd, we improved the best known asymptotic order for $L\to\infty$ from~\cite{DBLP:BoundingCounting}, $\mathcal{O}\left( 2^{L\left( n-1/2+\log_2\left( 1+1/\sqrt{\pi n} \right)/2 \right)} \right)$ to $\mathcal{O}\left(  2^{L(n-1)}\right)$. This means that in this case a half dimension in each layer is gained.
  
  Finally, we explained how even stronger bounds can be derived. We state the collection $\gamma\in\Gamma$ that would yield the strongest result that can be obtained with our theory. It involves a geometrical and combinatorial problem that needs to be solved first in order to construct the necessary matrices $B^{(\gamma)}_{n'}$, $n'\in\mathbb{N}_{+}$.
  \newpage
  \appendix
  \section{Proofs and intermediate results}
  \label{sec:appendix}
  \subsection{Basic facts for one layer}
  We assume the definitions and conventions from Section~\ref{sec:AnalysisOneLayer}
\begin{lemma}
  Let $n_0,n_1,n_2\in\mathbb{N}_+$, $g_2\in\text{RL}(n_1,n_2)$ and let $g_1:\mathbb{R}^{n_0}\to\mathbb{R}^{n_1}$ be an affine linear function, i.e. there exist $A\in \mathbb{R}^{n_1\times n_0}$ and $c\in\mathbb{R}^{n_1}$ such that for all $x\in\mathbb{R}^{n_0}:g_1(x)=Ax+c$.
  Then the function $g_2\circ g_1\in \textnormal{RL}(n_0,n_2)$ and it holds that 
  \begin{equation*}
	\forall x\in\mathbb{R}^{n_0}:\quad S_{g_2\circ g_1}(x)=S_{g_2}(g_1(x))  
  \end{equation*} for the signatures $S_{g_2\circ g_1}$ and $S_{g_2}$ as in Definition~\ref{def:Signat}.
  \label{lem:CompositionSignature}
  \begin{proof} For all $x\in\mathbb{R}^{n_0}$, it holds that 
	\begin{eqnarray*}
	  g_2\circ g_1(x)= g_2(g_1(x))&=& 
	  \begin{pmatrix}
		\sigma(\langle Ax+c,w^{(h)}_1\rangle+b^{(h)}_1)\\
		\vdots\\
		\sigma(\langle Ax+c,w^{(h)}_{n_2}\rangle+b^{(h)}_{n_2})
	  \end{pmatrix}= 
	  \begin{pmatrix}
		\sigma(\langle x,A^\intercal w^{(h)}_1\rangle+(\langle c,w^{(h)}_1\rangle+ b^{(h)}_1))\\
	  \vdots\\
	  \sigma(\langle x,A^\intercal w^{(h)}_{n_2}\rangle+(\langle c,w^{(h)}_{n_2}\rangle+b^{(h)}_{n_2}))
	\end{pmatrix},
  \end{eqnarray*}which means that $g_2\circ g_1\in\textnormal{RL}(n_0,n_2)$. By Definition~\ref{def:Signat}, for all $i\in\left\{ 1,\dots,n_2 \right\}$ and $x\in\mathbb{R}^{n_0}$
	\begin{equation*}
	  S_{g_2\circ g_1}(x)_i=1  \iff \langle Ax+c,w^{(h)}_i\rangle+b_i>0\iff S_{g_2}(g_1(x))_i=1\qedhere
	\end{equation*} 
  \end{proof}
\end{lemma}
For the rest of this section assume that $n,n'\in \mathbb{N}_{+}$ and $h\in\textnormal{RL}(n,n')$. 
\begin{lemma}
  \label{lem:Convex}
  For any $s\in\left\{ 0,1 \right\}^{n'}$, $R_h(s)$ is a convex set.
  \begin{proof}
	Fix $s\in\left\{ 0,1 \right\}^{n'}$ and assume we have two points $x,x'\in R_h(s)$. Furthermore, let $i\in\left\{ 1,\dots,n' \right\}$. Now we check convexity. For $\alpha\in(0,1)$, define the convex combination $x_\alpha:=\alpha x+(1-\alpha)x'$. It holds that 
	\begin{eqnarray*}
	  S_h\left( x_\alpha\right)_i=1&\iff&  \langle \alpha x+(1-\alpha)x', w^{(h)}_i\rangle +b^{(h)}_i>0\\
	  &\iff&\alpha \left( \langle x,w^{(h)}_i\rangle +b^{(h)}_i \right)+(1-\alpha)\left( \langle x',w^{(h)}_i\rangle+b^{(h)}_i \right)>0
	\end{eqnarray*}If $s_i=1$, then by construction $S_h(x)_i=S_h(x')_i=s_i=1$, such that $\left( \langle x,w^{(h)}_i \rangle +b^{(h)}_i\right)>0$ and $\left( \langle x',w^{(h)}_i\rangle+b^{(h)}_i \right)>0$, which implies $S_h(x_\alpha)_i=1$ by the above formula. Similarly, if $s_i=0$, $S_{h}(x_\alpha)_i=0$. Hence we have shown that for every $i\in\left\{ 1,\dots,n' \right\}$ and every convex combination $x_\alpha$ of $x$ and $x'$, the signature of $x_\alpha$ is equal to those of $x$ and $x'$ in the $i$-th coordinate. This proves the lemma.
  \end{proof}
\end{lemma}
\begin{lemma}
  \label{lem:Partition}
The convex sets $\left( R_h(s) \right)_{s\in\left\{ 0,1 \right\}^{n'}}$ form a partition of $\mathbb{R}^{n}$, i.e. the following two conditions hold:
  \begin{enumerate}
	\item $\forall s,s'\in\left\{ 0,1 \right\}^{n'}:\quad R_h(s)\cap R_h(s')=\left\{  \right\}\iff s\neq s'$
	\item $\bigcup_{s\in\left\{ 0,1 \right\}^{n'}}R_h(s)=\mathbb{R}^{n}$.
  \end{enumerate}
  \begin{proof}We prove both conditions individually.
	\begin{enumerate}
	  \item Let $s,s'\in\left\{ 0,1 \right\}^{n'}$. If $R_h(s)\cap R_h(s')=\left\{  \right\}$ then obviously, $s\neq s'$. If $R_h(s)\cap R_h(s')\neq \left\{  \right\}$ then there exists $x\in R_h(s)\cap R_h(s')$, which implies that $s=S_h(x)=s'$.
	  \item Since $R_h(s)\subset\mathbb{R}^{n}$ for all $s\in\left\{ 0,1 \right\}^{n'}$, it holds that $\bigcup_{s\in\left\{ 0,1 \right\}^{n'}}R_h(s)\subset\mathbb{R}^{n}$. But for any $x\in \mathbb{R}^{n}$, it holds by Definition~\ref{def:SigSReg} that $x\in R_h(S_h(x))\subset\bigcup_{s\in\left\{ 0,1 \right\}^{n'}}R_h(s)$.\qedhere
	\end{enumerate}
  \end{proof}
\end{lemma}

\begin{definition}
  \label{def:Restrict}
  For $s\in\left\{ 0,1 \right\}^{n'}$, denote the restriction of $h$ on the set $R_h(s)$ by $h_s$, i.e.  $
  h_s:
  R_h(s)\to \mathbb{R}^{n'},\;
	  x\mapsto h(x).
	$
\end{definition}
\begin{lemma}
  For any $s\in\left\{ 0,1 \right\}^{n'}$, the function $h_s$ from Definition~\ref{def:Restrict} is an affine linear function of the form
  \begin{equation*}
	h_s:\begin{cases}
	  R_h(s)&\to\mathbb{R}^{n'}\\
	  x&\mapsto\textnormal{diag}(s)\left( W^{(h)}x+b^{(h)} \right).
	\end{cases}
  \end{equation*}
  \begin{proof}
	We prove the lemma coordinate-wise. For every $i\in\left\{ 1,\dots,n' \right\}$ we either have $s_i=0$ or $s_i=1$
	\begin{itemize}
	  \item Assume $s_i=0$. Let $x\in R_h(s)$. By definition $\langle x,w^{(h)}_i\rangle+b^{(h)}_i\le0$, which implies
		\begin{equation*}
		  h_s(x)_i=h(x)_i=\sigma(\langle x,w^{(h)}_i\rangle +b^{(h)}_i)=0=\left( \textnormal{diag}(s)\left( W^{(h)}x+b^{(h)} \right) \right)_i
		\end{equation*}
	  \item Assume $s_i=1$. Let $x\in R_h(s)$. Similarly, $\langle x,w_i\rangle+ b_i>0$, which implies
		\begin{equation*}
		  h_s(x)_i=h(x)_i=\sigma(\langle x,w_i\rangle +b_i)= \langle x,w_i\rangle+b_i=\left( \textnormal{diag}(s)\left( W^{(h)}x+b^{(h)} \right) \right)_i\qedhere
		\end{equation*}
	\end{itemize}
  \end{proof}
  \label{lem:AffineRep}
\end{lemma}
For $s\in\left\{ 0,1 \right\}^{n'}$, the function $h_s$ can be extended affine linearly in a natural way to the whole space $\mathbb{R}^{n}$.
\begin{definition}
  \label{def:AffineExt}
  For $s\in\left\{ 0,1 \right\}^{n'}$, let $\tilde h_s$ be the affine linear extension of $h_s$ on $\mathbb{R}^{n}$, i.e. let
	$\tilde h_s:
	\mathbb{R}^{n}\to\mathbb{R}^{n'},\; 
	x\mapsto \textnormal{diag}(s)\left( W^{(h)}x+b^{(h)} \right).
	$
\end{definition}
\begin{corollary}
  \label{cor:RepOneLayer}
 Let $\tilde h_s$, $s\in\left\{ 0,1 \right\}^{n'}$ as in Definition~\ref{def:AffineExt}. It holds that
  \begin{equation*}
	\forall x \in \mathbb{R}^{n}:\quad h(x)=\sum_{s\in \mathcal{S}_h}^{}\mathds{1}_{R_h(s)}(x)\:\tilde h_s(x).
  \end{equation*}
  \begin{proof}
	For every $x\in \mathbb{R}^{n} $, $x\in R_h(S_h(x))$ by Definition \ref{def:SigSReg}, which implies $h(x)=h_{S_h(x)}(x)$ by Definition \ref{def:Restrict}. Furthermore, $h_{S_h(x)}(x)=\tilde h_{S_h(x)}(x)$ by Lemma~\ref{lem:AffineRep} and Definition~\ref{def:AffineExt}. Now, Lemma~\ref{lem:Partition} and Definition~\ref{def:AttainedSig} justify the equality
	\begin{equation*}
	  h(x)=h_{S_h(x)}(x)=\tilde h_{S_h(x)}(x)=\sum_{s\in\left\{ 0,1 \right\}^{n'}}^{}\mathds{1}_{R_h(s)}(x)\:\tilde h_s(x)=\sum_{s\in \mathcal{S}_h}^{}\mathds{1}_{R_h(s)}(x)\:\tilde h_s(x).\qedhere
	\end{equation*}
  \end{proof}
\end{corollary}
The following result is an adaption of a result on hyperplane arrangements by T. Zaslavsky from 1975, see \cite{1975Zaslavsky}. We state it in terms of our notation and give a proof for completeness.
\begin{lemma}
  \label{lem:OneLayer1975}
  For all $n,n'\in\mathbb{N}_+$ the following upper bound on the number of attained signature holds: 
	\begin{equation}
	  \forall h'\in\textnormal{RL}(n,n')\quad	  |\mathcal{S}_{h'}|\le\sum_{j=0}^{n}{n'\choose j}
	\label{eq:1975ToShow}
	\end{equation}
  \begin{proof}
	Obviously
	\begin{equation*}
	  \forall n,n'\in\mathbb{N}\;\forall h'\in \text{RL}(n,n'):\quad |\mathcal{S}_{h'}|\le|\left\{ 0,1 \right\}^{n'}|= 2^{n'}.
	\end{equation*}
	This implies
	\begin{equation}
	  \label{eq:form1smallerm0}
	  \forall n\in\mathbb{N}\;\forall n'\in\left\{ 0,\dots,n \right\}\;\forall h'\in \text{RL}(n,n'):\quad |\mathcal{S}_{h'}|\le \sum_{j=0}^{n}{n'\choose j}.
	\end{equation}
	Furthermore, we know that
	\begin{equation}
	  \label{eq:form0equal1}
	  \forall n'\in\mathbb{N}, \forall h'\in \text{RL}(1,n'):|\mathcal{S}_{h'}|\le n'+1=\sum_{j=0}^{1}{ n'\choose j }
	\end{equation} since the $n'$ hyperplanes $H^{(h')}_1,\dots,H^{(h')}_{n'}$ corresponding to $h'\in\text{RL}(1,n')$ are points on the real line, hence they can induce at most $n'+1$ non-empty $R_{h'}(s)$, $s\in\left\{ 0,1 \right\}^{n'}$ in the sense of Definition~\ref{def:SigSReg}. 
	
	Now we will prove an induction step. Fix $m,m'\in\mathbb{N}_{+}$ assume that equation~\eqref{eq:1975ToShow} is true for $n=m, n'=m'-1$ and for $n=m-1, n'=m'-1$. Then it is also true for $n=m, n'=m'$. To see this take $h'\in\text{RL}(m,m')$ and define the functions
	\begin{align*}
	  \tilde h&:
	  \begin{cases}
		\mathbb{R}^{m}&\to\mathbb{R}^{m'-1}\\
		x&\mapsto \left( h'(x)_1,\dots,h'(x)_{m'-1} \right)
	  \end{cases}\\
	  \hat h&:
	  \begin{cases}
		H^{(h')}_{m'}&\to\mathbb{R}^{m'-1}\\
		x&\mapsto \left( h'(x)_1,\dots,h'(x)_{m'-1} \right)
	  \end{cases}
	\end{align*}
	We know that $|\mathcal{S}_{h'}|=|\mathcal{S}_{\tilde h}|+|\mathcal{S}_{\hat h}|$ because the number of regions $R_{h'}(x), s\in\left\{ 0,1 \right\}^{n_1}$ defined by the hyperplanes $H^{(h')}_1,\dots,H^{(h')}_{m'}$ is equal to the number of regions $R_{\tilde h}(s), s\in\left\{ 0,1 \right\}^{n_1}$ defined by the hyperplanes $H^{(h')}_1,\dots,H^{(h')}_{m'-1}$ plus the number of these regions that are cut into two by the hyperplane $H^{(h')}_{m'}$. Now note that $H^{(h')}_{m'}$ is an affine linear subspace of $\mathbb{R}^{m}$ with dimension $m-1$, i.e. it is homeomorphic to $\mathbb{R}^{m-1}$ such that we can use the two assumptions of the induction step to conclude
	\begin{equation*}
	  |\mathcal{S}_{h'}|=|\mathcal{S}_{\tilde h}|+|\mathcal{S}_{\hat h}|\le \sum_{j=0}^{m}{m'-1\choose j}+\sum_{j=0}^{m-1}{m'-1\choose j}=\sum_{j=0}^{m}{m'\choose j}.
	\end{equation*} Equations~\eqref{eq:form1smallerm0} and~\eqref{eq:form0equal1} provide a suitable anchor for this induction such that equation~\eqref{eq:1975ToShow} holds for all $n,n'\in\mathbb{N}_{0}$.
  \end{proof}
\end{lemma}
 It is shown in \cite{1975Zaslavsky} that this bound is sharp for hyperplanes ``in general position'', i.e. for all $n,n'\in\mathbb{N}_+$, there exists a $h'\in\text{RL}(n,n')$ such that equation~\eqref{eq:1975ToShow} becomes an equality.

  \subsection{Basic results for multiple fully connected layers}
  We assume the definitions and conventions from Section~\ref{sec:multipleLayers}.
\begin{lemma}
  Let $i\in\left\{ 1,\dots,L \right\}$ and let $f:\mathbb{R}^{n_0}\to \mathbb{R}^{n_i}$ be an affine linear function, i.e. there exist $A\in\mathbb{R}^{n_i\times n_0}$ and $b\in \mathbb{R}^{n_i}$ such that for all $x\in \mathbb{R}^{n_0}$,  $f(x)=Ax+b$.
  Then for every convex set $C\in\mathbb{R}^{n_i}$, the pre-image 
	$f^{-1}(C):=\left\{ x\in\mathbb{R}^{n_0}|f(x)\in C \right\}$ is convex itself.
  \begin{proof}
	Let $x,y\in f^{-1}(C)$ and $\alpha\in[0,1]$. Then 
	\begin{equation*}
	  f\left( \alpha x+\left( 1-\alpha \right)y \right)=\alpha (Ax+b)+\left( 1-\alpha \right)(Ay+b)=\alpha f(x)+\left( 1-\alpha \right)f(y)\in C
	\end{equation*} since $C$ is convex and $f(x),f(y)\in C$.
  \end{proof}
  \label{lem:AffineConvex}
\end{lemma}
\begin{lemma}
  \label{lem:ConvexMulti}
  For any $s\in \left\{ 0,1 \right\}^{n_1}\times\cdots\times\left\{ 0,1 \right\}^{n_L}$, $R_{\mathbf{h}}(s)$ is a convex set.
  \begin{proof}
	First note that 
	\begin{eqnarray*}
	  R_{\mathbf{h}}(s)&=& \left\{ x\in \mathbb{R}^{n_0}\vert\;S_{h_1}(x)=s_1\right\}\cap\bigcap_{i=2}^L\left\{ x\in\mathbb{R}^{n_0}\vert\;S_{h_i}(h_{i-1}\circ\dots\circ h_1(x))=s_i\right\}\\
	  &=&R_{h_1}(s_1)\cap\bigcap_{i=2}^L\left( h_{i-1}\circ\dots\circ h_{1} \right)^{-1}\left( R_{h_i}(s_i) \right),
	\end{eqnarray*}where the sets $R_{h_1}(s_1)\subset \mathbb{R}^{n_0},\dots,R_{h_L}(s_L)\subset\mathbb{R}^{n_{L-1}}$ are convex by Lemma~\ref{lem:Convex}. Now denote by $ \tilde h_{1,s_1},\dots, \tilde h_{L,s_L}$ the affine linear extensions of the restrictions of $h_1,\dots,h_L$ onto the sets $R_{h_1}(s_1), \dots,R_{h_L}(s_L)$ respectively according to Definition~\ref{def:AffineExt}. We will now show by induction that 
	\begin{align}
	  \label{eq:inductionConvex}
	  \begin{split}
	  &R_{h_1}(s_1)\cap\bigcap_{i=2}^{l}\left( h_{i-1}\circ\dots\circ h_{1} \right)^{-1}\left( R_{h_i}(s_i) \right)\\
	=\quad&R_{h_1}(s_1)\cap\bigcap_{i=2}^{l}\left( \tilde h_{i-1,s_{i-1}}\circ\dots\circ \tilde h_{1,s_1} \right)^{-1}\left( R_{h_i}(s_i) \right).
	  \end{split}
	\end{align} for all $l\in\left\{ 1,\dots,L \right\}$. This is clearly true for $l=1$ because in this case the equation states $R_{h_1}(s_1)=R_{h_1}(s_1)$. For the induction step, assume that equation~\eqref{eq:inductionConvex} holds for $l\in\left\{ 1,\dots,L-1 \right\}$. Then
	\begin{align*}
	  &R_{h_1}(s_1)\cap\bigcap_{i=2}^{l+1}\left( h_{i-1}\circ\dots\circ h_{1} \right)^{-1}\left( R_{h_i}(s_i) \right)\\
	  =\quad &\left( R_{h_1}(s_1)\cap\bigcap_{i=2}^{l}\left( h_{i-1}\circ\dots\circ h_{1} \right)^{-1}\left( R_{h_i}(s_i) \right) \right)\cap\left( h_l\circ\dots\circ h_1 \right)^{-1}\left( R_{h_{l+1}}(s_{l+1}) \right)\\
	  =\quad &\left( R_{h_1}(s_1)\cap\bigcap_{i=2}^{l}\left( h_{i-1}\circ\dots\circ h_{1} \right)^{-1}\left( R_{h_i}(s_i) \right) \right)\cap\left( \tilde h_{l,s_l}\circ\dots\circ \tilde h_{1,s_1} \right)^{-1}\left( R_{h_{l+1}}(s_{l+1}) \right)\\
	=\quad&R_{h_1}(s_1)\cap\bigcap_{i=2}^{l+1}\left( \tilde h_{i-1,s_{i-1}}\circ\dots\circ \tilde h_{1,s_1} \right)^{-1}\left( R_{h_i}(s_i) \right)
	\end{align*}
	by our assumption and the fact that $h_l\circ\dots\circ h_1=\tilde h_{l,s_{l}}\circ\dots\circ \tilde h_{1,s_1}$ on the set 
	\begin{equation*}
	  \left( R_{h_1}(s_1)\cap\bigcap_{i=2}^{l}\left( h_{i-1}\circ\dots\circ h_{1} \right)^{-1}\left( R_{h_i}(s_i) \right) \right)\subset \mathbb{R}^{n_0}.
	\end{equation*}By induction, equation~\eqref{eq:inductionConvex} holds in particular for $l=L$ such that 
	\begin{align*}
	  R_{\mathbf{h}}(s)	  \quad =\quad &R_{h_1}(s_1)\cap\bigcap_{i=2}^L\left( h_{i-1}\circ\dots\circ h_{1} \right)^{-1}\left( R_{h_i}(s_i) \right)\\
	   =\quad&R_{h_1}(s_1)\cap\bigcap_{i=2}^{L}\left( \tilde h_{i-1,s_{i-1}}\circ\dots\circ \tilde h_{1,s_1} \right)^{-1}\left( R_{h_i}(s_i) \right).
	\end{align*}
Since the functions $\tilde h_{i-1,s_{i-1}}\circ\dots\circ \tilde h_{1,s_1}$, $i\in\left\{ 2,\dots,L \right\}$ are affine linear, this is an intersection of convex sets by Lemma \ref{lem:AffineConvex}, which is convex itself.
  \end{proof}
\end{lemma}
\begin{definition}
  \label{def:RestrictMulti}
  For $s\in \left\{ 0,1 \right\}^{n_1}\times\cdots\times\left\{ 0,1 \right\}^{n_L}$, denote the restriction of $f_{\mathbf{h}}$ from equation~\eqref{eq:FunF} on the set $R_{\mathbf{h}}(s)$ by 
	$f_{\mathbf{h},s}:
	R_{\mathbf{h}}(s)\to\mathbb{R}^{n_L},\;
	  x\mapsto f_{\mathbf{h}}(x)=h_L\circ\dots\circ h_1 (x).
	  $
\end{definition}
\begin{lemma}
  \label{lem:AffineRepMulti}
  For any $s\in \left\{ 0,1 \right\}^{n_1}\times\cdots\times\left\{ 0,1 \right\}^{n_L}$, the function $f_{\mathbf{h},s}$ is affine linear. Furthermore, for $x\in R_{\mathbf{h}}(s)$, $f_{\mathbf{h},s}(x)$ is explicitly given by 
  \begin{equation}
	\label{eq:ExplicitMulti}
	f_{\mathbf{h},s}(x)=\tilde{\left( h_L \right)}_{s_L}\circ\dots\circ \tilde{\left( h_1 \right)}_{s_1}(x),
  \end{equation}where 
  \begin{equation*}
	\tilde{\left( h_i \right)}_{s_i}(z)=\textnormal{diag}(s_i)\left( W^{(h_i)}z+b^{(h_i)} \right)\quad \textnormal{ for }i\in\left\{ 1,\dots,L \right\}, z\in\mathbb{R}^{n_{i-1}}.
  \end{equation*}
  \begin{proof}
	Let $x\in R_{\mathbf{h}}(s)$. By Definitions~\ref{def:MultiSignat} and \ref{def:MultiSignatReg}, $s=(s_1,\dots,s_L)$ where $s_1=S_{h_1}(x)$ and $s_i=S_{h_i}\left( h_{i-1}\circ\dots\circ h_{1}(x) \right)$ for $i\in\left\{ 2,\dots,L \right\}$. Furthermore, $x\in R_{h_1}(s_1)$ and for $i\in\left\{ 2,\dots,L \right\}$, $h_{i-1}\circ\dots\circ h_1(x)\in R_{h_i}(s_i)$. Hence with the notation from Definition~\ref{def:Restrict}, 
	\begin{equation*}
	  f_{\mathbf{h},s}(x)=f_{\mathbf{h}}(x)=h_{L}\circ\dots\circ h_1(x)=\left( h_L \right)_{s_L}\circ\dots\circ \left( h_1 \right)_{s_1}(x),
	\end{equation*}where all the composed functions on the right-hand side are affine linear by Lemma~\ref{lem:AffineRep}. With the notation of Definition~\ref{def:AffineExt} we obtain equation~\eqref{eq:ExplicitMulti}.
  \end{proof}
\end{lemma}

From Lemma \ref{lem:AffineRepMulti}, we know that $f_{\mathbf{h},s}$ is affine linear. We can extend it as follows:
\begin{definition}
  Let $\tilde f_{\mathbf{h},s}$ denote the affine linear extension of $f_{\mathbf{h},s}$ from Definition~\ref{def:RestrictMulti}.
  \label{def:AffineExtMulti}
\end{definition}
\begin{lemma}
  The collection of sets $\left( R_{\mathbf{h}}(s) \right)_{s\in\mathcal{S}_{\mathbf{h}}}$ form a partition of $\mathbb{R}^{n_0}$.
  \begin{proof}
	By Definition~\ref{def:MultiSignatReg}, these sets are disjoint and furthermore, it holds that
	\begin{equation*}
	  \bigcup_{s\in \mathcal S_{\mathbf{h}}}R_{\mathbf{h}}(s)= \bigcup_{s\in \mathcal S_{\mathbf{h}}} \left\{ x\in\mathbb{R}^{n_0}\;\vert\; S_{\mathbf{h}}(x)=s\right\}  = \left\{ x\in\mathbb{R}^{n_0}\;\vert\; S_{\mathbf{h}}(x)\in\mathcal S_{\mathbf{h}}\right\}=\mathbb{R}^{n_0}.\qedhere
	\end{equation*}
  \end{proof}
  \label{lem:PartitionMulti}
\end{lemma}

Similarly to Corollary~\ref{cor:RepOneLayer}, the previous lemma implies the following statement.

\begin{corollary}
  \label{cor:RepMultiLayers}
  With the notation from Definition~\ref{def:AffineExtMulti}, it holds that
  $f_{\mathbf{h}}=\sum_{s\in \mathcal{S}_{\mathbf{h}}}^{}\mathds{1}_{R_{\mathbf{h}}(s)}\tilde f_{\mathbf{h},s}$.
\end{corollary}
\subsection{Results for reproving Mont\'ufar's bound}
In this section we prove the statements that are used for the motivation in Section~\ref{sec:Motivation}. Note however, that the resulting bound from Corollary~\ref{cor:firstUpperBound} is a special case of our theory, see Section~\ref{sec:Zaslavsky}.
\label{app:firstBound}
\begin{lemma}
  Let $n_0,n_1,n_2\in\mathbb{N}$.  Let $ g_1:\mathbb{R}^{n_0}\to\mathbb{R}^{n_1}$ be an affine linear function, i.e. there exists $A\in\mathbb{R}^{n_1\times n_0}$ and $b\in\mathbb{R}^{n_1}$ such that for all $x\in \mathbb{R}^{n_0}$,
	$g_1(x)=Ax+b$.
Furthermore, assume that $g_2\in\text{RL}(n_1,n_2)$. Then
  \begin{equation*}
	|\left\{ S_{g_2}(  g_1(x) )\vert x\in \mathbb{R}^{n_0} \right\}| \le \sum_{j=0}^{\textnormal{rank}(g_1)}{n_{2}\choose j}.
  \end{equation*}
  \begin{proof}
	By definition, $\textnormal{rank}(g_1)$ is the dimension of the affine linear space 
	\begin{equation*}
	 U:=\left\{ g_1(x)|\;x\in \mathbb{R}^{n_0} \right\}\subset \mathbb{R}^{n_1}. 
	\end{equation*}
	  There exists a affine linear bijective map 
	$  \Phi:U\to\mathbb{R}^{\textnormal{rank}(g_1)}$. 
	Now, it holds that
	\begin{eqnarray*}
	  &&|\left\{ S_{g_2}(  g_1(x) )\vert x\in \mathbb{R}^{n_0} \right\}| =  |\left\{ S_{g_2}(  \Phi^{-1}\circ\Phi\circ g_1(x) )\vert x\in \mathbb{R}^{n_0} \right\}|\\
	  &=&  |\left\{ S_{g_2\circ\Phi^{-1}}(  \Phi\circ g_1(x) )\vert x\in \mathbb{R}^{n_0} \right\}|=  |\left\{ S_{ g_2\circ\Phi^{-1}}(  z )\vert z\in \mathbb{R}^{\textnormal{rank}(g_1)} \right\}|\\
	  &=& |\mathcal{S}_{g_2\circ\Phi^{-1}}|\le \sum_{j=0}^{\textnormal{rank}(g_1)}{n_2\choose j}
	\end{eqnarray*}by Lemmas~\ref{lem:CompositionSignature} and~\ref{lem:OneLayer1975}.
  \end{proof}
  \label{lem:AffineComposition}
\end{lemma}
\begin{lemma}
  \label{lem:UpperboundLemma}
  For $i\in\left\{ 1,\dots,L-1 \right\}$ and given $(s^*_1,\dots,s^*_{i})\in \mathcal{S}_{\mathbf{h}}^{(i)}$, it holds that
  \begin{equation}
	\vert\left\{s_{i+1}\in\left\{ 0,1\right\}^{n_{i+1}} |\;(s^*_1,\dots,s^*_i,s_{i+1})\in\mathcal{S}_{\mathbf{h}}^{(i+1)} \right\} \vert\le \sum_{j=0}^{\min\left( n_0,\vert s^*_1\vert,\dots,\vert s^*_i\vert \right)}{n_{i+1}\choose j}.
	\label{eq:UpperboundLemma}
  \end{equation}
  \begin{proof}
	For $i\in\left\{ 1,\dots,L-1 \right\}$, let $\mathbf{h}^{(i)}=(h_1,\dots,h_i)$ and $R_{\mathbf{h}^{(i)}}$ as in Definition~\ref{def:MultiSignatReg}. Combining this with Definitions~\ref{def:MultiSignat} and~\ref{def:Signatures}, it follows that 
	
	\begin{eqnarray*}
	  &&\vert\left\{s_{i+1}\in\left\{ 0,1 \right\}^{n_{i+1}}|\;(s^*_1,\dots,s^*_i,s_{i+1})\in\mathcal{S}_{\mathbf{h}}^{(i+1)} \right\} \vert\\
	  &=& |\left\{ S_{h_{i+1}}\left( h_{i}\circ\dots\circ h_1(x) \right)\vert x\in R_{\mathbf{h}^{(i)}}(s_1^*,\dots,s_i^*) \right\}|.
	\end{eqnarray*}
	By Lemma~\ref{lem:AffineRepMulti}, the function 
	$  f: R_{\mathbf{h}^{(i)}}((s^*_1,\dots,s^*_i))\to \mathbb{R}^{n_{i+1}},
		x\mapsto h_{i}\circ\dots\circ h_1(x)$
 is affine linear and if we define $\tilde f$ to be the affine linear extension of $f$ to $\mathbb{R}^{n_0}$, it follows that 
	\begin{eqnarray*}
	  && |\left\{ S_{h_{i+1}}\left( h_{i}\circ\dots\circ h_1(x) \right)\vert x\in R_{\mathbf{h}^{(i)}}(s^*_1,\dots,s^*_i) \right\}|= |\left\{ S_{h_{i+1}}( \tilde f(x) )\vert x\in R_{\mathbf{h}^{(i)}}(s^*_1,\dots,s^*_i) \right\}|\\
	  &\le& |\left\{ S_{h_{i+1}}( \tilde f(x) )\vert x\in \mathbb{R}^{n_0} \right\}|\le \sum_{j=0}^{\textnormal{rank}(\tilde f)}{n_{i+1}\choose j}
	\end{eqnarray*}by Lemma~\ref{lem:AffineComposition}. The result now follows from the fact that the rank of $\tilde f$ is bounded by the minimum of the ranks of $h_1,\dots,h_{i+1}$ on the set $R_{\mathbf{h}_i}(s_1^*,\dots,s_i^*)$, which itself is bounded by $\min\left( n_0,|s_1^*|,\dots,|s_i^*| \right)$.
  \end{proof}
\end{lemma}
\begin{lemma}
  \label{lem:main}
  It holds that $\vert\mathcal{S}_{\mathbf{h}}^{(1)}|\le\sum_{j=0}^{n_0}{n_1\choose j}$.
  \begin{proof}
	This follows from Lemma~\ref{lem:OneLayer1975}.
  \end{proof}
\end{lemma}
\begin{theorem}
  \label{thm:main}
  For $i\in \left\{ 1,\dots,L-1 \right\}$, it holds that	
  \begin{equation*}
	|\mathcal{S}^{(i+1)}_{\mathbf{h}}|\le\sum_{\left( s_1,\dots,s_i \right)\in\mathcal{S}^{(i)}_{\mathbf{h}}}^{}\sum_{j=0}^{\min\left( n_0,|s_1|,\dots,|s_i| \right)} {n_{i+1}\choose j}.
  \end{equation*}
  \begin{proof}
	It holds that
	\begin{equation*}
	  \vert \mathcal{S}_{\mathbf{h}}^{(i+1)}\vert=\sum_{\left( s_1,\dots,s_{i+1} \right)\in\mathcal{S}_{\mathbf{h}}^{(i+1)}}1=\sum_{(s_1,\dots,s_i)\in\mathcal{S}_{\mathbf{h}}^{(i)}}\vert\left\{s_{i+1}\in\left\{ 0,1\right\}^{n_i+1}|\;(s_1,\dots,s_{i+1})\in\mathcal{S}_{\mathbf{h}}^{(i+1)} \right\} \vert.
	\end{equation*}
	Now use Lemma~\ref{lem:UpperboundLemma}.
  \end{proof}
\end{theorem}
  \begin{corollary}It holds that
	\label{cor:firstUpperBound}
  \begin{equation*}
	|\mathcal{S}_{\mathbf{h}}\vert\le\prod_{i=0}^{L}\sum_{j=0}^{\min(n_0,\dots,n_i)}{n_{i+1}\choose j}
  \end{equation*}
  \begin{proof}
	For $i\in\left\{ 1,\dots,L-1 \right\}$, Theorem~\ref{thm:main} implies  
  \begin{equation*}
	|\mathcal{S}^{(i+1)}_{\mathbf{h}}|\le\sum_{\left( s_1,\dots,s_i \right)\in\mathcal{S}^{(i)}_{\mathbf{h}}}^{}\sum_{j=0}^{ \min(n_0,\dots,n_i)} {n_{i+1}\choose j}=|\mathcal{S}_{\mathbf{h}}^{(i)}|\sum_{j=0}^{\min\left( n_0,\dots,n_i \right)}{n_{i+1}\choose j}.
  \end{equation*}
It follows that
  \begin{equation*}
	|\mathcal{S}_{\mathbf{h}}\vert=\vert\mathcal{S}_{\mathbf{h}}^{(L)}\vert\le|\mathcal{S}_{\mathbf{h}}^{(L-1)}|\sum_{j=0}^{\min(n_0,\dots,n_{L-1})}{n_L\choose j}\le\dots\le\prod_{i=0}^{L-1}\sum_{j=0}^{\min(n_0,\dots,n_i)}{n_{i+1}\choose j}.\qedhere
  \end{equation*}
  \end{proof}
\end{corollary}
\subsection{Results for the derivation of the framework}
The goal of this section is to prove the main results Theorem~\ref{thm:MainResultPhi} and Corollary~\ref{cor:MainResultMatrix} from Section~\ref{sec:MainResult}. We assume the definitions from Section~\ref{sec:Definitions}.
\label{app:secondBound}
\begin{lemma}
  \label{lem:preceqNorm}
  Assume for $v,w\in V$ that $v\preceq w$. Then $\|v\|_1\le \|w\|_1$.
  \begin{proof}
	By Definition~\ref{def:preceq}, $\|v\|_1=\sum_{j=0}^{\infty}v_j\le\sum_{j=0}^{\infty}w_j=\|w\|_1$.
  \end{proof}
\end{lemma}
\begin{lemma}
  \label{lem:partialOrder}
  The order relation $\preceq$ from Definition~\ref{def:preceq} is a partial order on $V$, i.e. for all $u,v,w\in V$, it holds that
  \begin{itemize}
	\item $u\preceq u$ (reflexivity)
	\item $u\preceq v \land v\preceq u\implies u=v$ (antisymmetry)
	\item $u\preceq v \land v\preceq w\implies u\preceq w$ (transitivity)
  \end{itemize}
  \begin{proof}
	This follows immediately from Definitions~\ref{def:V} and~\ref{def:preceq}.
  \end{proof}
\end{lemma}
\begin{lemma}
  \label{lem:preceqMultipleSummands}
  For $m\in \mathbb{N}_+$ and $a_1,\dots,a_m, b_1,\dots,b_m\in V$ it holds that 
  \begin{equation*}
	\forall i\in \left\{ 1,\dots,m \right\}\;a_i\preceq b_i \implies \sum_{i=1}^{m} a_i\preceq\sum_{i=1}^{m}b_i
  \end{equation*}
  \begin{proof}
	The assumption $\forall i\in \left\{ 1,\dots,m \right\}\;a_i\preceq b_i$ implies
	\begin{equation*}
	\forall J\in \mathbb{N}_0\quad \sum_{j=J}^{\infty}\sum_{i=1}^{m}(a_i)_j=\sum_{i=1}^{m}\sum_{j=J}^{\infty}(a_i)_j\le\sum_{i=1}^{m}\sum_{j=J}^{\infty}(b_i)_j=\sum_{j=J}^{\infty}\sum_{i=1}^{m}(b_i)_j
	\end{equation*}by Definition~\ref{def:preceq}.
  \end{proof}
\end{lemma}
The following two lemmas are an immediate consequence of the clipping function of Definition~\ref{def:alphafunction}.
\begin{lemma}
  \label{lem:alphapreceq}
  For all $i^*\in\mathbb{N}$ and all $v\in V$, it holds that $\textnormal{cl}_{i^*}(v)\preceq v$.
\end{lemma}
\begin{lemma}
  \label{lem:keepRelation}
  Let $v_1,v_2\in V$ and $i^*\in\mathbb{N}$. It holds that $v_1\preceq v_2\implies \textnormal{cl}_{i^*}(v_1)\preceq\textnormal{cl}_{i^*}(v_2)$.
\end{lemma}
\begin{lemma}
  \label{lem:2totheN}
  For $n,n'\in\mathbb{N}_{+}$ with $n\ge n'$, it holds holds that 
  \begin{align*}
	\max\left\{ \mathcal{H}_{n'}\left( \mathcal{S}_h \right)\mid h\in\textnormal{RL}(n,n') \right\}=\sum_{i=0}^{n'}{n'\choose i}{\rm e}_i.
  \end{align*}
  \begin{proof}Let $n,n'\in\mathbb{N}_{+}$. For $n\ge n'$, there exist $h\in\textnormal{RL}(n,n')$ such that all possible signatures are attained, i.e. $\mathcal{S}_{h}=\left\{ 0,1 \right\}^{n'}$. 
  \end{proof}
\end{lemma}
\begin{lemma}
  \label{lem:gammaMin}
Assume that  $\gamma\in\Gamma$ as in equation~\eqref{eq:Gamma}. Then
\begin{equation*}
  \forall n,n'\in\mathbb{N}_{+}\quad \max\left\{ \mathcal{H}_{n'}(\mathcal{S}_h)\mid h\in \textnormal{RL}(n,n') \right\}\preceq \gamma_{\min(n',n),n'}.
\end{equation*}
\begin{proof}
  We only need to consider the case where $n\ge n'$. In this case, Lemma~\ref{lem:2totheN}, shows
  \begin{equation*}
	\max\left\{ \mathcal{H}_{n'}(\mathcal{S}_h)\mid h\in \textnormal{RL}(n,n') \right\}=\sum_{i=0}^{n'}{n'\choose i}{\rm e}_i=\max\left\{ \mathcal{H}_{n}(\mathcal{S}_h)\mid h\in \textnormal{RL}(n',n') \right\}\preceq \gamma_{n',n'}.\qedhere
  \end{equation*}
\end{proof}
\end{lemma}
\begin{lemma}
  \label{lem:monotonicityPhi}
  For $\gamma\in\Gamma$, $n'\in\mathbb{N}_+$ and $v_1,v_2\in V$ the following monotonicity holds:
  \begin{equation*}
	v_1\preceq v_2\implies \varphi^{(\gamma)}_{n'}(v_1)\preceq \varphi^{(\gamma)}_{n'}(v_2)
  \end{equation*}
  \begin{proof}
	This follows from Definition~\ref{def:PhiFunc}, the second property of the bound condition of Definition~\ref{def:boundCondition}, and Lemmas~\ref{lem:keepRelation} and \ref{lem:preceqMultipleSummands}.
  \end{proof}
\end{lemma}

\begin{lemma}
  \label{lem:boundKappaByGamma}
  Assume that  $\gamma\in\Gamma$ as in equation~\eqref{eq:Gamma} and let $l\in\left\{ 1,\dots,L-1 \right\}$, $(s_1^*,\dots,s_{l-1}^*)\in\mathcal{S}_{\mathbf{h}}^{(l-1)}$. Then
\begin{equation*}
  \mathcal{H}_{n_l}\left( \left\{ s_l\in\left\{ 0,1 \right\}^{n_l}\mid (s_1^*,\dots,s_{l-1}^*,s_l)\in\mathcal{S}_{\mathbf{h}}^{(l)} \right\} \right)\preceq \gamma_{\min\left( n_0,|s_1^*|,\dots,|s_{l-1}^*|,n_l \right),n_l}.
\end{equation*}
\begin{proof} Similarly to the proof of Lemma~\ref{lem:UpperboundLemma}, for $\mathbf{h}^{(l-1)}=(h_1,\dots,h_{l-1})$, the function
	\begin{equation*}
	  f:
	  \begin{cases}
		R_{\mathbf{h}^{(l-1)}}((s^*_1,\dots,s^*_{l-1}))&\to \mathbb{R}^{n_{l}}\\
		x&\mapsto h_{l-1}\circ\dots\circ h_1(x)
	  \end{cases}
	\end{equation*} is affine linear and if we define $\tilde f$ to be the affine linear extension of $f$ to $\mathbb{R}^{n_0}$, it holds that 
  \begin{equation*}
	\mathcal{H}_{n_l}\left( \left\{ s_l\in\left\{ 0,1 \right\}^{n_l}\mid (s_1^*,\dots,s_{l-1}^*,s_l)\in\mathcal{S}_{\mathbf{h}}^{(l)} \right\} \right) \preceq \mathcal{H}_{n_l}\left( \left\{ S_{h_l}(\tilde f(x))|x\in\mathbb{R}^{n_0} \right\} \right).\\
\end{equation*}Since $\tilde f$ is affine linear, there exists a bijective affine linear map $\Phi:\tilde f(\mathbb{R}^{n_0})\to\mathbb{R}^{\textnormal{rank}(\tilde f)}$. This implies
  \begin{align*}
  &\mathcal{H}_{n_l}\left( \left\{ s_l\in\left\{ 0,1 \right\}^{n_l}\mid (s_1^*,\dots,s_{l-1}^*,s_l)\in\mathcal{S}_{\mathbf{h}}^{(l)} \right\} \right) \preceq \mathcal{H}_{n_l}\left( \left\{ S_{h_l}(\tilde f(x))|x\in\mathbb{R}^{n_0} \right\} \right)\\
  =\;&\mathcal{H}_{n_l}\left( \left\{ S_{h_l}\left( \Phi^{-1}\circ\Phi\circ \tilde f(x)  \right)\mid x\in\mathbb{R}^{n_0}\right\} \right) =\mathcal{H}_{n_l}\left( \left\{ S_{h_l\circ \Phi^{-1}}\left(z \right)\mid z\in\mathbb{R}^{\textnormal{rank}(\tilde f)}\right\} \right)\\
  =\;&\mathcal{H}_{n_l}\left( \mathcal{S}_{h_l\circ \Phi^{-1}} \right)\preceq\gamma_{\min\left( \textnormal{rank}(\tilde f),n_l \right),n_l}
\end{align*} by the first property of the bound condition from Definition~\ref{def:boundCondition} and by Lemma~\ref{lem:gammaMin} because $h_l\circ\Phi^{-1}\in\textnormal{RL}(\textnormal{rank}(\tilde f),n_l)$ by Lemma~\ref{lem:CompositionSignature}. The statement follows from $\gamma_{\min\left( \textnormal{rank}(\tilde f),n_l \right),n_l}\preceq\gamma_{\min\left( n_0,|s_1^*|,\dots,|s_{l-1}^*|,n_l \right),n_l}$ by the second property of the bound condition since $\textnormal{rank}(\tilde f)$ is bounded by $ \min(n_0,|s^*_1|,\dots,|s^*_{l-1}|)$.
\end{proof}
\end{lemma}

 \begin{lemma} Assume that  $\gamma\in\Gamma$ as in equation~\eqref{eq:Gamma}.
   \label{lem:main2}
   It holds that
   \begin{equation*}
	 \tilde{\mathcal{H}}^{(1)}\left( \mathcal{S}_{\mathbf{h}}^{(1)} \right)\preceq\varphi^{(\gamma)}_{n_1}({\rm e}_{n_0}).
   \end{equation*}
   \begin{proof}
	 Note that $h_1\in\textnormal{RL}(n_0,n_1)$ such that Lemmas~\ref{lem:keepRelation} and~\ref{lem:gammaMin} imply 
	 \begin{equation*}
	   \tilde{\mathcal{H}}^{(1)}\left( \mathcal{S}_{\mathbf{h}}^{(1)} \right)=\textnormal{cl}_{n_0}\left( \mathcal{H}_{n_1}\left( \mathcal{S}_{h_1} \right) \right)=\textnormal{cl}_{\min(n_0,n_1)}\left( \mathcal{H}_{n_1}\left( \mathcal{S}_{h_1} \right) \right)\preceq\textnormal{cl}_{n_0}\left( \gamma_{\min(n_0,n_1),n_1} \right)=\varphi_{n_1}\left( {\rm e}_{n_0} \right)\qedhere
	 \end{equation*}
   \end{proof}
 \end{lemma}
 \begin{proposition}
   \label{prop:main2}
   For $\gamma\in\Gamma$ and $l\in\left\{ 2,\dots,L \right\}$ it holds that
   \begin{equation*}
	 \tilde{\mathcal{H}}^{(l)}(\mathcal{S}_{\mathbf{h}}^{(l)})\preceq \varphi^{(\gamma)}_{n_l}(\tilde{\mathcal{H}}^{(l-1)}(\mathcal{S}_{\mathbf{h}}^{(l-1)})).
   \end{equation*}
   \begin{proof} The statement is implied by the following calculation, where the step indicated by $(*)$ follows from the Lemmas~\ref{lem:boundKappaByGamma}, \ref{lem:keepRelation} and~\ref{lem:preceqMultipleSummands}.
\begin{eqnarray*}
  \tilde{\mathcal{H}}^{(l)}\left( \mathcal{S}_{\mathbf{h}}^{(l)} \right)&=& \left( \sum_{(s_1,\dots,s_{l-1})\in\mathcal{S}_{\mathbf{h}}^{(l-1)}}^{}\sum_{s_l\in\left\{ 0,1 \right\}^{n_l}}^{}\mathds{1}_{\mathcal{S}_{\mathbf{h}}^{(l)}}\left( (s_1,\dots,s_l )\right)\mathds{1}_{ \left\{ j \right\}}\left( \min(n_0,|s_1|,\dots,|s_l|) \right) \right)_{j\in\mathbb{N}}\\
  &=&  \sum_{(s_1,\dots,s_{l-1})\in\mathcal{S}_{\mathbf{h}}^{(l-1)}}^{}\left(\sum_{s_l\in\left\{ 0,1 \right\}^{n_l}}^{}\mathds{1}_{\mathcal{S}_{\mathbf{h}}^{(l)}}\left( (s_1,\dots,s_l )\right)\mathds{1}_{ \left\{ j \right\}}\left( \min(n_0,|s_1|,\dots,|s_l|) \right) \right)_{j\in\mathbb{N}}\\
  &=&  \sum_{(s_1,\dots,s_{l-1})\in\mathcal{S}_{\mathbf{h}}^{(l-1)}}^{}\textnormal{cl}_{\min(n_0,|s_1|,\dots,|s_{l-1}|,n_l)}\left(\left( \sum_{s_l\in\left\{ 0,1 \right\}^{n_l}}^{}\mathds{1}_{\mathcal{S}_{\mathbf{h}}^{(l)}}\left( (s_1,\dots,s_l )\right)\mathds{1}_{ \left\{ j \right\}}\left( |s_l| \right) \right)_{j\in\mathbb{N}} \right)\\
  &=&  \sum_{(s_1,\dots,s_{l-1})\in\mathcal{S}_{\mathbf{h}}^{(l-1)}}^{}\textnormal{cl}_{\min(n_0,|s_1|,\dots,|s_{l-1}|,n_l)}(\underbrace{\mathcal{H}_{n_l}( \left\{ s_l\in\left\{ 0,1 \right\}^{n_l}|(s_1,\dots,s_l)\in\mathcal{S}_{\mathbf{h}}^{(l)} \right\}}_{\preceq \gamma_{\min(n_0,|s_1|,\dots,|s_{l-1}|,n_l),n_l}})  )\\
  &\overset{(*)}{\preceq}&  \sum_{(s_1,\dots,s_{l-1})\in\mathcal{S}_{\mathbf{h}}^{(l-1)}}^{}\textnormal{cl}_{\min(n_0,|s_1|,\dots,|s_{l-1}|,n_l)}(\gamma_{\min(n_0,|s_1|,\dots,|s_{l-1}|,n_l),n_l}  )\\
  &=&  \sum_{j=0}^{\infty}\tilde{\mathcal{H}}^{(l-1)}\left( \mathcal{S}_{\mathbf{h}}^{(l-1)} \right)_j \textnormal{cl}_{\min(j,n_l)}\left(\gamma_{\min(j,n_l),n_l}\right)=  \sum_{j=0}^{\infty}\tilde{\mathcal{H}}^{(l-1)}\left( \mathcal{S}_{\mathbf{h}}^{(l-1)} \right)_j \varphi^{(\gamma)}_{n_l}({\rm e}_j)\\
  &=&   \varphi^{(\gamma)}_{n_l}\left( \sum_{j=0}^{\infty}\tilde{\mathcal{H}}^{(l-1)}\left( \mathcal{S}_{\mathbf{h}}^{(l-1)} \right)_j{\rm e}_j \right)=\varphi^{(\gamma)}_{n_l}\left( \tilde{\mathcal{H}}^{(l-1)}\left( \mathcal{S}_{\mathbf{h}}^{(l-1)} \right) \right).\qedhere
\end{eqnarray*}
   \end{proof}
 \end{proposition}
 \begin{proof}[Proof of Theorem~\ref{thm:MainResultPhi}]
   First note that $\vert\mathcal{S}_{\mathbf{h}}^{(L)}|=\|\tilde{\mathcal{H}}^{(L)}\left( \mathcal{S}_{\mathbf{h}}^{(L)} \right)\|_1$. Now the above Proposition~\ref{prop:main2} and Lemmas~\ref{lem:preceqNorm}, ~\ref{lem:monotonicityPhi} and \ref{lem:main2} imply that
 \begin{eqnarray*}
   |\mathcal{S}_{\mathbf{h}}|&=& |\mathcal{S}_{\mathbf{h}}^{(L)}|=\|\tilde{\mathcal{H}}^{(L)}\left( \mathcal{S}_{\mathbf{h}}^{(L)} \right)\|_1\le\|\varphi^{(\gamma)}_{n_L}(\tilde{\mathcal{H}}^{(L-1)}(\mathcal{S}_{\mathbf{h}}^{(i-1)}))\|_1\le\dots\\
   &\le& \|\varphi^{(\gamma)}_{n_L}\circ\dots\circ\varphi^{(\gamma)}_{n_2}(\tilde{\mathcal{H}}^{(1)}(\mathcal{S}_{\mathbf{h}}^{(1)}))\|_1\le\|\varphi^{(\gamma)}_{n_L}\circ\dots\circ\varphi^{(\gamma)}_{n_1}({\rm e}_{n_0})\|_1.\qedhere
 \end{eqnarray*}
 \end{proof}
\begin{proof}[Proof of Corollary~\ref{cor:MainResultMatrix}]Note that for any $n'\in\mathbb{N}_{+}$ and any $v\in V$ the image 
   \begin{equation}
	 \varphi_{n'}^{(\gamma)}(v)=\sum_{n=0}^{\infty}v_n \textnormal{cl}_{\min(n,n')}(\gamma_{\min(n,n'),n'})=\sum_{n=0}^{n'}(\textnormal{cl}_{n'}(v))_n\textnormal{cl}_{n}\left(\gamma_{n,n'}  \right)
	\label{eq:phiClipping}
  \end{equation} has zeros at all indices larger than $n'$ by construction. Furthermore, for $N\in\mathbb{N}$ large enough such that $v_n=0$ for all $n> N$ the Definition~\ref{def:connMatrix} of $M_{N,n'}$ implies  
  \begin{equation*}
  (\textnormal{cl}_{n'}(v))_i=
  \begin{cases}
	\left( M_{N,n'} 
	\begin{pmatrix}
	  v_0\\
	  \vdots\\
	  v_N
  \end{pmatrix} \right)_{i+1}&\quad\textnormal{ for }i\in \left\{ 0,\dots,n' \right\}\\
	0&\quad\textnormal{ else.}
  \end{cases}
\end{equation*}Together with equation~\eqref{eq:phiClipping} and Definition~\ref{def:boundMatrix} this implies for all $n'\in\mathbb{N}_+$, $v\in V$ and all $N\in\mathbb{N}$ such that $v_n=0$ for all $n>N$:
  \begin{equation*}
	\left(\varphi_{n'}^{(\gamma)}(v)\right)_i=
	\begin{cases}
	  \sum_{n=0}^{n'}\left( 
	  M_{N,n'} 
	\begin{pmatrix}
	  v_0\\
	  \vdots\\
	  v_N
	\end{pmatrix} \right)
	_{n+1}\underbrace{\textnormal{cl}_{n}\left(\gamma_{n,n'}  \right)_{i}}_{\left( B_{n'}^{(\gamma)} \right)_{(i+1),(n+1)}}=\left( B_{n'}^{(\gamma)}M_{N,n'} 
		\begin{pmatrix}
	  v_0\\
	  \vdots\\
	  v_N
	\end{pmatrix}
  \right)_i\quad &\textnormal{ if }i\le n'\\
	  0\quad &\textnormal{if }i>n'\\
	\end{cases}
  \end{equation*}
  A recursive application of the above statement on equation~\eqref{eq:MainRecursivePhi} yields the desired equation~\eqref{eq:MainRecursiveMatrix}
 \end{proof}
\subsection{Analysis of the binomial bound matrices}
  \label{app:decompBound}
  In this section, we want to show how we can decompose the binomial bound matrices from Section~\ref{sec:BinomialCoefficients}. These are the bound matrices from Definition~\ref{def:boundMatrix} with $\gamma$ as in equation~\eqref{eq:binomialGamma}. For these matrices, we provide an explicit formula for a Jordan-like decomposition such that arbitrary powers can easily be computed in a closed-form expression, see Lemma~\ref{lem:CDecomp}. 
  \begin{definition}
	\label{def:3Mat}
	For $n\in \mathbb{N}$ and $\xi_j=\sum_{i=0}^{j-1}{n\choose i}, j\in\left\{ 1,\dots,\lceil \tfrac{n+1}{2}\rceil\right\}$ we define the following matrices. We distinguish the cases when $n$ is odd and when $n$ is even. Let

  \begin{eqnarray*}
	P_{2m}&=& 
\left(
\begin{array}{c|c}

  \begin{matrix}
	\xi_1 & 0 &\cdots& \cdots & 0\\
	0  & \xi_2-\xi_1 &\ddots& & \vdots\\
	\vdots  & \ddots &\ddots& \ddots& \vdots\\
	\vdots  &  &\ddots& \ddots& 0\\
	0  &\cdots &\cdots& 0& \xi_m-\xi_{m-1}\\
  \end{matrix} & 
  
  \begin{matrix}
	0 & \cdots &\cdots & \cdots & 0\\
	\vdots  &  & & & \vdots\\
	\vdots  &  & & & \vdots\\
	\vdots  &  & & & \vdots\\
	0  &\cdots &\cdots& \cdots& 0\\
  \end{matrix}  \\
\hline
  \begin{matrix}
	0 & \cdots &\cdots & \cdots & 0\\
	\vdots  &  & & & \vdots\\
	\vdots  &  & & & \vdots\\
	\vdots  &  & & & \vdots\\
	0  &\cdots &\cdots& \cdots& 0\\
  \end{matrix} &
  \begin{matrix}
	1 & -1 &0 & \cdots & 0\\
	0  & \ddots &\ddots&\ddots & \vdots\\
	\vdots  & \ddots &\ddots& \ddots& 0\\
	\vdots  &  &\ddots& \ddots& -1\\
	0  &\cdots &\cdots& 0& 1\\
  \end{matrix} 
\end{array}
\right)_{(2m)\times(2m)}\\
	J_{2m}&=& 
\left(
\begin{array}{c|c}

  \begin{matrix}
	\xi_1 & 0 &\cdots& \cdots & 0\\
	0  & \xi_2 &\ddots& & \vdots\\
	\vdots  & \ddots &\ddots& \ddots& \vdots\\
	\vdots  &  &\ddots& \ddots& 0\\
	0  &\cdots &\cdots& 0& \xi_m\\
  \end{matrix} & 
  
  \begin{matrix}
	0 & \cdots &\cdots & 0 & 1\\
	\vdots  &  &\iddots &\iddots & 0\\
	\vdots  &\iddots  &\iddots &\iddots & \vdots\\
	0  &\iddots &\iddots& & \vdots\\
	1  &0 &\cdots& \cdots& 0\\
  \end{matrix}  \\
\hline
  \begin{matrix}
	0 & \cdots &\cdots & \cdots & 0\\
	\vdots  &  & & & \vdots\\
	\vdots  &  & & & \vdots\\
	\vdots  &  & & & \vdots\\
	0  &\cdots &\cdots& \cdots& 0\\
  \end{matrix} &
  \begin{matrix}
	\xi_{m} & 0 &\cdots & \cdots & 0\\
	0  & \ddots &\ddots&& \vdots\\
	\vdots  & \ddots &\ddots &\ddots & 0\\
	\vdots  &  &\ddots& \ddots& 0\\
	0  &\cdots &\cdots& 0& \xi_1\\
  \end{matrix} 
\end{array}
\right)_{(2m)\times(2m)}\\
	C_{2m}&=& \left(
\begin{array}{c|c}

  \begin{matrix}
	\xi_1 & 0 &\cdots& \cdots & 0\\
	0  & \xi_2 &\ddots& & \vdots\\
	\vdots  & \ddots &\ddots& \ddots& \vdots\\
	\vdots  &  &\ddots& \ddots& 0\\
	0  &\cdots &\cdots& 0& \xi_m\\
  \end{matrix} & 
  
  \begin{matrix}
	0 & \cdots &\cdots & 0 & \xi_1\\
	\vdots  &  &\iddots &\scriptstyle{\xi_2-\xi_1} & \scriptstyle{\xi_2-\xi_1}\\
	\vdots  &\iddots  &\iddots &  & \vdots\\
	0  &\scriptstyle{\xi_{m-1}-\xi_{m-2}} &\cdots &\cdots & \scriptstyle{\xi_{m-1}-\xi_{m-2}}\\
	\scriptstyle{\xi_{m}-\xi_{m-1}}  &\cdots &\cdots& \cdots& \scriptstyle{\xi_m-\xi_{m-1}}\\
  \end{matrix}  \\
\hline
  \begin{matrix}
	0 & \cdots &\cdots & \cdots & 0\\
	\vdots  &  & & & \vdots\\
	\vdots  &  & & & \vdots\\
	\vdots  &  & & & \vdots\\
	0  &\cdots &\cdots& \cdots& 0\\
  \end{matrix} &
  \begin{matrix}
	\xi_{m} & \scriptstyle{\xi_m-\xi_{m-1}} &\cdots & \cdots & \scriptstyle{\xi_{m}-\xi_{m-1}}\\
	0  & \xi_{m-1} &\scriptstyle{\xi_{m-1}-\xi_{m-2}}&\cdots & \scriptstyle{\xi_{m-1}-\xi_{m-2}}\\
	\vdots  & \ddots &\ddots & &\vdots\\
  \vdots  &  &\ddots& \xi_2& \scriptstyle{\xi_2-\xi_1}\\
	0  &\cdots &\cdots& 0& \xi_1\\
  \end{matrix} 

\end{array}
\right)_{(2m)\times(2m)}
  \end{eqnarray*}
when $n$ is even, i.e. $m:=\tfrac{n}{2}\in\mathbb{N}$. Otherwise, when $n$ is odd, i.e. $m:=\tfrac{n-1}{2}\in\mathbb{N}$ we define
  \begin{eqnarray*}
	P_{2m+1}&=& 
\left(
\begin{array}{c|c}

  \begin{matrix}
	\xi_1 & 0 &\cdots& \cdots & 0\\
	0  & \xi_2-\xi_1 &\ddots& & \vdots\\
	\vdots  & \ddots &\ddots& \ddots& \vdots\\
	\vdots  &  &\ddots& \ddots& 0\\
	0  &\cdots &\cdots& 0& \xi_m-\xi_{m-1}\\
  \end{matrix} & 
  
  \begin{matrix}
	0 & \cdots &\cdots & \cdots & 0\\
	\vdots  &  & & & \vdots\\
	\vdots  &  & & & \vdots\\
	\vdots  &  & & & \vdots\\
	0  &\cdots &\cdots& \cdots& 0\\
  \end{matrix}  \\
\hline
  \begin{matrix}
	0 & \cdots &\cdots & \cdots & 0\\
	\vdots  &  & & & \vdots\\
	\vdots  &  & & & \vdots\\
	\vdots  &  & & & \vdots\\
	0  &\cdots &\cdots& \cdots& 0\\
  \end{matrix} &
  \begin{matrix}
	1 & -1 &0 & \cdots & 0\\
	0  & \ddots &\ddots&\ddots & \vdots\\
	\vdots  & \ddots &\ddots& \ddots& 0\\
	\vdots  &  &\ddots& \ddots& -1\\
	0  &\cdots &\cdots& 0& 1\\
  \end{matrix} 
\end{array}
\right)_{(2m+1)\times(2m+1)}\\
	J_{2m+1}&=& 
\left(
\begin{array}{c|c}

  \begin{matrix}
	\xi_1 & 0 &\cdots& \cdots & 0\\
	0  & \xi_2 &\ddots& & \vdots\\
	\vdots  & \ddots &\ddots& \ddots& \vdots\\
	\vdots  &  &\ddots& \ddots& 0\\
	0  &\cdots &\cdots& 0& \xi_m\\
  \end{matrix} & 
  
  \begin{matrix}
	0 & \cdots &\cdots & 0 & 1\\
	\vdots  &  &\iddots &\iddots & 0\\
	\vdots  &\iddots  &\iddots &\iddots & \vdots\\
	0  &1 &0& \cdots& 0\\
  \end{matrix}  \\
\hline
  \begin{matrix}
	0 & \cdots &\cdots & \cdots & 0\\
	\vdots  &  & & & \vdots\\
	\vdots  &  & & & \vdots\\
	\vdots  &  & & & \vdots\\
	0  &\cdots &\cdots& \cdots& 0\\
  \end{matrix} &
  \begin{matrix}
	\xi_{m+1} & 0 &\cdots & \cdots & 0\\
	0  & \ddots &\ddots&& \vdots\\
	\vdots  & \ddots &\ddots &\ddots & 0\\
	\vdots  &  &\ddots& \ddots& 0\\
	0  &\cdots &\cdots& 0& \xi_1\\
  \end{matrix} 
\end{array}
\right)_{(2m+1)\times(2m+1)}\\
C_{2m+1}&=& \left(
\begin{array}{c|c}

  \begin{matrix}
	\xi_1 & 0 &\cdots& \cdots & 0\\
	0  & \xi_2 &\ddots& & \vdots\\
	\vdots  & \ddots &\ddots& \ddots& \vdots\\
	\vdots  &  &\ddots& \ddots& 0\\
	0  &\cdots &\cdots& 0& \xi_m\\
  \end{matrix} & 
  
  \begin{matrix}
	0 & \cdots &\cdots & 0 & \xi_1\\
	\vdots  &  &\iddots &\scriptstyle{\xi_2-\xi_1} & \scriptstyle{\xi_2-\xi_1}\\
	\vdots  &\iddots  &\iddots &  & \vdots\\
	0  &\scriptstyle{\xi_{m}-\xi_{m-1}} &\cdots &\cdots & \scriptstyle{\xi_{m}-\xi_{m-1}}\\
  \end{matrix}  \\
\hline
  \begin{matrix}
	0 & \cdots &\cdots & \cdots & 0\\
	\vdots  &  & & & \vdots\\
	\vdots  &  & & & \vdots\\
	\vdots  &  & & & \vdots\\
	0  &\cdots &\cdots& \cdots& 0\\
  \end{matrix} &
  \begin{matrix}
	\xi_{m+1} & \scriptstyle{\xi_{m+1}-\xi_{m}} &\cdots & \cdots & \scriptstyle{\xi_{m+1}-\xi_{m}}\\
	0  & \xi_{m} &\scriptstyle{\xi_{m}-\xi_{m-1}}&\cdots & \scriptstyle{\xi_{m}-\xi_{m-1}}\\
	\vdots  & \ddots &\ddots & &\vdots\\
  \vdots  &  &\ddots& \xi_2& \scriptstyle{\xi_2-\xi_1}\\
	0  &\cdots &\cdots& 0& \xi_1\\
  \end{matrix} 
\end{array}
\right)_{(2m+1)\times(2m+1)}.
  \end{eqnarray*}

  \end{definition}
  \begin{lemma}
	The inverse of $P_n$ as in Definitions~\ref{def:3Mat} is given by 
  \begin{equation}
	\label{eq:invPEven}
	P_{2m}^{-1}=
\left(
\begin{array}{c|c}

  \begin{matrix}
	\frac{1}{\xi_1} & 0 &\cdots& \cdots & 0\\
	0  & \frac{1}{\xi_2-\xi_1} &\ddots& & \vdots\\
	\vdots  & \ddots &\ddots& \ddots& \vdots\\
	\vdots  &  &\ddots& \ddots& 0\\
	0  &\cdots &\cdots& 0& \frac{1}{\xi_m-\xi_{m-1}}\\
  \end{matrix} & 
  
  \begin{matrix}
	0 & \cdots &\cdots & \cdots & 0\\
	\vdots  &  & & & \vdots\\
	\vdots  &  & & & \vdots\\
	\vdots  &  & & & \vdots\\
	0  &\cdots &\cdots& \cdots& 0\\
  \end{matrix}  \\
\hline
  \begin{matrix}
	0 & \cdots &\cdots & \cdots & 0\\
	\vdots  &  & & & \vdots\\
	\vdots  &  & & & \vdots\\
	\vdots  &  & & & \vdots\\
	0  &\cdots &\cdots& \cdots& 0\\
  \end{matrix} &
  \begin{matrix}
	1 & \cdots &\cdots & \cdots & 1\\
	0  & \ddots & & & \vdots\\
	\vdots  & \ddots &\ddots& & \vdots\\
	\vdots  &  &\ddots& \ddots& \vdots\\
	0  &\cdots &\cdots& 0& 1\\
  \end{matrix} 
\end{array}
\right)_{(2m)\times(2m)}
  \end{equation}if $n=2m$ for some $m\in\mathbb{N}$ and by 
  \begin{equation}
	\label{eq:invPOdd}
	P_{2m+1}^{-1}=
\left(
\begin{array}{c|c}

  \begin{matrix}
	\frac{1}{\xi_1} & 0 &\cdots& \cdots & 0\\
	0  & \frac{1}{\xi_2-\xi_1} &\ddots& & \vdots\\
	\vdots  & \ddots &\ddots& \ddots& \vdots\\
	\vdots  &  &\ddots& \ddots& 0\\
	0  &\cdots &\cdots& 0& \frac{1}{\xi_m-\xi_{m-1}}\\
  \end{matrix} & 
  
  \begin{matrix}
	0 & \cdots &\cdots & \cdots & 0\\
	\vdots  &  & & & \vdots\\
	\vdots  &  & & & \vdots\\
	\vdots  &  & & & \vdots\\
	0  &\cdots &\cdots& \cdots& 0\\
  \end{matrix}  \\
\hline
  \begin{matrix}
	0 & \cdots &\cdots & \cdots & 0\\
	\vdots  &  & & & \vdots\\
	\vdots  &  & & & \vdots\\
	\vdots  &  & & & \vdots\\
	0  &\cdots &\cdots& \cdots& 0\\
  \end{matrix} &
  \begin{matrix}
	1 & \cdots &\cdots & \cdots & 1\\
	0  & \ddots & & & \vdots\\
	\vdots  & \ddots &\ddots& & \vdots\\
	\vdots  &  &\ddots& \ddots& \vdots\\
	0  &\cdots &\cdots& 0& 1\\
  \end{matrix} 
\end{array}
\right)_{(2m+1)\times(2m+1)}
  \end{equation}
  otherwise for $n=2m+1$, $m\in\mathbb{N}$. In equation~\eqref{eq:invPEven} the lower right partition matrix is of dimension $m\times m$ whereas in equation~\eqref{eq:invPOdd}, it is of dimension $m+1\times m+1$.
  \begin{proof}
	One easily checks that $P_nP_{n}^{-1}=I_{n}$ with $P_n$ from Definition~\ref{def:3Mat} and $P_{n}^{-1}$ as above.
  \end{proof}
  \end{lemma}
  \begin{lemma}
	\label{lem:CDecomp}
	For $n\in\mathbb{N}$ and matrices $C_n$, $P_n$ and $J_n$ as in Definition~\ref{def:3Mat}, we have a Jordan-like decomposition of the form
	\begin{equation*}
	  C_{n}=P_nJ_nP^{-1}_n
	\end{equation*}
	\begin{proof}
	  
  A straight-forward calculation yields
  \begin{eqnarray*}
	P_{2m}J_{2m}P^{-1}_{2m}&=& P_{2m}
\left(
\begin{array}{c|c}
  \begin{matrix}
	\frac{\xi_1}{\xi_1} & 0 &\cdots& \cdots & 0\\
	0  & \frac{\xi_2}{\xi_2-\xi_1} &\ddots& & \vdots\\
	\vdots  & \ddots &\ddots& \ddots& \vdots\\
	\vdots  &  &\ddots& \ddots& 0\\
	0  &\cdots &\cdots& 0& \frac{\xi_m}{\xi_m-\xi_{m-1}}\\
  \end{matrix} & 
  
  \begin{matrix}
	0 & \cdots &\cdots & 0 & 1\\
	\vdots  &  &\iddots &\iddots & \vdots\\
	\vdots  &\iddots  &\iddots &  & \vdots\\
	0  &\iddots & & & \vdots\\
	1  &\cdots &\cdots& \cdots& 1\\
  \end{matrix}  \\
\hline
  \begin{matrix}
	0 & \cdots &\cdots & \cdots & 0\\
	\vdots  &  & & & \vdots\\
	\vdots  &  & & & \vdots\\
	\vdots  &  & & & \vdots\\
	0  &\cdots &\cdots& \cdots& 0\\
  \end{matrix} &
  \begin{matrix}
	\xi_{m} & \cdots &\cdots & \cdots & \xi_m\\
	0  & \xi_{m-1} &\cdots& \cdots& \xi_{m-1}\\
	\vdots  & \ddots &\ddots &\cdots & \vdots\\
	\vdots  &  &\ddots& \ddots& \vdots\\
	0  &\cdots &\cdots& 0& \xi_1\\
  \end{matrix} 
\end{array}
\right)\\
&=& C_{2m}
  \end{eqnarray*} when $n$ is even, i.e. $m:=\tfrac{n}{2}\in\mathbb{N}$ and
  \begin{eqnarray*}
	&&P_{2m+1}J_{2m+1}P^{-1}_{2m+1}\\&=& P_{2m+1}
\left(
\begin{array}{c|c}

  \begin{matrix}
	\frac{\xi_1}{\xi_1} & 0 &\cdots& \cdots & 0\\
	0  & \frac{\xi_2}{\xi_2-\xi_1} &\ddots& & \vdots\\
	\vdots  & \ddots &\ddots& \ddots& \vdots\\
	\vdots  &  &\ddots& \ddots& 0\\
	0  &\cdots &\cdots& 0& \frac{\xi_m}{\xi_m-\xi_{m-1}}\\
  \end{matrix} & 
  
  \begin{matrix}
	0 & \cdots &\cdots & 0 & 1\\
	\vdots  &  &\iddots &\iddots & \vdots\\
	\vdots  &\iddots  &\iddots &  & \vdots\\
	0  &1 &\cdots& \cdots& 1\\
  \end{matrix}  \\
\hline
  \begin{matrix}
	0 & \cdots &\cdots & \cdots & 0\\
	\vdots  &  & & & \vdots\\
	\vdots  &  & & & \vdots\\
	\vdots  &  & & & \vdots\\
	0  &\cdots &\cdots& \cdots& 0\\
  \end{matrix} &
  \begin{matrix}
	\xi_{m+1} & \cdots &\cdots & \cdots & \xi_{m+1}\\
	0  & \xi_m &\cdots&\cdots & \xi_m\\
	\vdots  & \ddots &\ddots &\cdots & \vdots\\
	\vdots  &  &\ddots& \ddots& \vdots\\
	0  &\cdots &\cdots& 0& \xi_1\\
  \end{matrix} 
\end{array}
\right)\\
&=&  C_{2m+1}
  \end{eqnarray*}when $n$ is odd, i.e. $m:=\tfrac{n-1}{2}\in\mathbb{N}$.
	\end{proof}
  \end{lemma} 
  From now on let $B_n$ be as in Section~\ref{sec:BinomialCoefficients}, i.e. let $B_n=B_n^{(\gamma)}$ as in Definition~\ref{def:boundMatrix} with $\gamma$ as in equation~\eqref{eq:binomialGamma}.
  \begin{proposition}
	\label{prop:BC}
	For $n\in\mathbb{N}_{+}$, the binomial bound matrix $B_n$ is $B_n=C_{n+1}$.
	\begin{proof}
	  One easily checks that the construction of the binomial bound matrices explained in Section~\ref{sec:BinomialCoefficients} yields exactly the matrix $C_{n+1}$ for $B_n$.
	\end{proof}
  \end{proposition} 
  This is an important result since it allows us to compute powers of the bound matrices explicitly.
  \begin{corollary}
	\label{cor:powerB}
	For $l, n\in\mathbb{N}_+$, the $l$-th power of $B_n$ is given by
	\begin{equation*}
	  B^l_n=P_{n-1}J^{l}_{n-1}P^{-1}_{n-1}
	\end{equation*}
	\begin{proof}
	  This follows from Proposition~\ref{prop:BC} and Lemma~\ref{lem:CDecomp}.
	\end{proof}
  \end{corollary}
  The previous corollary is useful because this expression can be easily calculated.
  \begin{lemma}
	\label{lem:powerJ}
	For $n\in\mathbb{N}_+$ and $J_n$ as in Definition~\ref{def:3Mat}, it holds that
\begin{equation*}
	J^l_{2m}=
\left(
\begin{array}{c|c}

  \begin{matrix}
	\xi_1^l & 0 &\cdots& \cdots & 0\\
	0  & \xi_2^l &\ddots& & \vdots\\
	\vdots  & \ddots &\ddots& \ddots& \vdots\\
	\vdots  &  &\ddots& \ddots& 0\\
	0  &\cdots &\cdots& 0& \xi_m^l\\
  \end{matrix} & 
  
  \begin{matrix}
	0 & \cdots &\cdots & 0 & l\xi_1^{(l-1)}\\
	\vdots  &  &\iddots &\iddots & 0\\
	\vdots  &\iddots  &\iddots &\iddots & \vdots\\
	0  &\iddots &\iddots& & \vdots\\
	l\xi_m^{(l-1)}  &0 &\cdots& \cdots& 0\\
  \end{matrix}  \\
\hline
  \begin{matrix}
	0 & \cdots &\cdots & \cdots & 0\\
	\vdots  &  & & & \vdots\\
	\vdots  &  & & & \vdots\\
	\vdots  &  & & & \vdots\\
	0  &\cdots &\cdots& \cdots& 0\\
  \end{matrix} &
  \begin{matrix}
	\xi^l_{m} & 0 &\cdots & \cdots & 0\\
	0  & \ddots &\ddots&& \vdots\\
	\vdots  & \ddots &\ddots &\ddots & 0\\
	\vdots  &  &\ddots& \ddots& 0\\
	0  &\cdots &\cdots& 0& \xi_1^l\\
  \end{matrix} 
\end{array}
\right)_{(2m)\times(2m)}\\
	\end{equation*}
when $n$ is even, i.e. $m:=\tfrac{n}{2}\in\mathbb{N}$ and
	\begin{equation*}
	J_{2m+1}^l= 
\left(
\begin{array}{c|c}

  \begin{matrix}
	\xi_1^l & 0 &\cdots& \cdots & 0\\
	0  & \xi_2^l &\ddots& & \vdots\\
	\vdots  & \ddots &\ddots& \ddots& \vdots\\
	\vdots  &  &\ddots& \ddots& 0\\
	0  &\cdots &\cdots& 0& \xi_m^l\\
  \end{matrix} & 
  
  \begin{matrix}
	0 & \cdots &\cdots & 0 & l\xi_1^{(l-1)}\\
	\vdots  &  &\iddots &\iddots & 0\\
	\vdots  &\iddots  &\iddots &\iddots & \vdots\\
	0  &l\xi_m^{(l-1)} &0& \cdots& 0\\
  \end{matrix}  \\
\hline
  \begin{matrix}
	0 & \cdots &\cdots & \cdots & 0\\
	\vdots  &  & & & \vdots\\
	\vdots  &  & & & \vdots\\
	\vdots  &  & & & \vdots\\
	0  &\cdots &\cdots& \cdots& 0\\
  \end{matrix} &
  \begin{matrix}
	\xi^l_{m+1} & 0 &\cdots & \cdots & 0\\
	0  & \ddots &\ddots&& \vdots\\
	\vdots  & \ddots &\ddots &\ddots & 0\\
	\vdots  &  &\ddots& \ddots& 0\\
	0  &\cdots &\cdots& 0& \xi_1^l\\
  \end{matrix} 
\end{array}
\right)_{(2m+1)\times(2m+1)}
\end{equation*} when $n$ is odd, i.e. $m:=\tfrac{n-1}{2}\in\mathbb{N}$.
\begin{proof}
 We omit this easy proof. 
\end{proof}
  \end{lemma}
  We have given the theory how to explicitly calculate an arbitrary natural power of any of the binomial bound matrices from Section~\ref{sec:BinomialCoefficients}. This leads to the following corollary.  
  \begin{corollary}
	\label{cor:powerBNorm}
	Let $n\in\mathbb{N}_+$, $i\in\left\{ 0,\dots,n \right\}$ and $l>0$. It holds that
	  \begin{equation*}
		\|B^l_{n}e_{i+1}\|_1=
	\begin{cases}
 (\sum_{j=0}^{i}{n\choose j})^l &\quad \textnormal{ if }i\le\lfloor n/2\rfloor\\  
\left( \sum_{j=0}^{\lfloor n/2\rfloor}{n\choose j} \right)^{l}+l\sum_{s=\lfloor n/2\rfloor +1}^{i}\left( \sum_{j=0}^{n-s}{n\choose j} \right)^{l-1}{n\choose n-s}&\quad \textnormal{ if }i>\lfloor n/2\rfloor.
	\end{cases}
	  \end{equation*}

	\begin{proof}
  If $n$ is even, set $m=\tfrac{n}{2}$. In this case, an explicit calculation using Corollary~\ref{cor:powerB} and Lemma~\ref{lem:powerJ} yields
  \begin{align*}
	&\|B_{2m}^l e_{i+1}\|_1=\|P_{2m+1}J^l_{2m+1}P_{2m+1}^{-1}e_{i+1}\|_1 \\
	=&
	\begin{cases}
	  \left( \sum_{j=0}^{i}{2m\choose j} \right)^{l}&\quad \textnormal{ for }i\in\left\{ 0,\dots,m \right\}\\  
	  \left( \sum_{j=0}^{m}{2m\choose j} \right)^{l}+l\sum_{s=m+2}^{i+1}\left( \sum_{j=0}^{2m+1-s}{2m\choose j} \right)^{l-1}{2m\choose 2m+1-s}&\quad \textnormal{ for }i\in\left\{ m+1,\dots,2m \right\}.
	\end{cases}
  \end{align*}

  If $n$ is odd, we take $m\in\mathbb{N}_+$ such that $n=2m-1$. In the same way we obtain
  \begin{align*}
	&\|B_{2m-1}^l e_{i+1}\|_1=\|P_{2m}J^l_{2m}P_{2m}^{-1}e_{i+1}\|_1 \\
	=&
	\begin{cases}
	  \left( \sum_{j=0}^{i}{2m-1\choose j} \right)^{l}&\quad \textnormal{ for }i\in\left\{ 0,\dots,m-1 \right\}\\  
	  \left( \sum_{j=0}^{m-1}{2m-1\choose j} \right)^{l}+l\sum_{s=m+1}^{i+1}\left( \sum_{j=0}^{2m-s}{2m-1\choose j} \right)^{l-1}{2m-1\choose 2m-s}&\quad \textnormal{ for }i\in\left\{ m,\dots,2m-1 \right\}.
	\end{cases}
  \end{align*}
	\end{proof}
  \end{corollary}
  \subsection{Results for the discussion of Section~\ref{sec:discussion}}
  \label{sec:discussionProofs}
\begin{proof}[Proof of Lemma~\ref{lem:MontufarNaive}]
  The first property follows from Lemmas~\ref{lem:sameArgPreceq} and \ref{lem:monotonicityPhi}. For the strictness condition note that
	\begin{eqnarray*}
	 \prod_{l=1}^{L}\sum_{j=0}^{\min(n_0,\dots,n_{l-1})}{n_{l}\choose j}<
	 \prod_{l=1}^{L} 2^l
	 &\iff& \exists l\in\left\{ 1,\dots,L \right\}\quad \min\left( n_0,\dots,n_{l-1} \right)<n_l\\
	 &\iff& \exists l\in\left\{ 1,\dots,L \right\}\quad n_{l-1}<n_l.\qedhere
	\end{eqnarray*}
  \end{proof}
\begin{definition}
  For $v_1,v_2\in V$ we write $v_1\prec v_2$ if and only if $v_1\preceq v_2$ and $v_1\neq v_2$.
\end{definition}
\begin{lemma} For all $\gamma\in\Gamma$,  $n'\in\mathbb{N}_+$ and $v_1,v_2\in V$
  \label{lem:strictstrictMonotonicityPhi}
  \begin{equation*}
	v_1\preceq v_2 \land \|v_1\|_1< \|v_2\|_1\implies \|\varphi_{n'}^{(\gamma)}(v_1)\|_1<\|\varphi_{n'}^{(\gamma)}(v_2)\|_1.
  \end{equation*}
  \begin{proof}
	We assume $v_1\preceq v_2 \land \|v_1\|_1< \|v_2\|_1$. By Definition~\ref{def:preceq}, we can find $\tilde v_1,\delta\in V$ such that $\delta\neq 0$, $v_1\preceq \tilde v_1$ and $\tilde v_1+\delta=v_2$ and by Definition~\ref{def:PhiFunc} and Lemma~\ref{lem:monotonicityPhi}, 
	\begin{equation*}
	  \|\varphi_{n'}^{(\gamma)}(v_1)\|_1\le  \|\varphi_{n'}^{(\gamma)}(\tilde v_1)\|_1<\|\varphi_{n'}^{(\gamma)}(\tilde v_1)\|_1+\|\varphi_{n'}^{(\gamma)}(\delta )\|_1=\|\varphi_{n'}^{(\gamma)}(\tilde v_1+ \delta )\|_1.\qedhere
	\end{equation*}
  \end{proof}
\end{lemma}
\begin{lemma} Assume that $\gamma^{(1)},\gamma^{(2)}\in\Gamma$ and that for any $n\in\left\{ 0,\dots,n' \right\}$, $\tilde \gamma^{(1)}_{n',n}\preceq\gamma^{(2)}_{n',n}$. For $n'\in\mathbb{N}_+$ and $v\in V$, it holds that
  \label{lem:sameArgPreceq}
  \begin{equation*}
	\varphi_{n'}^{(\gamma^{(1)})}(v)\preceq\varphi_{n'}^{(\gamma^{(2)})}(v).
  \end{equation*}
  \begin{proof}
	The result follows directly from Definition~\ref{def:PhiFunc}.
  \end{proof}
\end{lemma}
From now on assume that $\gamma,\tilde \gamma$ are as in equations~\eqref{eq:zaslavskyGamma} and~\eqref{eq:binomialGamma}, i.e. for $n'\in\mathbb{N}_+,n\in\left\{ 0,\dots,n' \right\}$ assume
\begin{equation}
  \label{eq:gammadefComparison}
  \gamma_{n,n'}=\sum_{j=0}^{n}{n'\choose j}{\rm e}_{n'},\quad \tilde \gamma_{n,n'}=\sum_{j=0}^{n}{n'\choose j}{\rm e}_{n'-j}.
\end{equation}
\begin{lemma} When $\tilde \gamma$ is as above, then for all $n'\in\mathbb{N}_+$, $v_1,v_2\in V$
  \label{lem:strictNomIneqPhi}
  \begin{equation*}
	\textnormal{cl}_{n'}(v_1)\prec\textnormal{cl}_{n'}(v_2) \implies \|\varphi_{n'}^{(\tilde \gamma)}(v_1)\|_1<\|\varphi_{n'}^{(\tilde\gamma)}(v_2)\|_1.
  \end{equation*}
  \begin{proof}
	This follows directly from the definition of $\tilde\gamma$ in equation~\eqref{eq:gammadefComparison} and from Definition~\ref{def:PhiFunc}.
  \end{proof}
\end{lemma}
\begin{lemma} For $n'\in\mathbb{N}_+$ and $v\in V$, it holds that
  \label{lem:sameArg}
  \begin{equation*}
	\|\varphi_{n'}^{(\tilde\gamma)}(v)\|_1=\|\varphi_{n'}^{(\gamma)}(v)\|_1.
  \end{equation*}
  \begin{proof}
	This directly follows from equation~\eqref{eq:gammadefComparison} and Definition~\ref{def:PhiFunc} because for $n'\in\mathbb{N}_+$, $n\in\left\{ 0,\dots,n' \right\}$ $\|\gamma_{n,n'}\|_1=\|\tilde\gamma_{n,n'}\|_1$. 
  \end{proof}
\end{lemma}
\begin{proof}[Proof of Lemma~\ref{lem:BinomialMontufar}]
  The first property follows from Lemmas~\ref{lem:sameArgPreceq} and \ref{lem:monotonicityPhi}. For the strictness condition assume $n'\in\mathbb{N}_+$ and $i\in\left\{0,\dots,n' \right\}$. We distinguish two cases:
  \begin{itemize}
	\item If $2i\le n'$ then $\varphi_{n'}^{(\tilde\gamma)}({\rm e}_i)=\textnormal{cl}_i\left( \tilde\gamma_{i,n'} \right)=\textnormal{cl}_i\left(\gamma_{i,n'} \right)=\varphi_{n'}^{(\gamma)}({\rm e}_i)$.
	\item If $2i> n'$ then the following holds
	  \begin{eqnarray}
		&\varphi_{n'}^{(\tilde\gamma)}({\rm e}_i)= \textnormal{cl}_i\left( \tilde\gamma_{i,n'} \right)\preceq\textnormal{cl}_i\left(\gamma_{i,n'} \right)= \varphi_{n'}^{(\gamma)}({\rm e}_i)&\\
	  &\|\varphi_{n'}^{(\tilde\gamma)}({\rm e}_i)\|_1= \|\textnormal{cl}_i\left( \tilde\gamma_{i,n'} \right)\|_1=\|\textnormal{cl}_i\left(\gamma_{i,n'} \right)\|_1= \|\varphi_{n'}^{(\gamma)}({\rm e}_i)\|_1&\\
	  &\min\{j\in\mathbb{N}|(\tilde\gamma_{i,n'})_j\neq 0\}=n'-i&\\
	  &\min\{j\in\mathbb{N}|(\gamma_{i,n'})_j\neq 0\}=i>n-i&
	  \end{eqnarray}
	  From the above properties, it follows that for $n''\in\mathbb{N}_{+}$
	  \begin{equation*}
		\textnormal{cl}_{n''}\left( \varphi_{n'}^{(\tilde\gamma)}({\rm e}_i) \right) \prec\textnormal{cl}_{n''}\left( \varphi_{n'}^{(\gamma)}({\rm e}_i) \right) \iff n'-i<n''
	  \end{equation*}
  \end{itemize}
  The above case analysis shows that for $n',n''\in\mathbb{N}_{+}$ and $i\in\left\{ 0,\dots,n' \right\}$
  \begin{align*}
	\textnormal{cl}_{n''}\left( \varphi_{n'}^{(\tilde\gamma)}({\rm e}_i) \right) \prec\textnormal{cl}_{n''}\left( \varphi_{n'}^{(\gamma)}({\rm e}_i) \right) & \textnormal{ if } n'<\min(2i,i+n''),\\
	\textnormal{cl}_{n''}\left( \varphi_{n'}^{(\tilde\gamma)}({\rm e}_i) \right) =\textnormal{cl}_{n''}\left( \varphi_{n'}^{(\gamma)}({\rm e}_i) \right) & \textnormal{ if } n'\ge\min(2i,i+n'').
  \end{align*} 
  This implies for $n',n\in\mathbb{N}_+$ and $v\in V$
  \begin{align}
	\label{eq:clipPreceq}
	\textnormal{cl}_{n''}\left( \varphi_{n'}^{(\tilde\gamma)}(\textnormal{cl}_{n'}(v)) \right) \prec\textnormal{cl}_{n''}\left( \varphi_{n'}^{(\gamma)}(\textnormal{cl}_{n'}(v)) \right) &\textnormal{ if } n'<\min(2i^*,i^*+n''),\\
	\label{eq:clipEquality}
	\textnormal{cl}_{n''}\left( \varphi_{n'}^{(\tilde\gamma)}(\textnormal{cl}_{n'}(v)) \right) =\textnormal{cl}_{n''}\left( \varphi_{n'}^{(\gamma)}(\textnormal{cl}_{n'}(v)) \right) &\textnormal{ if } n'\ge\min(2i^*,i^*+n'')
  \end{align} with $i^*=\min(n',\max\left\{ i\in\mathbb{N}|v_i\neq 0 \right\})$.

  Note that for all $l\in\left\{ 1,\dots,L-1 \right\}$
  \begin{equation}
	\label{eq:maxNonzero}
	\max\left\{ i\in\mathbb{N}\middle\vert\left(\varphi_{l-1}^{(\gamma)}\circ\cdots\circ\varphi^{(\gamma)}_{1}({\rm e}_0) \right)_i\neq 0 \right\}=\min_{j=0,\dots,l-1} n_j
  \end{equation}since neither $\left( \textnormal{cl}_{n}(\gamma_{n,n'}) \right)_n=0$ nor $\left( \textnormal{cl}_n(\tilde \gamma_{n,n'}) \right)_n=0$ for any $n'\in\mathbb{N}_{+},n\in\left\{ 0,\dots,n' \right\}$ by equation~\eqref{eq:gammadefComparison}.  Now note that
  \begin{align}
  &n_l<\min\left( n_0,\dots,n_{l} \right)+\min\left( n_0,\dots,n_{l+1} \right)\nonumber\\
	  \label{eq:reformulationMin}
	  \iff& n_l<\min\left(2\min\left( n_0,\dots,n_{l} \right),\min\left( n_0,\dots,n_{l}\right) +n_{l+1}\right)
  \end{align}
  \begin{itemize}
	\item If there exists no $l\in\left\{ 1,\dots,L-1 \right\}$ such that $n_l<\min\left( n_0,\dots,n_{l} \right)+\min\left( n_0,\dots,n_{l+1} \right)$ then for all $l\in \left\{ 1,\dots,L-1 \right\}$ 
	  \begin{equation*}
		n_l\ge\min(2\min\left( n_0,\dots,n_{l} \right),\min\left( n_0,\dots,n_{l} \right)+n_{l+1})
	  \end{equation*} by equation~\eqref{eq:reformulationMin}
	 such that by equations~\eqref{eq:maxNonzero} and~\eqref{eq:clipEquality}, 
	  \begin{equation*}
		\textnormal{cl}_{n_{l+1}}\left( \varphi_{n_l}^{(\tilde\gamma)}(\textnormal{cl}_{n_l}(\varphi^{(\gamma)}_{n_{l-1}}\circ\dots\circ\varphi^{(\gamma)}_{n_1}({\rm e}_{n_0}))) \right) =\textnormal{cl}_{n_{l+1}}\left( \varphi_{n_l}^{(\gamma)}(\textnormal{cl}_{n_l}(\varphi^{(\gamma)}_{n_{l-1}}\circ\dots\circ\varphi^{(\gamma)}_{n_1}({\rm e}_{n_0}))) \right).
		\label{eq:exchangeTilde}
	  \end{equation*} Successive application of the above result for $l=L-1,\dots,1$ yields
	  \begin{equation*}
		\varphi^{(\gamma)}_{n_L}\circ\dots\circ\varphi^{(\gamma)}_{n_1}({\rm e}_{n_0})=\varphi^{(\gamma)}_{n_L}\circ\varphi^{(\tilde\gamma)}_{n_{L-1}}\circ\dots\circ\varphi^{(\tilde \gamma)}_{n_1}({\rm e}_{n_0})
	  \end{equation*} such that Lemma~\ref{lem:sameArg} implies 
	  $\|\varphi^{(\gamma)}_{n_L}\circ\dots\circ\varphi^{(\gamma)}_{n_1}({\rm e}_{n_0})\|_1=\|\varphi^{(\tilde\gamma)}_{n_L}\circ\dots\circ\varphi^{(\tilde \gamma)}_{n_1}({\rm e}_{n_0})\|_1$, hence the bounds are equal for $\gamma$ and $\tilde \gamma$.
	\item If there is a $l\in\left\{ 1,\dots,L-1 \right\}$ such that $n_l<\min\left( n_0,\dots,n_{l} \right)+\min\left( n_0,\dots,n_{l+1} \right)$, then equations~\eqref{eq:clipPreceq}, \eqref{eq:reformulationMin} and \eqref{eq:maxNonzero} imply
	  \begin{equation*}
		\textnormal{cl}_{n_{l+1}}\left( \varphi_{n_l}^{(\tilde\gamma)}(\textnormal{cl}_{n_l}(\varphi^{(\gamma)}_{n_{l-1}}\circ\dots\circ\varphi^{(\gamma)}_{n_1}({\rm e}_{n_0}))) \right) \prec\textnormal{cl}_{n_{l+1}}\left( \varphi_{n_l}^{(\gamma)}(\textnormal{cl}_{n_l}(\varphi^{(\gamma)}_{n_{l-1}}\circ\dots\circ\varphi^{(\gamma)}_{n_1}({\rm e}_{n_0}))) \right).
	  \end{equation*}If we combine this with Lemmas~\ref{lem:sameArgPreceq}, \ref{lem:alphapreceq} and \ref{lem:monotonicityPhi}, we obtain
	  \begin{equation*}
		\textnormal{cl}_{n_{l+1}}\left( \varphi_{n_l}^{(\tilde\gamma)}(\textnormal{cl}_{n_l}(\varphi^{(\tilde\gamma)}_{n_{l-1}}\circ\dots\circ\varphi^{(\tilde\gamma)}_{n_1}({\rm e}_{n_0}))) \right) \prec\textnormal{cl}_{n_{l+1}}\left( \varphi_{n_l}^{(\gamma)}(\textnormal{cl}_{n_l}(\varphi^{(\gamma)}_{n_{l-1}}\circ\dots\circ\varphi^{(\gamma)}_{n_1}({\rm e}_{n_0}))) \right).
	  \end{equation*}Using Lemma~\ref{lem:strictNomIneqPhi} for $n'=n_{l+1}$ once and then Lemmas~\ref{lem:monotonicityPhi}, \ref{lem:strictstrictMonotonicityPhi} and \ref{lem:sameArgPreceq} inductively yields $\|\varphi^{(\tilde\gamma)}_{n_L}\circ\dots\circ\varphi^{(\tilde\gamma)}_{n_1}({\rm e}_{n_0})\|_1<\|\varphi^{(\gamma)}_{n_L}\circ\dots\circ\varphi^{(\gamma)}_{n_1}({\rm e}_{n_0})\|_1$ hence the binomial bound is strictly better than Mont\'ufar's bound.
  \end{itemize}
  \end{proof}
\newpage

\bibliography{obereSchranke} 

\begin{thebibliography}{1}

\bibitem{Montufar:2014:NLR:2969033.2969153}
G.~Mont\'{u}far, R.~Pascanu, K.~Cho, and Y.~Bengio, ``On the number of linear
  regions of deep neural networks,'' in {\em Proceedings of the 27th
  International Conference on Neural Information Processing Systems - Volume
  2}, NIPS'14, (Cambridge, MA, USA), pp.~2924--2932, MIT Press, 2014.

\bibitem{Montufar17}
G.~Mont\'{u}far, ``Notes on the number of linear regions of deep neural
  networks,'' 03 2017.

\bibitem{DBLP:BoundingCounting}
T.~Serra, C.~Tjandraatmadja, and S.~Ramalingam, ``Bounding and counting linear
  regions of deep neural networks,'' {\em CoRR}, vol.~abs/1711.02114, 2017.

\bibitem{piecewiseQuadratic}
H.~Pottmann, R.~Krasauskas, B.~Hamann, K.~Joy, and W.~Seibold, ``On piecewise
  linear approximation of quadratic functions,'' {\em Journal for Geometry and
  Graphics Volume}, vol.~4, pp.~31--53, 01 2000.

\bibitem{Bartlett2006}
P.~L. Bartlett and S.~Mendelson, ``Empirical minimization,'' {\em Probability
  Theory and Related Fields}, vol.~135, pp.~311--334, Jul 2006.

\bibitem{pmlr-v70-raghu17a}
M.~Raghu, B.~Poole, J.~Kleinberg, S.~Ganguli, and J.~Sohl-Dickstein, ``On the
  expressive power of deep neural networks,'' in {\em Proceedings of the 34th
  International Conference on Machine Learning} (D.~Precup and Y.~W. Teh,
  eds.), vol.~70 of {\em Proceedings of Machine Learning Research},
  (International Convention Centre, Sydney, Australia), pp.~2847--2854, PMLR,
  06--11 Aug 2017.

\bibitem{1975Zaslavsky}
T.~Zaslavsky, ``Facing up to arrangements: Face-count formulas for partitions
  of space by hyperplanes: Face-count formulas for partitions of space by
  hyperplanes,'' {\em American Mathematical Soc.}, vol.~154, 1975.

\bibitem{BMVC2016_87}
S.~Zagoruyko and N.~Komodakis, ``Wide residual networks,'' in {\em Proceedings
  of the British Machine Vision Conference (BMVC)} (E.~R.~H. Richard C.~Wilson
  and W.~A.~P. Smith, eds.), pp.~87.1--87.12, BMVA Press, September 2016.

\end{thebibliography}
\bibliographystyle{ieeetr}
  \end{document}